%% file: JDSA.tex
\documentclass[sn-mathphys]{sn-jnl}

\usepackage[small]{caption}
\usepackage{epsfig}
\usepackage{subfig}
\usepackage{multirow}
\usepackage{booktabs}
\usepackage{pifont}
\usepackage{color}
\usepackage{etoolbox}
\usepackage{colortbl}
\usepackage{xcolor,colortbl}
\usepackage{mathtools}
\usepackage{rotating}
\usepackage{sidecap}
\usepackage{array}
\usepackage[final]{pdfpages}
\usepackage{chngpage}
\usepackage{lscape}
\usepackage{xspace}
\usepackage{mathbbol,mathtools,amsmath,amssymb,amsfonts,amsbsy,latexsym}

\jyear{2021}%

\theoremstyle{thmstyleone}%
\newtheorem{theorem}{Theorem}
\newtheorem{proposition}[theorem]{Proposition}%

\theoremstyle{thmstyletwo}%
\newtheorem{example}{Example}%

\theoremstyle{thmstylethree}%
\newtheorem{definition}{Definition}%

\definecolor{G1}{rgb}{0.67, 0.8, 0.94}
\definecolor{G2}{gray}{0.90}
\definecolor{G5}{rgb}{0.00, 0.75, 0.93}
\definecolor{G4}{rgb}{0.55, 0.75, 0.93}
\definecolor{G3}{rgb}{0.88, 0.85, 0.90}

\def\closedx{\textsc{closedDiversity}}
\def\cx{\textsc{ClosedDiv}}

\def\closedp{\textsc{closedPattern}}
\def\flexics{\textsc{Flexics}}
\def\gflexics{\textsc{GFlexics}}
\def\eflexics{\textsc{EFlexics}}
\def\cdflexics{\textsc{CFFlexics}}
\def\letsip{\textsc{LetSIP}}
\def\letsipcdf{\textsc{LetSIP-cdf}}
\def\newletsip{\textsc{LetSIP-disc}}
\def\newletsiplin{\textsc{LetSIP-disc-Lin}}
\def\newletsipexp{\textsc{LetSIP-disc-Exp}}
\def\newletsipcdf{\textsc{LetSIP-disc-cdf}}
\def\newletsipcdfexp{\textsc{LetSIP-disc-cdf-exp}}

\def\aggregexp{\textsc{EXP}}
\def\aggregsum{\textsc{Linear}}
\def\newlin{\textsc{Lin}}
\def\newexp{\textsc{Exp}}
\def\cpim{\textsc{CP4IM}}

\def\weightgen{\textsc{Weightgen}}
\def\aple{\textsc{APLe}}
\def\ipm{\textsc{IPM}}
\def\ranksvm{\textsc{RankSVM}}

\def\eflexics{\textsc{EFlexics}\xspace}
\def\cdflexics{\textsc{cdFlexics}\xspace}
\def\cpim{\textsc{CP4IM}\xspace}

\newcommand{\bdd}{\mathcal{T}}			
\newcommand{\items}{\mathcal{I}}
\newcommand{\lang}[1]{{\cal L}_{#1}}

\newcommand{\store}{\mathcal{H}}
\newcommand{\SDB}{\mathcal{D}}

\newcommand{\jac}[0]{\ensuremath{Jac}}

\newcommand{\ie}{{\it i.e., \/}}

\newcommand{\ea}[1]{#1 {\it et al.}}

\newcommand{\freq}[0]{\ensuremath{sup}_{\SDB}}

\newcommand{\cover}[0]{\mathcal{V}_{\SDB}}

\newcommand{\surp}[0]{\ensuremath{surp}}
\newcommand{\jmax}[0]{\ensuremath{J_{max}}\xspace}

\newcommand{\card}[1]{\ensuremath{\lvert #1\rvert}}

\newcommand{\icv}{\ensuremath{ICV}}
\newcommand{\avg}[0]{\ensuremath{Avg}}

\newcommand{\theory}[3]{\mathcal{T}(#1,#2,#3)}
\newcommand{\theoryextension}[4]{\{#1 #2 #3\mid\,\,#4\}}
\newcommand{\Lang}[0]{{\cal L}}
\begin{document}

\title[Exploiting complex pattern features for interactive pattern mining]{Exploiting complex pattern features for interactive pattern mining}


\author*[1]{\fnm{Arnold} \sur{Hien}}\email{arnold.hien@unicaen.fr}

\author[2]{\fnm{Samir} \sur{Loudni}}\email{samir.loudni@imt-atlantique.fr}
\equalcont{These authors contributed equally to this work.}

\author[3]{\fnm{Noureddine} \sur{Aribi}}\email{aribi.noureddine@gmail.com}
\author[1]{\fnm{Abdelkader} \sur{Ouali}}\email{abdelkader.ouali@unicaen.fr}
\author[1]{\fnm{Albrecht} \sur{Zimmermann}}\email{albrecht.zimmermann@unicaen.fr}

\equalcont{These authors contributed equally to this work.}

\affil*[1]{\orgdiv{CNRS-UMR GREYC}, \orgname{Normandie Univ., UNICAEN}, \orgaddress{\city{Caen}, \postcode{14032}, \country{France}}}

\affil[2]{\orgdiv{TASC (LS2N-CNRS)}, \orgname{IMT Atlantique}, \orgaddress{\street{4 rue Alfred Kastler}, \city{Nantes}, \postcode{44300}, \country{France}}}

\affil[3]{\orgdiv{LITIO}, \orgname{University of Oran1}, \orgaddress{\city{Oran}, \postcode{31000},  \country{Algeria}}}


\abstract{Recent years have seen a shift from a pattern mining process that has users define constraints before-hand, and sift through the results afterwards, to an interactive one. This new framework depends on exploiting user feedback to learn a quality function for patterns. Existing approaches have a weakness in that they use static pre-defined low-level features, and attempt to learn independent weights representing their importance to the user. As an alternative, we propose to work with more complex features that are derived directly from the pattern ranking imposed by the user. Learned weights are then aggregated onto lower-level features and help to drive the quality function in the right direction. We explore the effect of different parameter choices experimentally and find that using higher-complexity features leads to the selection of patterns that are better aligned with a hidden quality function while not adding significantly to the run times of the method.

Getting good user feedback requires to quickly present diverse patterns, something that we achieve but pushing an existing diversity constraint into the sampling component of the interactive mining system \letsip{}. Resulting patterns allow in most cases to converge to a good solution more quickly.

Combining the two improvements, finally, leads to an algorithm showing clear advantages over the existing state-of-the-art.}



\maketitle
\input{1-introduction}

\input{2-preliminaries}
\input{3-cdflexcis}
\input{4-new-features}

\input{5-aggregations}
\input{6-experimentations}
\input{7-related-work}
\input{8-conclusions} 

\section*{Declarations}

\begin{itemize}
\item Conflict of interest/Competing interests: all authors declare not having any financial or non-financial interests that are directly or indirectly related to the work submitted for publication.	
\item Availability of data, materials and code: They will be available in the final submission. 
\end{itemize}

\eject
\bibliography{JDSA.bib}

\input{appendix}



\end{document}

%% file: 1-introduction.tex
\section{Introduction}
\label{sec:intro}
Constraint-based pattern mining is a fundamental data mining task, extracting locally interesting patterns to be either interpreted directly by domain experts, or to be used as descriptors in downstream tasks, such as classification or clustering. Since the publication of the seminal paper \cite{Agrawal94}, two problems have limited the usability of this approach: 1) how to translate user preferences and background knowledge into constraints, and 2) how to deal with the large result sets that often number in the thousands or even millions of patterns. Replacing the original support-confidence framework with other quality measures \cite{DBLP:conf/kdd/TanKS02} does not address the pattern explosion. Post-processing results via condensed representations \cite{Lakhal} still typically leaves many patterns, while pattern set mining \cite{DBLP:conf/sdm/RaedtZ07} just pushes the problem further down the line.

In recent years, research on \emph{interactive pattern mining} has proposed to alter the mining process itself: instead of specifying constraints once, mining a result set, and then post-processing it, interactive pattern mining performs an iterative loop \cite{DBLP:conf/icml/Rueping09}. This loop involves three repeating main steps:
(1) pattern extraction in which a relatively small set of patterns is extracted;
(2) interaction in which the user expresses his preferences w.r.t. those patterns;
(3) preference learning in which the expressed preferences are translated into a quality assessment function for mining patterns in future iterations.

Existing approaches \cite{DBLP:conf/icml/Rueping09,Dzyuba:letsip,DBLP:journals/sadm/BhuiyanH16} have a short-coming, however: to enable preference learning, they represent patterns by independent descriptors, such as included items or covered transactions, and expect the learned function, usually a regression or multiplicative weight model, to handle relations. 
Furthermore, to get quality feedback, it is important to present \emph{diverse} patterns to the user, an aspect that is at best indirectly handled by existing work.

The contributions of this paper are the following: 
\begin{enumerate}
	\item We introduce \cdflexics, which improves the pattern sampling method proposed in \cite{Dzyuba:flexics}, \flexics, by integrating \closedx~into it to extract patterns. This will allow us to exploit both the sampling and the $XOR$ constraints of \flexics~which has the advantage that it is anytime, as well as the filtering of \cx~to maximize patterns diversity. 
	\item 
	We extend a recent interactive pattern mining approach (\letsip~\cite{Dzyuba:letsip}) by introducing new, more complex, descriptors for \textit{explainable ranking}, thereby leading to a more concise and diverse pattern sets.
	These descriptors exploit the concept of {\bf discriminating sub-patterns}, which separate patterns that are given low rank by the user from those with high rank. By temporarily adding those descriptors, we can learn weights for them, which are then apportioned to involved items without blowing up the feature space.
	\item We combine those two improvements to \letsip{}, i.e. the improved sampling of \cdflexics and the more complex features learned dynamically during the iterative mining process. This allows the resulting algorithm to both return more diverse patterns more efficiently, and to leverage user feedback more effectively to improve later iterations.
\end{enumerate} 

In Section \ref{sec:contrib:cdflexics}, we discuss how we integrate \closedx into \flexics. In Section \ref{sec:technical-contribution}, we explain how to derive complex pattern features, followed by a discussion of how to map user feedback back into primitive features. Section \ref{sec:exps} containts extensive experiments evaluating our contributions, before we discuss related work in Section \ref{sec:related_works}. We conclude in Section \ref{sec:conclusion}. 

%% file: 2-preliminaries.tex
\section{Preliminaries}
\label{sec:preliminaries}

\begin{algorithm}[t]
	\caption{The general CP solving algorithm} \label{alg:generic_cp}
	\algrenewcommand\algorithmicrequire{\textbf{Data}:}
	\algrenewcommand\algorithmicensure{\textbf{Result}:}
	\algblock[Name]{Begin}{End}
	\begin{algorithmic}[1]		
		\Function{\textsc{Solve}}{$\mathcal{S} = \left<X, D, \mathcal{C}\right>$}
		\Begin
		\State $\textsc{Propagate}(\mathcal{S})$
		\If{$\exists x_i \in X \,s.t.\, \card{dom(x_i)} = 0$}
		\State \Return $\emptyset$
		\ElsIf{$\exists x_i \in X \,s.t.\, \card{dom(x_i)} > 1$}
		\State $d \leftarrow \textsc{MakeDecision}(\mathcal{P})$ \;
		\State \Return $\textsc{Solve}(\mathcal{P}\land d) \cup \textsc{Solve}(\mathcal{P}\land \lnot d)$
		\Else
		\State \Return solution $D$ 
		\EndIf
		\End
		\EndFunction
	\end{algorithmic}
\end{algorithm}

\subsection{Constraint Programming (CP)}
Constraint programming \cite{HoeveK06} is a powerful paradigm which offers a generic and modular approach to model and solve combinatorial problems.
A CP model consists of a set of variables $X=\{x_1,\ldots, x_n\}$, a set of domains $D$ mapping each variable $x_i\in X$ to a finite set of possible values $dom(x_i)$, and a set of constraints $\mathcal{C}$ on $X$.
A constraint $c \in \mathcal{C}$ is a relation that specifies the allowed combinations of values for its variables $X(c)$. An assignment on a set $Y \subseteq X$ of variables is a mapping from variables in $Y$ to values in their domains. A solution is an assignment on $X$ satisfying all constraints.

To find solutions, CP solvers usually work in two steps. First, the search space is reduced by adding constraints (called decisions) in a depth-first search. A decision is, for example, an assignment of a variable to a value in its domain. Then comes a phase of propagation that checks the satisfiability of the constraints of the CP model. This phase can also filter (i.e. remove) values in the domains of
variables that do not lead to solutions . At each decision, constraint filtering algorithms prune
the search space by enforcing local consistency properties like {\it domain consistency}. A constraint $c$ on $X(c)$ is domain consistent, if and only if, for every  $x_i \in X(c)$ and every $v \in dom(x_i)$, there is an assignment satisfying $c$ such that $(x_i = v)$. 
The baseline recursive algorithm for solving is presented in Algorithm 1. It relies on the two functions \textsc{Propagate} and \textsc{MakeDecision} to respectively filter values by using constraints and reduce the search space by making decisions.

\subsection{Pattern mining} 
Given a database $\SDB$, a language $\Lang$ defining subsets of the data and a selection predicate $q$ that determines whether an element $\phi \in \Lang$ (a pattern) describes an interesting subset of $\SDB$ or not, the task is to find all interesting patterns, that is,
$$\theory{\Lang}{\SDB}{q}=\theoryextension{\phi}{\in}{\Lang}{q(\SDB,\phi)\,\,is\,\,true}$$

A well-known pattern mining task is {\it frequent itemset mining}~\cite{Agrawal94}. 
Let $\items=\{1, ..., n\}$ be a set of $n$ \textit{items}, an \emph{itemset} (or pattern) $P$ is a non-empty subset of $\items$. The language of itemsets corresponds to $\lang{\items} = 2^{\items} \backslash \emptyset$. A transactional dataset $\SDB$ is a bag (or multiset) of transactions over $\items$, where each \emph{transaction} $t$ is a subset of  $\items$, i.e., $t \subseteq \items$; $\bdd = \{1,...,m\}$ a set of $m$ \textit{transaction} indices. An itemset $p$ \emph{occurs} in a transaction $t$, iff $p \subseteq t$. The \emph{cover} of $p$ in $\SDB$ is the bag of transactions in which it occurs: $\cover(p) = \{t\in \SDB  {}\mid p \subseteq  t\}$. The \emph{support} of $p$ in $\SDB$ is the size of its cover: $\freq(p)$ = $\card{\cover(p)}$. A well-known \emph{interestingness} predicate is a threshold on the support of the itemsets, the \emph{minimal support}: $\theta$. The task in frequent set mining is to compute \emph{frequent itemsets}:  $\theoryextension{p}{\in}{\lang{\items}}{\freq(p) \geq \theta}$.,

Additional constraints on the individual itemsets to reduce the redundancy between patterns, such as closed frequent itemsets\cite{Lakhal}, have been proposed. An itemset $p$ is said to be closed if there is no $q \supseteq p$ such that $\freq(p) = \freq(q)$, that is, $p$ is maximal with respect to set inclusion among those itemsets having the same support. 

\medskip
\noindent
\textbf{\bf Constraint Programming for Itemset Mining.} Several proposals have investigated relationships between itemset mining and constraint programming (CP) to revisit data mining tasks in a declarative and generic way \cite{BelaidBL19a,LazaarLLMLBB16,RaedtGN08,SchausAG17,KhiariBC10}. The declarative aspect represents the key advantage of the proposed CP approaches. This allows modeling of complex user queries without revising the solving process but at the cost of efficiency.

A CP model for mining frequent closed itemsets was proposed in \cite{LazaarLLMLBB16} which has successfully encoded both the closeness relation as well as the anti-monotonicity of frequency into one global constraint called \closedp.  It uses  a vector of Boolean variables $(x_1, \ldots, x_{\card{\items}})$ for representing itemsets, where $x_i$ represents the presence of the item $i \in \items$ in the itemset. 
We will use the notation $x^+ = \{i \in \items \,\mid\, x_i = 1 \}$ to represent the pattern associated to these variables.

\begin{definition}[\closedp~constraint]
	Let $(x_1, \ldots, x_{\card{\items}})$ be a vector of Boolean variables, $\theta$ a support threshold and $\bdd$ a dataset. 
	The constraint $\closedp_{\SDB, \theta}(x_1, \ldots, x_{\card{\items}})$  holds if and only if $x^+$ is a closed frequent itemset w.r.t. the threshold $\theta$. 
\end{definition}

\subsection{Pattern Diversity}
The discovery of a set of diverse itemsets from binary data is an important data mining task. Different measures have been proposed to measure the diversity of itemsets. In this paper, we consider the Jaccard index as a measure of similarity on sets. We use it to quantify the overlap of the covers of itemsets. 

\begin{definition}[Jaccard index]
	Let $p$ and $q$ be two itemsets. The \emph{Jaccard index} is defined as 
	$$
	\jac(p,q) = \frac{\lvert \cover{p} \cap \cover{q} \rvert}{\lvert \cover{p} \cup \cover{q} \rvert}
	$$
\end{definition}

A lower Jaccard indicates low similarity between itemset covers, and can thus
be used as a measure of diversity between pairs of itemsets.

\begin{definition}[\bf Diversity/Jaccard constraint] 
	Let $p$ and $q$ be two itemsets. Given the $\jac$ measure and a diversity threshold $J_{max}$, we say that $p$ and $q$ are pairwise diverse 
	iff $\jac(p,q) \leq J_{max}$.
\end{definition} 

We now formalize the diversity problem for itemset mining. Diversity is controlled through a threshold on the Jaccard similarity of pattern occurrences.

\begin{definition}[Diverse frequent itemsets]
	\label{prob1}
		Let $\mathcal{H}$ be a history of pairwise diverse frequent itemsets, and $J_{max}$ a bound on the maximum allowed Jaccard value. The task is to mine new itemsets $p$ such that $p$ is diverse according to $\mathcal{H}$: 
		$$\texttt{div}_{J}(p, \mathcal{H}, J_{max}) \Leftrightarrow \forall h \in \mathcal{H}, \jac(p,h) \leq J_{max}$$
\end{definition}

In~\cite{HienLALLOZ20}, we showed that the Jaccard index has no monotonicity property, which prevents usual pruning techniques and makes classical pattern mining unworkable. To cope with the non-monotonicity of the Jaccard similarity, we proposed an anti-monotonic lower bound relaxation of the Jaccard index, which allow effective pruning. 

\begin{proposition}[Lower bound]
Let $p$ and $h$ be two patterns, and $\theta$ a bound on the minimal frequency of the patterns, then
$$LB(p,h) = \frac{\max(0, \theta-\card{\cover(p)\backslash\cover(h)}}{\card{\cover(h)} + \card{\cover(p)\backslash\cover(h)}} \le \jac(p,h)
$$
\end{proposition}

This Approximate bound is then exploited within the \closedp~constraint to mine pairwise diverse frequent closed itemsets.

\begin{definition}[\bf \closedx]\label{def:closed:div}
	Let $x$ be a vector of Boolean item variables, $\mathcal{H}$ a history of pairwise diverse frequent closed itemsets (initially empty), $\theta$ a support threshold, $J_{max}$ a bound on the maximum allowed Jaccard, and $\SDB$ a dataset. 
	The $\closedx_{\SDB, \theta}(x,\mathcal{H},J_{max})$ constraint  holds if and only if: (1) $x^+$ is closed; (2) $x^+$ is frequent, $\freq(x^+) \geq \theta$; (3) $x^+$ is diverse, $\forall\, h \in \mathcal{H}, LB(x^+,h) \leq J_{max}$.
\end{definition}

By using this lower bound, it is easier to filter items that would give patterns too close to the ones in the history (see \cite{HienLALLOZ20} for more details).  

\subsection{\letsip: Weighted constrained sampling} 
Our work is based on \letsip~\cite{Dzyuba:letsip}, which in turn uses \flexics~\cite{Dzyuba:flexics}. In this and the next section, we will therefore sketch the pattern sampling and feedback + preference learning steps of \letsip~(see Algorithm~\ref{algo:letsip}) before describing our own proposals in Sections \ref{sec:contrib:cdflexics} and \ref{sec:technical-contribution}. 
The main idea behind \flexics~consists of partitioning the search space into a number of cells, and then sampling a solution from a random cell. 
This technique was inspired by recent advances in weighted solution sampling for SAT problems~\cite{SatSampling-AAAI15}.
It exploits \weightgen~\cite{DBLP:journals/corr/ChakrabortyFMSV14} to sample SAT problem solutions. 
\weightgen~partitions the search space into cells using random XOR constraints and then extracts a pattern from a randomly selected cell. 
An individual XOR constraint over variables X has the form $ \bigotimes b_{i} . x_{i} = b_{0}$, where $b_{0\mid i} \in \{0, 1\}$. The coefficients $b_i$ determine the variables involved in the constraint,
whereas the parity bit $b_0$ determines whether an even or an odd number of variables must be set to $1$. Together, $m$ XOR constraints identify one cell belonging to a partitioning of the overall solution space into $2^m$ cells.
\weightgen~uses a two-step procedure consisting of 1) estimating the number of XOR required for the partitioning and 2) generating these random constraints and sample the solutions. The sampling step needs the use of an efficient oracle to enumerate the solutions in the different cells. \ea{Dzyuba} evaluated two oracles: \eflexics~which employs the Eclat Algorithm for the enumeration and \gflexics, a CP-based method based on \cpim~\cite{RaedtGN08}. 
In this paper, we propose to use an efficient CP-based method, called \cdflexics, based the \closedx global constraint. 

\subsection{\letsip: Preference based pattern mining}
User feedback w.r.t. patterns takes the form of providing a total order over a (small) set of patterns, called a query $Q$. User feedback $\{p_1 \succ p_3 \succ p_2 \succ p_4\}$, for instance, indicates that the user prefers $p_1$ over $p_3$, which they prefer over $p_2$ in turn, and so on. This total order is translated into pairwise rankings $\{(p_1\succ p_3), (p_1\succ p_2), (p_1\succ p_4), \ldots\}$.
Patterns are represented using a feature representation $\mathbf{F} = \{f_{1},\ldots,f_{d}\} \in \mathbb{R}^d$. 
Examples of features include $Len(p) = \card{p}/\card{\items}$, $Freq(p) = \freq(p)/\card{\SDB}$, in which case the corresponding $f_i$ is truly in $\mathbb{R}$, or $Items(i,p) = [i \in p]$; and $Trans(t_i,p) = [p \subseteq t_i]$, where $[.]$ denotes the Iverson bracket, leading to (partial) feature vectors $\in \{0,1\}^{\card{\items}}$ and $\in \{0,1\}^{\card{\SDB}}$, respectively. The total number of features depends on the dimensions of the data, and is
equal to $\card{\items} + \card{\SDB} + 2$. 
Given this feature representation, each ranked pair $p_i\succ p_j$ corresponds to a classification example $(\mathbf{F}_i-\mathbf{F}_j, +)$.

\letsip~uses the parameterized logistic function
$$\phi_{logistic}(p;w,A) = A + \frac{1-A}{1+e^{-\mathbf{w\cdot {F}}}}$$ as a learnable quality function for patterns with $\mathbf{F}$ the aforementioned feature representation of a pattern $p$, $\mathbf{w}$ the weight vector to be learned, and $A$ is a parameter that controls the range of the interestingness measures, i.e.  $\phi_{logistic} \in [A,1]$. 
For details, we direct the reader to the original publication~\cite{Dzyuba:letsip}.
Finally, to select the $k$ patterns $Q$ to be presented to the user,  
\letsip~uses \flexics~to sample $k$ patterns proportional to $\phi_{logistic}$. These patterns are selected according to a {\sc Top}($m$) strategy, which picks the $m$ highest-quality patterns from a cell (Algorithm~\ref{algo:letsip}, line~\ref{letSip:top:m}). Moreover, to help users to relate the queries to each other, \letsip~propose to retain the top $\ell$ patterns from the previous query and only sample $k-\ell$ new patterns ~(Algorithm~\ref{algo:letsip}, lines \ref{letSip:samplinga}-\ref{letSip:samplingb}).

\begin{example}\label{weigth:letsip}
Let $\items = \{0, \ldots, 6\}$ be a set of items and 
$Q = \{p_1, p_2,p_3, p_4\}$ a sample of $four$ patterns, where $p_1 = \{4,6\}, p_2 = \{1,6\}, p_3 = \{0\}$ and $p_4 = \{3\}$. 
Let's consider a user ranking $U = \{ p_2 \succ p_1 \succ p_3 \succ p_4\}$ as feedback and suppose that $\freq(p_1) = 0.54$, $\freq(p_2) = 0.18$ and $\freq(p_3) = 0.36$. Using items and frequency as features, we have $\mathbf{F}_1 = (0,0,0,0,1,0,1,0.54)$, $\mathbf{F}_2 = (0,1,0,0,0,0,1,0.18)$, $\ldots$ To learn the  vector of weights $\mathbf{w}$, $U$ is translated into pairwise rankings $\{$($p_2 \succ p_1$), ($p_2 \succ p_3$), $\ldots\}$ and distances between feature vectors for each pair are calculated. 
So, given the preference $p_2\succ p_1$, we have $\mathbf{F}_2-\mathbf{F}_1 = (0,1,0,0,-1,0,0,-0.36)$. After the first iteration of \letsip, the learned weight vector is  $\mathbf{w} = (-0.33, 0.99, 0, -0.99, 0.33, 0, 1.33,0.15)$. 
 \end{example}

%% file: 3-cdflexcis.tex
\section{Sampling Diverse Frequent Patterns}
\label{sec:contrib:cdflexics}

\letsip’s sampling component is based on \flexics. An interactive system seeks to ensure faster learning of accurate models by targeted selection of patterns to show to the user. An important aspect is that the user be quickly presented with diverse results. If patterns are too similar to each other, deciding which one to prefer can become challenging, and if they appear in several successive iterations, it eventually becomes a slog. 

We propose in this section a modification to \flexics, called \cdflexics, to explicitly control the diversity of sampled patterns, thus ensuring pattern diversity within each query and thus sufficient exploration. Our key idea is to exploit the diversity constraint \closedx~as an oracle to select the patterns in the different cells (partitions). 
We start by presenting our strategy for integrating \closedx~into \flexics in section~\ref{sec:cdflexics:integration}, then we detail the associated filtering algorithm in section~\ref{sec:cdflexics:prop}.

\subsection{Integrating the diversity constraint into \flexics}
\label{sec:cdflexics:integration}
In order to turn \closedx~into a suitable oracle, we propose to control the diversity of patterns sampled 
from each cell as follows. We maintain a local history $\store_{loc}$ for each cell. 
Initially, $\store_{loc}$ is initialized to empty. Each time a new diverse pattern is mined from the current cell, it is added to $\store_{loc}$ and the exploration of the remaining search space of the cell continues using the new updated history. Finally, the pattern sampled for that cell is picked among the diverse patterns in $\store_{loc}$,  with a probability proportional to a given quality measure, e.g., frequency. This process is repeated for each new cell until the required number of samples is reached. With a such strategy, the exploration of the different cells become more efficient and faster, as their size is greatly reduced thank to the filtering of \closedx~which avoids the discovery of non diverse patterns within each cell. 

\begin{center}
	\begin{figure}[!t]\centering
		\begin{tabular}{ccccc}
			\begin{minipage}{.16\textwidth} \label{prop:step:1} 
					\begin{center} \footnotesize
						$
						\left\{
						\begin{array}{l}
						x_{1} \oplus x_{5} = 0 \\
						x_{1} \oplus x_{3} \oplus x_{5} = 0 \\
						x_{4} \oplus x_{5} = 1 \\
						\end{array}
						\right.
						$
				{\footnotesize 1) Random XOR \\ Constraints}
				\end{center}
			\end{minipage}
			&
			\begin{minipage}{.15\textwidth} 
				\begin{center}\footnotesize
						\begin{tabular}{*{2}{c|c}}
							1  0  0  0  1 & 0 \\
							1  0  1  0  1 & 0 \\
							0  0  0  1  1 & 1 \\
						\end{tabular}
					{\footnotesize 2) Initial \\ constraint matrix}
				\end{center}
			\label{fig:xor:step:2}
			\end{minipage}
			&
			\begin{minipage}{.22\textwidth} \label{propagate:step-3}
				\begin{center}\footnotesize
						\begin{tabular}{*{2}{c|c}}
							\quad 1  0  0  0  1 & 0 \\
							$\rightarrow$ 0  0  {\bf 1}  0  0 & {\bf 0} \\
							\quad 0  0  0  1  1 & 1 \\
						\end{tabular} 
					{\footnotesize 3) Echelonized matrix: \\ $x_3 = 0$ is derived}
				\end{center}
			\end{minipage}
			&
			\begin{minipage}{.20\textwidth}  \label{propagate:step-4}
				\begin{center}\footnotesize
						\begin{tabular}{*{2}{c|c}}
							1  0  0  0  1 & 0 \\
							0  0  0  1  1 & 1 \\
							0  0  0  0  0 & 0 \\
						\end{tabular} 
					{\footnotesize 4) Updated matrix 
						(rows 2 and 3 are swapped)}
				\end{center}
			\end{minipage}
			&
			\begin{minipage}{.18\textwidth} \label{propagate:step-5}
				\begin{center} \footnotesize
						\quad \quad \;\; $\downarrow$ $\downarrow$ \\
						\begin{tabular}{*{2}{c|c}}
							\quad 1  0  0  0 {\bf 0} & {\bf 1} \\
							$\rightarrow$ 0  0  0  {\bf 0 0} & {\bf 1} \\
							\quad 0  0  0  0  0 & 0 \\
						\end{tabular}
					{\footnotesize 5) If $x_4$ et $x_5$ are set to 1, the system is unsat.}
				\end{center}
			\end{minipage} \\
		\end{tabular}
		
		\caption{Propagation example of XOR constraints using Gaussian elimination. \label{propagate-example}}
	\end{figure}
\end{center}

\subsection{Filtering of \cdflexics}
\label{sec:cdflexics:prop} 
The filtering of \cdflexics combines two propagators, that of \closedx~(see \cite{HienLALLOZ20} for more details) and the propagator for XOR constraints. In CP, constraints interact through shared variables, i.e., as soon as a constraint has been propagated (no more values can be pruned), and at least a value has been eliminated from the domain of a variable, say $v$, then the propagation algorithms of all other constraints involving $v$ are triggered. 
For instance, when the propagator of \closedx~modifies the domain of item variables, these modifications must be propagated to XOR constraints. Now, we detail the filtering algorithm of the XOR constraints. This propagator is based on the method of {\it Gaussian elimination} \cite{GomesHSS07}, a classical algorithm for solving systems of linear equations. All coefficients $b_{i} \in \{0, 1\}$ of the XOR constraints form a binary matrix $\mathcal{M}$ of size $(n,m+1)$, where $n$ is the number of Individual XOR constraints and $m$ is the number of items in the dataset, the last column $b_{0} \in \{0, 1\}$ represents the parity bit. 
Thus, for each constraint XOR $k$, if a variable $x_i$ is involved in this constraint, then $\mathcal{M}[k,i] = 1$, otherwise $\mathcal{M}[k,i] = 0$. Figure~\ref{propagate-example}.2 shows the augmented matrix representation of the XOR constraints system of Figure~\ref{propagate-example}.1 on a dataset of $5$ items. 

\begin{algorithm}\footnotesize 
	\caption{Filtering for XOR constraints \label{propagate-xor}}
	\algrenewcommand\algorithmicrequire{\textbf{In}:}
	\algrenewcommand\algorithmicensure{\textbf{InOut}:}
	\algrenewcommand\algorithmicforall{\textbf{foreach}}
	
	\algblock[Name]{Begin}{End}
	\begin{algorithmic}[1]
		\Require $\mbox{Coefficient matrix} \mathcal{M}$
		\Ensure Vector of Boolean variables $(x_1\ldots x_n)$	     
		\Begin
		\State {$x^+ \gets \{i \vert x_i=1\}, \, x^- \gets \{i \vert x_i=0\}, \, x^* \gets \{i \vert i \notin x^+ \cup x^-\}$}
		
		\ForAll{($i \in x^+ \ \wedge \mbox{row}\ r \in \mathcal{M}$)} \label{alg:update-M} 
		
		\If{($\mathcal{M}[r][i] = 1$)}
		\State $b_{0}^r \gets$ $1-b_{0}^r$
		\EndIf
		\State $\mathcal{M}[r][i] \gets 0$
		\EndFor 
		
		\If{ \textsc{CheckInconsistency()}} \Return $fail$
		\ElsIf{($x^* = \emptyset$)} \Return $true$
		\Else 
		\State \textsc{Echlonize($\mathcal{M}$)} \label{alg:echelonize}
		\If{\textsc{CheckInconsistency()}} \Return $fail$
		\Else \ {\sc FixVariables()}  
		\EndIf
		\EndIf
		\State \Return $true$ 
		
		\End
		
		\Function{\sc CheckInconsistency} {  }
		\ForAll{$row$ $r \in \mathcal{M}$} 
		\If{($b_{0}^r = 1$ $\wedge$ $\sum_{i=1}^{n}{\mathcal{M}[r][i]} = 0$)}
		\State \Return $true$
		\EndIf
		\EndFor 
		\State \Return $false$ 
		\EndFunction
		
		\Procedure{\sc FixVariables} {  }
		\ForAll{$row$ $r \in \mathcal{M}$}  		
		\If{($\sum_{i=1}^{n}{\mathcal{M}[r][i]} = 1$)}
		\State $dom(x_i) \gets dom(x_i)-\{1-b_0^r\};$ 
		\State $x^+\gets x^+ \cup \{i\}$; $\vee$ $x^-\gets x^- \cup \{i\}$
		\EndIf
		\EndFor 
		\EndProcedure 	
	\end{algorithmic}
\end{algorithm}

At each step~(see Algorithm~$\ref{propagate-xor}$), the matrix $\mathcal{M}$ 
is updated with the latest variable assignments $x^+$: for each variable $x_i$ set to $1$, 
the parity bit of the constraint where the variable appears is inverted and the coefficient 
of this variable in the matrix $\mathcal{M}$ is set to $0$. 
Figure~\ref{propagate-example}.5 shows this transformation where the parity bit of line $1$ becomes 1 and the coefficient of variable $x_5$ is set $0$. Then, we check if the updated matrix does not lead to an inconsistency, i.e., if a row becomes empty while its right hand side is equal to $1$. In this case, the XOR constraints system 
is unsatisfiable and the current search branch terminates (Figure~\ref{propagate-example}.5). Otherwise, an {\it echelonization\footnote{This operation is performed using the $m4ri$ library.}} operation (line~$\ref{alg:echelonize}$) is performed with the Gauss method to obtain an echelonized matrix : all ones are on or above the main diagonal and all non-zero rows are above any rows of all zero~(Figure~\ref{propagate-example}.3). During echelonization, two situations allow propagation. If a row
becomes empty while its right hand side is equal to $1$, the system is unsatisfiable. If a row contains only one
free variable, it is assigned the right hand side of the row (Figure~\ref{propagate-example}.3). 

%% file: 4-new-features.tex
\begin{algorithm}[t]\footnotesize
	\caption{\letsip~extension (new steps are in bold font) \label{algo:letsip}}
	\algrenewcommand\algorithmicrequire{\textbf{In}:}
	\algrenewcommand\algorithmicensure{\textbf{Parameters}:}
	\algblock[Name]{Begin}{End}
	\begin{algorithmic}[1]
		\Require Dataset $\SDB$, minimal frequency threshold $\theta$
		\Ensure Query of size $k$, number of iterations $T$, query retention $\ell$, range $A$, \flexics~error tolerance $\kappa$, Patterns feature representations $\mathbf{F}$ 
		\Begin
		\State Weight vector $\mathbf{w}^0 \gets \mathbf{0}$, 
		\State Feedback $U \gets \emptyset$, 
		\State Ranked patterns $Q_0^* \gets \emptyset$ 
		\State \Comment Initialization
		\State Ranking function $h_0 \gets \textsc{Logistic}(\mathbf{w}^0, A)$
		\For{$t = 1, 2 \ldots T$} {\Comment \footnotesize{Mine, Interact, Learn, Repeat loop}}
		
		\State $R \gets \textsc{TakeFirst}(Q^*_{t-1}, \ell)$ \label{letSip:samplinga} 
		\State {\Comment \footnotesize{Retain the top $\ell$ patterns from the previous iteration}}
		\State Query $Q_t \gets R \cup \textsc{SamplePatterns}(\SDB, h_{t-1}) \times (k-\card{R})$ \label{letSip:samplingb} \State {\Comment \footnotesize{Call SamplePatterns $(k-\card{R})$ times}}
		\State $\mathcal{P}_{t} \gets \textsc{Rank}(Q_{t})$; $U \gets U \cup \mathcal{P}_{t}$ 
		\State {\Comment \footnotesize{Ask user to rank patterns in $Q_t$}}
		\State $\mathbf{F} \gets \textsc{ConvertToFeaturesVectors}(\mathcal{P}_t, \SDB)$
		\State {\Comment \footnotesize{Convert patterns to feature representations}}
		\State ${p_{disc}} \gets$ {\bf \textsc{LearnDiscriminatingPattern}}$(\mathcal{P}_t)$ \label{letSip:discrPattern} \State {\Comment \footnotesize{Extract discriminating pattern (Algorithm~\ref{algo:acx:learn})}}
		\State $\mathbf{F}^* \gets$ {\bf \textsc{AddDiscrPatternFeature}}$(\mathbf{F}, {p_{disc}}, \SDB)$ \label{letSip:AddDiscPattern}
		\State {\Comment \footnotesize{Update pattern features}}
		\State $\mathbf{w}^t \gets $ {\bf \textsc{Aggregate}}(\textsc{LearnWeights}$(U, \lambda, \mathbf{w}^{t-1}, \mathbf{F}^*))$ \label{letSip:updateWeights} 
		\State {\Comment \footnotesize{Update pattern features weights}}
		\State $\mathbf{F} \gets$ {\bf \textsc{RemoveDiscrPatternFeature}}$(\mathbf{F}^*,{p_{disc}},\SDB)$\label{letSip:update:features}
		\State {\Comment \footnotesize{remove discr. pattern feature}}
		
		\State $h_t \gets \textsc{Logistic}(\mathbf{w}^t, A)$
		\EndFor
		\End
		
		\Function{\sc SamplePatterns}{$\SDB$, Sampling weight $\mathbf{w}: L \rightarrow [A,1]$}
		\State $C \gets \textbf{\cdflexics} (\SDB, freq(.) \geq \theta, \mathbf{w}; \kappa)$
		\State \Return {\sc Top}($m$) highest-wheighted patterns from $C$ \label{letSip:top:m}
		\EndFunction
	\end{algorithmic}
\end{algorithm}
\section{Preference learning based on explainable ranking}
\label{sec:technical-contribution}
\subsection{Towards more expressive and learnable pattern representations}

Features for pattern representation involve indicator variables for included items or sub-graphs, or for covered transactions \cite{Dzyuba:letsip,DBLP:journals/sadm/BhuiyanH16}, or pattern length or frequency \cite{Dzyuba:letsip}. The issue is that such features are treated as if they were independent, whether in the logistic function mentioned above, or multiplicative functions \cite{DBLP:conf/icml/Rueping09,DBLP:journals/sadm/BhuiyanH16}. While this allows to identify pattern components that are \emph{globally} interesting for the user, it is impossible to learn relationships such as "the user is interested in item $i_1$ if item $i_3$ is present but item $i_4$ is absent".
In addition, the pattern elements whose inclusion is indicated are defined before-hand, and the user feedback has no influence on them.
We therefore propose to learn more expressive features  in order to improve the learning of user preferences. 

In this work, we propose to consider {\it discriminating sub-patterns} that better capture (or explain) these preferences. 
Those features exploit ranking-correlated patterns, \ie patterns that influence the user ranking either by allowing some patterns to be well ranked or the opposite.

\subsection{Interclass Variance}
As explained above, our goal is to mine sub-patterns that discriminate between patterns that have been given a high user ranking and those that received a low one. An intuitive way of modelling this problem consists of considering the numerical ranks given to individual patterns as numerical labels and the mining setting as akin to regression. We are not aiming to build a full regression model but only to mine an individual pattern that correlates with the numerical label. 
For this purpose, we use  the interclass variance measure as proposed by \cite{morishita}.

\begin{definition}
	\label{def:icv}
	Let $\mathcal{P}$ be a set of ranked patterns of size $k$ according to user ranking $U$, and $\mathcal{P}^*$ be a subset of $\mathcal{P}$. 
	We define by $\mu(\mathcal{P}^*)$ the average rank of patterns $p \in \mathcal{P}^*$:
	\begin{equation}
	\label{eqn:muP}
	\mu(\mathcal{P}^*) = \frac{1}{\card{\mathcal{P}^*}}\cdot\sum_{p \in \mathcal{P}^*}r(p), 
	\end{equation}
	where $r(p)$ is the rank of $p$ according to a user feedback.
\end{definition}

\noindent
Let $\mathcal{P}_{y} = \{p \in \mathcal{P} \,\lvert\, p \supseteq y \}$: the subset of patterns $p$ in $\mathcal{P}$ in which the sub-pattern $y$ is present, and $\overline{\mathcal{P}}_{y} = \mathcal{P} \backslash  \mathcal{P}_{y}$.
The interclass variance of the sub-pattern $y$ is defined by:

\begin{equation}
\label{eqn:icv}
\icv(y) = \card{\mathcal{P}_{y}} \cdot (\mu(\mathcal{P}) - \mu(\mathcal{P}_{y}))^2 + \card{\overline{\mathcal{P}}_{y}} \cdot  (\mu(\mathcal{P}) - \mu(\overline{\mathcal{P}}_{y}))^2
\end{equation}

The sub-pattern $y$ having the largest \icv~is then the one which is the most strongly correlated with the user ranking of the patterns.

\begin{example}\label{icv-letsip}
	Consider again our example with a query $Q = \{p_1, p_2, p_3, p_4\}$ and the user feedback $U = \{p_2 \succ p_1 \succ p_3 \succ p_4\}$.  Let $\mathcal{P}$ the corresponding ranked patterns w.r.t. $U$. 
	For $y = \{1\}$, we get $\mathcal{P}_{y} = \{p_2\}$ (since $\{1,6\} \supseteq \{1\}$), $\overline{\mathcal{P}}_{y} = \{p_1, p_3, p_4\}$, $\mu(\mathcal{P}_y) = 1$; $\mu(\overline{\mathcal{P}}_y) = \frac{r(p_1)+r(p_3)+r(p_4)}{3}=\frac{(2+3+4)}{3} = 3$ and $\mu(\mathcal{P}) = 2.5$. Applying equation~(\ref{eqn:icv}), we get $\icv(\{1\}) = (2.5-1)^2 + 3\cdot(2.5-3)^2 = 3$. 
	Similarly, $\icv(\{0\})=0.33$. 
\end{example}

\subsection{Learning Discriminating Sub-Patterns}
To find the sub-pattern $y \subseteq p$ with the greatest interclass variance \icv~to be used as a new feature  
($p \in \mathcal{P}$, st. $\mathcal{P}$ being the set of ranked patterns), we systematically search the pattern space spanned by the items involved in patterns in $\mathcal{P}$. 
Semantically, this is the sub-pattern whose presence in one or more itemsets $p$ has influenced their ranking.
So, if $y \subseteq p$, we can say that the ranking of $p$ at the $i^{th}$ position in $\mathcal{P}$ is more likely to be explained by the presence of sub-pattern $y$. 


\begin{algorithm} \footnotesize
	\caption{The \textsc{LearnDiscriminatingPattern} algorithm} \label{algo:acx:learn}
	\algrenewcommand\algorithmicforall{\textbf{foreach}}
	
	\algblock[Name]{Begin}{End}
	\begin{algorithmic}[1]
		\Function{LearnDiscriminatingPattern}{set of ranked patterns $\mathcal{P}$}
		\State	$\icv_{max} \gets 0$, $p_{disc} \gets \emptyset$ 		
		\State  $allItems \gets items(\cup p \in \mathcal{P})$		
		\State  $subPatterns \gets allItems$
		\ForAll{$item \in allItems$}
		\label{loop1_begin}
		\If{$\icv(item) \geq \icv_{max}$}\label{alg:R2:f}
		\State $\icv_{max} \gets \icv(item)$ 
		\State $p_{disc} \gets \{item\}$
		\EndIf
		\EndFor
		\label{loop1_end}
		\ForAll{$elt \in subPatterns$}
		\label{loop2_begin}
		\While{($(item \in allItems \backslash elt) \wedge (\exists p \in  \mathcal{P}\; st.\; \{elt \cup item\} \subseteq p)$)} \label{loop3_begin}
		\If{$\icv(elt \cup item) \geq \icv_{max}$}
		\label{if:newDescriptor}
		\State $\icv_{max} \gets \icv(elt \cup \{item\})$
		\State $p_{disc} \gets elt \cup \{item\}$ 	\label{best:subDescriptor}
		\State $subPatterns \gets subPatternss \cup \,p_{disc}$  \label{dyn:set}
		\EndIf
		\EndWhile
		\EndFor
		\label{loop2_end}
		\State \Return $p_{disc}$ 	\label{return:res}
		\EndFunction
	\end{algorithmic}
\end{algorithm}

Algorithm~\ref{algo:acx:learn} implements the function {\bf \textsc{LearnDiscriminatingPattern}} (see Algorithm~\ref{algo:letsip}, line \ref{letSip:discrPattern}), 
which learns the best discriminating pattern as a feature. 
Its accepts as input a set of ranked patterns $\mathcal{P}$, 
and returns the sub-itemset with the highest ICV. 
Its starts by calculating the \icv~of all items of the patterns $p \in \mathcal{P}$ (loop~\ref{loop1_begin}$-$\ref{loop1_end}). 
Then, it iteratively combines the items to form a larger and finer-grained discriminating pattern~(loop~\ref{loop2_begin}-\ref{loop2_end})~.
Obviously, before combining sub-itemsets, 
we should ensure that the resulting sub-pattern belongs to an existing pattern $p \in \mathcal{P}$ (line~\ref{loop3_begin}). 
If such a pattern exists (line~\ref{if:newDescriptor}), we update the $\icv_{max}$, we save the best discriminating pattern computed so far (i.e. $p_{disc}$, line~\ref{best:subDescriptor}) and we extend the set of sub-itemsets with the $p_{disc}$ pattern for further improvements (line~\ref{dyn:set}). 
Finally, the best discriminating pattern is returned at line~\ref{return:res}.

\begin{proposition}[Time complexity]
	In the worst case, Algorithm \ref{algo:acx:learn} runs in a quadratic time $\mathcal{O}(n^2)$, where $n = \vert items(\cup p \in \mathcal{P}) \vert$.
\end{proposition}

\begin{proof} 
	The first loop runs in $\mathcal{O}(n)$. 
	For the $i^{th}$ iteration of the outer loop (6-14), we have, for a fixed size $subitemsets$, at-most $(n - i)$ iterations for the inner loop. 
	Thus, giving $\mathcal{O}(\frac{n^2}{2})$ for both loops. 
	However, $subPatterns$ evolves with at-most $\vert p_{max}\vert$ elements, since $\vert p_{disc}\vert \leq \vert p_{max}\vert$, which gives additional complexity of $\mathcal{O}(n \times \vert p_{max}\vert)$. Since $\vert p_{max}\vert \leq n$, the worst case complexity is $\mathcal{O}(n^2)$.
\end{proof}

\begin{example}\label{running-example}
	Consider again the set of ranked patterns $\mathcal{P} = \{p_2, p_1, p_3, p_4\}$ w.r.t. $U = \{p_2 \succ p_1 \succ p_3 \succ p_4\}$.  
	After the first loop of Algorithm~\ref{algo:acx:learn} we get $subPatterns = \{0,1,3,4,6\}$, 
	$\icv({1}) = \icv({3}) = 3$, $\icv({0}) = \icv({4}) = 0.33$ and $\icv({6}) = 4$. Thus, $\icv_{max} = 4$ and $p_{disc} = \{6\}$. The second loop~(lines~\ref{loop2_begin}-\ref{loop2_end}) seeks to find the best combination of subitemsets that maximizes the interclass variance value $\icv_{max}$. 
	
	To illustrate this reasoning, let us look at the second iteration of the outer loop, where $elt = \{1\}$ and $allItems = \{0,1,3,4,6\}$. 
	In this case, the inner \textbf{while} loop will only consider 
	items $item \in allItems \backslash \{1\}$, 
	such that $elt \cup \{item\}$ belongs to an existing pattern $p \in  \mathcal{P}$. 
	Thus, we find $item = \{6\}$ for which $elt \cup \{item\} \subseteq p_2$. 
	As $\icv(\{1,6\}) = 3 < \icv_{max}$, $p_{disc}$ remains unchanged.  
	Note that for $item = \{3\}$ or $\{4\}$, $\nexists p \in  \mathcal{P}\; st.\; (\{1\} \cup \{item\}) \subseteq p$. For our example, Algorithm~\ref{algo:acx:learn} returns 
	$p_{disc} = \{6\}$. 
\end{example}

%% file: 5-aggregations.tex
\section{Learning weights for $\phi_{logistic}$ from discriminating sub-patterns}
\label{sec:aggr:func}
\subsection{Exploiting the ICV-based Features}

Learning the discriminating sub-pattern as a new feature in \letsip~brings meaningful knowledge to consider during an interactive preference learning. 
In fact, this sub-pattern correlated with the user's ranking emphasizes the items of interest related to his ranking. Now, we describe how these discriminating patterns can be used in order to improve 
the  learning function $ \varphi_{logistic}$ for patterns. 

A direct way of exploiting discriminating sub-patterns $p_{disc}$ (explaining user rankings) consists of adding them to the feature vector representing patterns. 
However, this increases the size of feature vectors, introduces additional cost and most probably leads to over-fitting and generalization issues of the learning function $\phi_{logistic}$.

Instead, we temporarily add new features representing the discriminating sub-pattern $p_{disc}$, denoted $\mathbf{F}^{disc}$, to the initial feature representation $\mathbf{F}$ of individual patterns only to learn and update the weights of $\mathbf{F}$ (see Algrithm~\ref{algo:letsip}, line~\ref{letSip:AddDiscPattern}):  
\begin{itemize}
	\item $disc^{Patt}$: a binary feature for each pattern $p \in \mathcal{P}$; 
	equals 1 iff the corresponding discriminating sub-pattern $p_{disc}$ belongs to $p$; 
	$$disc^{Patt}(p,p_{disc}) = [p_{disc} \subseteq p]$$,  
	\item $disc^{Freq}$: numerical feature; the frequency of the discriminating sub-pattern belonging to $p$;
	
	$$
	disc^{Freq}(p, p_{disc}) = \left\{
	\begin{array}{ll}
	\freq(p_{disc})/\card{\SDB}\ if\ p_{disc} \subseteq p\\
	0\ otherwise
	\end{array}\right.
	$$

	\item $disc^{Size}$: numerical feature; the relative size of the discriminating sub-pattern belonging to $p$;
	  
	  $$
	  disc^{Size}(p, p_{disc}) = \left\{
	  \begin{array}{ll}
	  \card{p_{disc}}/\card{\items}\ if\ p_{disc} \subseteq p\\
	  0\ otherwise
	  \end{array}\right.
	  $$
\end{itemize}

Let $\mathbf{F}^* = \mathbf{F} \mathbin\Vert \mathbf{F}^{disc}$ the concatenated feature representation st. $\mathbf{F}^{disc} = (disc^{Patt}, disc^{Freq}, disc^{Size})$. As the size of  $\mathbf{F}$ is $\card{\items} + \card{\SDB} + 2$ and the size of $\mathbf{F}^{disc}$ is $3$, thus 
we have in total at most $\card{\items} + \card{\SDB} + 5$ features. 

Now, similar to \letsip, we need to learn a vector of feature weights defining the ranking function $\phi_{logistic}$. 
We proceed in three steps:  
\begin{enumerate}
	\item[(i)] First, we use $\mathbf{F}^*$ to learn the weight vector $\mathbf{w}_{\mathbf{F}}$ associated to  $\mathbf{F}$ (as in \letsip) and the weight vector $\mathbf{w}_{\mathbf{F}^{disc}}$ associated to $\mathbf{F}^{disc}$. 
	
	\item[(ii)] Second, using $\mathbf{w}_{\mathbf{F}^{disc}}$, we update the weight $\mathbf{w}_{f} \in \mathbf{w}_{\mathbf{F}}$ of each feature $f \in \mathbf{F}$ by mean of a multiplicative aggregation function (see Algrithm~\ref{algo:letsip}, line~\ref{letSip:updateWeights}). 
	
	\item[(iii)] Finally, to learn from new discriminating sub-patterns, we remove the features $\mathbf{F}^{disc}$ from $\mathbf{F}^{*}$ (see Algrithm~\ref{algo:letsip}, line~\ref{letSip:update:features}).
\end{enumerate} 

The next section shows how to update the feature weights $\mathbf{w}_{\mathbf{F}}$ using the learned feature weights $\mathbf{w}_{\mathbf{F}^{disc}}$. 

\begin{example} 
\label{exple:icvUsage} 
Let us consider example~\ref{running-example} and the discriminating sub-pattern $p_{disc} = \{6\}$ computed using Algorithm~\ref{algo:acx:learn}. Let us assume that $\mathbf{F}$ represents items and frequency features, 
while $\mathbf{F}^{disc} = (disc^{Patt}, disc^{Freq})$.  
According to example~\ref{weigth:letsip}, pattern $p_1 = \{4, 6\}$ is represented by the feature vector $\mathbf{F}_1 = (0,0,0,0,1,0,1,0.54)$, while pattern $p_3$ is represented by $\mathbf{F}_3 = (1,0,0,0,0,0,0,0.36)$. Now, using 
the two features 
$disc^{Patt}$ and $disc^{Freq}$, we obtain for $p_1$ the full feature vector 
$\mathbf{F}_1^{*} = (0, 0, 0, 0, 1, 0, 1, 0.54, \textbf{1}, \textbf{0.63})$, 
since $\mathbf{F_1}^{disc} = (1, 0.63)$, 
i.e., $disc^{Patt}(p_1,\{6\}) = 1$ and $disc^{Freq}(p_1,\{6\}) = 0.63$. Similarly, for $p_3 = \{0\}$ we get $\mathbf{F}_3^{*} = (1, 0, 0, 0, 0, 0, 0, 0.36, \textbf{0}, \textbf{0.0})$, since $\mathbf{F_3}^{disc} = (0, 0.0)$, i.e., $disc^{Patt}(p_3,\{6\}) = 0$, and $disc^F(p_3,\{6\}) = 0.0$. 
\end{example}

\subsection{Updating feature weights of patterns}
\label{sec:update:agg}

Given a complete features representation $\mathbf{F}^* = \mathbf{F} \mathbin\Vert \mathbf{F}^{disc}$ for which we learned two weight vectors: $\mathbf{w}_{\mathbf{F}}$ for features in $\mathbf{F}$ 
and $\mathbf{w}_{\mathbf{F}^{disc}}$ for new features in $\mathbf{F}^{disc}$, i.e., $disc^{Patt}$, $disc^{Sup}$, and  $disc^{Size}$.
Here, we show how to combine the weight vector $\mathbf{w}_{\mathbf{F}^{disc}}$, so as to reveal features (items, transactions, ...) in $\mathbf{F}$ that are seen as making better ranking prediction than others. 

Unlike \letsip, our approach aims to identify those features in terms of discriminating sub-pattern $p_{disc}$ and seeks to reward (or penalize) features in $\mathbf{F}$ by updating their current weight $\mathbf{w}_f$ ($f \in \mathbf{F}$) using a multiplicative aggregation function. 
Let $p_{disc}$ be the discriminating sub-pattern and let 
$p \in \mathcal{P}$. For a given feature $f_{disc} \in \mathbf{F}^{disc}$ and its associated learned weight $\mathbf{w}_{f_{disc}}$, the weight $\mathbf{w}_f$ is updated as follows: 
\begin{itemize}
	\item $(f_d \equiv disc^{Patt})$
	\vspace*{-0.2cm}
	\begin{itemize}
		\item for each \textit{item} feature $f_i \in \mathbf{F}$ s.t. $i \in p_{disc}\ \wedge$\ $p_{disc} \subseteq p$, $\mathbf{w}_{f_{i}} = \mathbf{w}_{f_{i}} \cdot \Phi(\mathbf{w}_{f_{d}})$
		\item for each \textit{transaction} feature $f_t \in \mathbf{F}$ s.t. $ t \in \cover(p_{disc}) \cap \cover(p)$, $\mathbf{w}_{f_{t}} = \mathbf{w}_{f_{t}} \cdot \Phi(\mathbf{w}_{f_{d}})$
	\end{itemize}
	
	\item $(f_d \equiv disc^{Sup} \vee f_d \equiv disc^{Size})$
	\vspace*{-0.2cm}\begin{itemize}
		\item for the \textit{frequency} feature $f_{sup} \in \mathbf{F}$ s.t. $p_{disc} \subseteq p$, $\mathbf{w}_{f_{sup}} = \mathbf{w}_{f_{sup}} \cdot \Phi(\mathbf{w}_{f_{d}})$
		\item for the \textit{size} feature $f_{size} \in \mathbf{F}$ s.t. $p_{disc} \subseteq p$, $\mathbf{w}_{f_{size}} = \mathbf{w}_{f_{size}} \cdot \Phi(\mathbf{w}_{f_{d}})$
	\end{itemize}
\end{itemize}

Intuitively, this updating procedure works since it assigns higher weights to the features that better explain the user ranking over time, thus increasing the probability of being picked in the next iteration.



\smallskip
\noindent
\textbf{Aggregation functions}
Our approach is inspired by~\cite{DBLP:journals/toc/AroraHK12}, which proposes two multiplicative functions for aggregating the learned weights. The updating scheme described above is performed after each ranking. 
At iteration $t> 0$, for each feature $f \in \mathbf{F} $, its weight $\mathbf{w}_f^{(t)}$ is updated using a multiplicative aggregation function $\Phi$: 
$$\mathbf{w}_f^{(t)} = \mathbf{w}_f^{(t)} \cdot \Phi(m, \eta)$$ 
where $m$ is a reward (or a loss) applied to $f$ and $\eta \in ]0, \frac{1}{2}]$ is parameter. We use two multiplicative factors~\cite{DBLP:journals/toc/AroraHK12}: 
 (1)~a {\it linear factor} $\Phi_1(m, \eta) = 1 + \eta m$; and 
 (2)~an {\it exponential factor} $\Phi_2(m, \eta) = \exp^{\eta m}$, where $m$ stands for $\mathbf{w}_{f_{d}}^{t}$, the learned feature weight in $\mathbf{F}^{disc}$. 
 Thus, we have 
 \begin{equation}
 \label{eqn:aggr:fct}
 \mathbf{w}_f^{(t)} = 
 \left\{
 \begin{array}{l}
 \mathbf{w}_f^{(t)} \cdot (1 + \eta \mathbf{w}_{f_{d}}^{(t)})\\
 \vee \\
 \mathbf{w}_f^{(t)} \cdot \exp^{\eta \mathbf{w}_{f_{d}}^{(t)}} 
 \end{array}
 \right.
 \end{equation}
In our experiments (see Section \ref{sec:exps}), we compare both aggregation functions to update features' weights. 

\begin{example} 
\label{exple:aggregation}
Let us consider example~\ref{exple:icvUsage}. 
At iteration $1$, we get $p_{disc} = \{6\}$, the learned weight vector $\mathbf{w}_\mathbf{F}^{(1)} = (-0.33, 0.99, 0, -0.99, 0.33, 0, 1.33,0.15)$ 
for items and frequency features of patterns (see example~\ref{weigth:letsip}),  
and $\mathbf{w}_{\mathbf{F}^{disc}}^{(1)} = (\mathbf{w}_{disc^{Patt}}^{(1)}, \mathbf{w}_{disc^{Sup}}^{(1)})$ $= (1.33, 0.84)$.
For our example, we consider the linear factor and we set $\eta$ to $0.2$. 
For the items feature, since, $p_{disc} = \{6\} \subseteq p_1$, we have $\mathbf{w}_6^{(1)} = \mathbf{w}_{6}^{(1)} \cdot (1 + \eta \mathbf{w}_{disc^{Patt}}^{(1)})$. 
As $\mathbf{w}_6^{(1)} = 1.33$ and $\mathbf{w}_{disc^{Patt}}^1 = 1.33$, after the updating step, 
we get a new value for $\mathbf{w}_6^{(1)} = 1.68$. 
For the frequency feature, we have $\mathbf{w}_{f_{sup}}^{(1)} = \mathbf{w}_{f_{sup}}^{(1)} \cdot (1 + \eta \mathbf{w}_{disc^{F}}^{(1)})$. 
As $\mathbf{w}_{f_{sup}}^{(1)} = 0.15$ and $\mathbf{w}_{disc^{Sup}}^{(1)} = 0.84$, we get $\mathbf{w}_{f_{sup}}^{(1)} = 0.175$. 
The resulting weights are fed into the $\phi_{logistic}$ function to extract new patterns for the next iteration.
\end{example}

%% file: 6-experimentations.tex
\section{Experiments}
\label{sec:exps}

In this section, we experimentally evaluate our approach for interactive learning user models from feedback. The evaluation focuses on the effectiveness of learning using discriminating sub-patterns and \cdflexics. 
We denote by 
\begin{itemize}
	\item \newletsip~our extension of \letsip~using discriminating sub-patterns; 
	\item \letsipcdf~our extension of \letsip~using \cdflexics;   
	\item \newletsipcdf~our extension of \letsip~using both discriminating sub-patterns and \cdflexics. 
\end{itemize}

We address the following research questions: 
\begin{enumerate}
	\item [$Q_1:$] What effect do LetSIP’s parameters have on the quality of patterns learned by \newletsip? 
	
	\item [$Q_2:$] What effect do LetSIP’s parameters have on the quality of sampled patterns using \cdflexics? 
	
	\item [$Q_3:$] How do our different extensions compare to \letsip? 
\end{enumerate}

\subsection{Evaluation methodology}
Evaluating interactive data mining algorithms is hard, for domain experts are scarce and it is impossible to collect enough data for drawing reliable conclusions. To perform an extensive evaluation, we emulate user feedback using  a (hidden) quality measure $\varphi$, which is not known to the learning algorithm. 
We follow the same protocol used in \cite{Dzyuba:letsip}: for each dataset, a set $\mathcal{S}$ of frequent patterns is mined without prior user knowledge.
We assume that there exists a user ranking $R^*$ on the set $\mathcal{S}$, derived from $\varphi$, i.e. $p \succ q \Leftrightarrow \varphi(p) > \varphi(q)$. Thus, the task is to learn to sample frequent patterns proportional to $\varphi$. As $\varphi$, we use frequency $\freq$, and surprisingness $\surp$, where 
$\surp(p) = \max\{\freq(p) - \prod_{i=1}^{\card{p}} \freq(\{i\}),0\}$. 


\begin{table} \centering
 	\begin{tabular}{|l||r|r|r|r|}
		\hline
 		Dataset         & $\card{\SDB}$ & $\card{\items}$ & $\theta$ &  $\#\mathcal{S}$\\ \hline \hline
 		Anneal          & $812$             & $89$       & $660$    & $149\, 331$       \\ \hline
 		Chess           & $3\, 196$         & $75$       & $2\,014$      & $155118$        \\ \hline
 		German          & $1\, 000$         & $110$      & $350$    & $161\, 858$       \\ \hline
 		Heart-cleveland & $296$             & $95$       & $115$    & $153\, 214$       \\ \hline
 		Heart			& $270$				& $38$		 &	$5$		& $137\,832$				\\ \hline
 		Hepatitis       & $137$             & $68$       & $35$     & $148\, 289$       \\ \hline	
 		Hypo			& $3\, 163$			& $47$		 &	$428$		& $158\,910$					\\ \hline
 		Kr-vs-kp        & $3\, 196$         & $73$       & $2\,014$    & $155118$       \\ \hline
 		Lymph           & $148$             & $68$       & $48$     & $146\, 969$       \\ \hline
 		Mushroom        & $8\, 124$         & $112$      & $813$     & $155657$       \\ \hline
 		Soybean         & $630$             & $50$       & $28$     & $143\, 519$       \\ \hline
 		Vehicle         & $846$				& $58$		 & $46$			&	$154\,847$				\\ \hline
 		Vote            & $435$             & $48$       & $25$     & $142\, 095$       \\ \hline 		
 		Wine			& $178$				& $45$		 & $3$			&	$175\,042$				\\ \hline 
 		Zoo-1           & $101$             & $36$       & $10$     & $151\, 806$       \\ \hline
 	\end{tabular}
\caption{Dataset Characteristics. $\theta$ represents an absolute value.}\label{tab:datasets}
\end{table}

To compare the performance of our approaches (\newletsip, \letsipcdf~and \newletsipcdf) with \letsip, we use \textit{cumulative regret}, which is the difference between the ideal value of a certain measure $M$ and its observed value, summed over iterations for each dataset. At each iteration, we evaluate the regret of ranking pattern $p_i$ by 
$\varphi$. For that, we compute the percentile rank $pct.rank(p_i)$ by $\varphi$ of each pattern $p_{i} \,(1 \leq i \leq k)$ of the query as follows: 
$$
pct.rank(p_i) = \frac{\card{q\in \mathcal{S}, \, \varphi(q) \leq \varphi(p_i)}}{\card{\mathcal{S}}}
$$

Thus, the ideal value is $1$ (e.g., the highest possible value of $\varphi$ over all frequent patterns has the percentile rank of $1$). The regret is then defined as 
$$1 - M_{(1 \leq i \leq k)}(pct.rank(p_i))$$ 
where $M \in \{\max, \avg\}$. We repeat each experiment $10$ times with different random seeds; the average obtained cumulative regret is reported. 
To ensure that all methods are sampled on the same pattern bases, at each iteration of \letsip, the same sets of XOR constraints are generated in both approaches. 

For the evaluation we used UCI data-sets, available at the CP4IM repository\footnote{\url{https://dtai.cs.kuleuven.be/CP4IM/datasets/}}. For each dataset, we set the minimal support threshold such that 
the size of $\mathcal{S}$ is approximately 1,400,000 frequent patterns. Table~\ref{tab:datasets} shows the data set statistics. Each experiment involves $100$ iterations (queries). We use the default values suggested by \cite{Dzyuba:letsip} for the parameters of \letsip: $\lambda = 0.001$, $\kappa = 0.9$, $A = 0.1$ and {\sc Top}($1$).

The implementation used in this article is available online\footnote{The link will be available in the final submission.}. All experiments were conducted as single-threaded runs on AMD Opteron 6174 processors (2.2GHz) with a RAM of 256 GB and a time limit of 24 hours.

\begin{figure}[t]\centering
	\begin{minipage}{.4\textwidth}\centering
		\subfloat[\newletsiplin: varying $k$, $\max.\varphi$\label{fig:top0:max:lin}]{
			\includegraphics[scale=0.3]{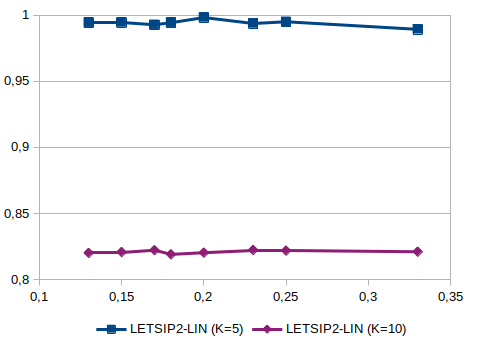}
		}
	\end{minipage} 
	\begin{minipage}{.45\textwidth}
		\centering
		\subfloat[{\newletsipexp: varying $k$, $\max.\varphi$}\label{fig:top0:max:exp}]{
			\includegraphics[scale=0.3]{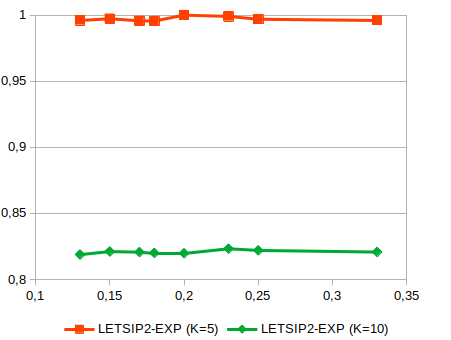}
		} 
	\end{minipage}
	\begin{minipage}{.4\textwidth}\centering
		\subfloat[\newletsiplin: varying $k$, $\avg.\varphi$.\label{fig:top0:avg:lin}]{
			\includegraphics[scale=0.3]{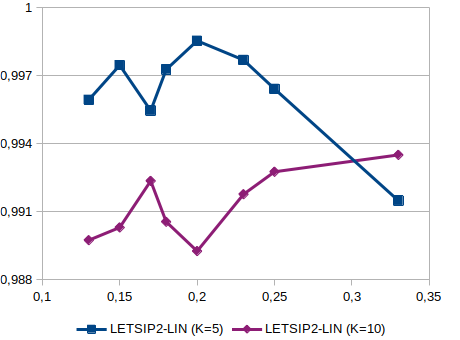}
		}
	\end{minipage} 
	\begin{minipage}{.45\textwidth}
		\centering
		\subfloat[{\newletsipexp: varying $k$, $\avg.\varphi$}\label{fig:top0:avg:exp}]{
			\includegraphics[scale=0.3]{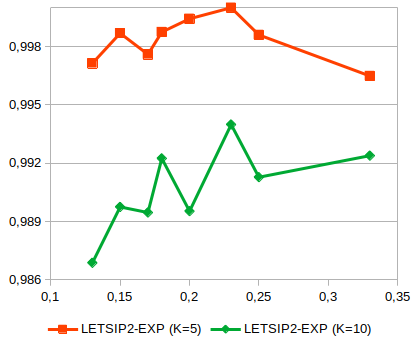}
		} 
	\end{minipage}
\begin{minipage}{.4\textwidth}\centering
	\subfloat[\newlin \ vs. \newexp \ w.r.t. $\max.\varphi$.\label{fig:top0:max:exp:lin}]{
		\includegraphics[scale=0.3]{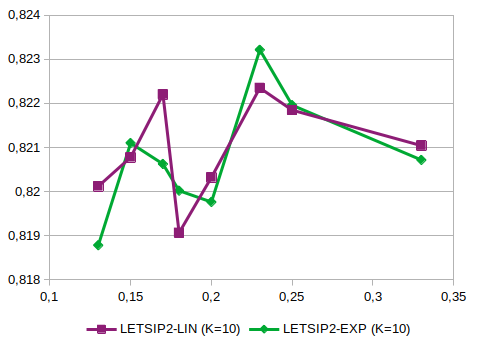}
	}
\end{minipage} 
\begin{minipage}{.45\textwidth}
	\centering
	\subfloat[{\newlin \ vs. \newexp \ w.r.t. $\avg.\varphi$.}\label{fig:top0:avg:exp:lin}]{
		\includegraphics[scale=0.3]{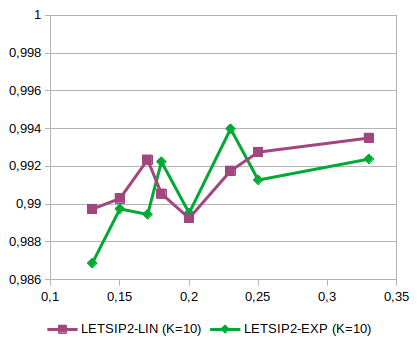}
	} 
\end{minipage}
	\caption{Effect of \newletsip's parameters on regret w.r.t. two performance measures ($\max.\varphi$ and $\avg.\varphi$) and query retention $\ell$ = 0. Results are aggregated over data sets and the three feature combinations I, IT and ILFT. Regret values are normalized to the range [0,1]. \label{fig:letsip2:l:0}}
\end{figure}

\begin{figure}[t]\centering
\begin{minipage}{.4\textwidth}\centering
	\subfloat[\newletsiplin: varying $k$, $\max~\varphi$.\label{fig:top1:max:lin}]{
		\includegraphics[scale=0.3]{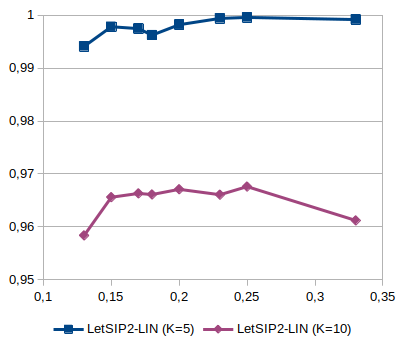}
	}
\end{minipage} 
\begin{minipage}{.45\textwidth}
	\centering
	\subfloat[{\newletsipexp: varying $k$, $\max~\varphi$}\label{fig:top1:max:exp}]{
		\includegraphics[scale=0.3]{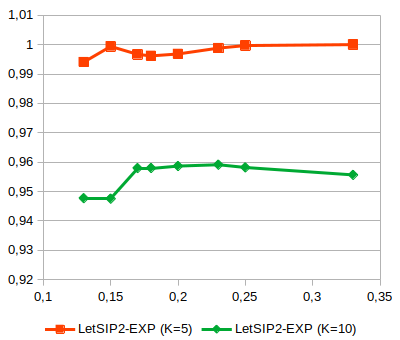}
	} 
\end{minipage}
\begin{minipage}{.4\textwidth}\centering
	\subfloat[\newletsiplin: varying $k$ $\avg~\varphi$.\label{fig:top1:avg:lin}]{
		\includegraphics[scale=0.3]{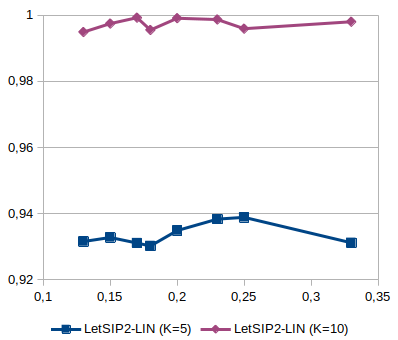}
	}
\end{minipage} 
\begin{minipage}{.45\textwidth}
	\centering
	\subfloat[{\newletsipexp: varying $k$ $\avg~\varphi$}\label{fig:top1:avg:lin:exp}]{
		\includegraphics[scale=0.3]{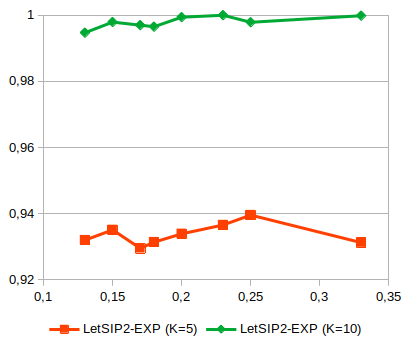}
	} 
\end{minipage}
\begin{minipage}{.4\textwidth}\centering
	\subfloat[\newlin vs. \newexp w.r.t. $\max.\varphi$.\label{fig:top1:max:exp:lin}]{
		\includegraphics[scale=0.3]{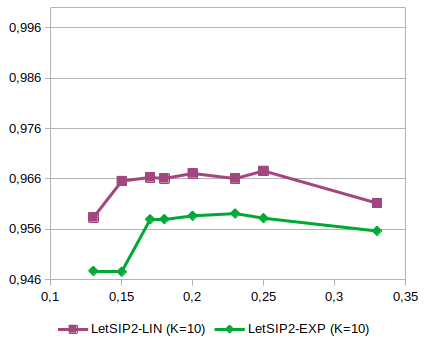}
	}
\end{minipage} 
\begin{minipage}{.45\textwidth}
	\centering
	\subfloat[{\newlin vs. \newexp w.r.t. $\avg.\varphi$.}\label{fig:top1:avg:exp:lin}]{
		\includegraphics[scale=0.3]{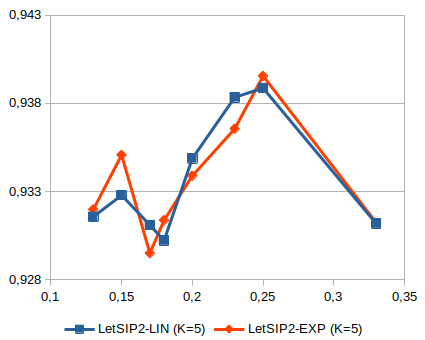}
	} 
\end{minipage}
\caption{Effect of \newletsip's parameters on regret w.r.t. two performance measures ($\max.\varphi$ and $\avg.\varphi$) and query retention $\ell$ = 1. Results are aggregated over data sets and the three feature combinations I, IT and ILFT. Regret values are normalized to the range [0,1]. \label{fig:letsip2:l:1}}
\end{figure}

\subsection{Experimental results}

\subsubsection{Evaluating Aspects of $\newletsip$}
\label{xp:eval:letsip2}

In the first part of the experimental evaluation, we take a look at the effect of different parameter settings on \newletsip, the influence of different pattern features, and how it compares to \letsip{} itself.

\noindent\textbf{Evaluating the influence of parameter settings}
We evaluate the effects of the choice of parameters values on the performance of \newletsip, in particular query size $k \in \{5,10\}$, the aggregation function $\Phi$ and the range $\eta$\footnote{$\eta \in \{0.13, 0.15, 0.17, 0.18, 0.2,$ $0.23, 0.25, 0.33\}$}. We use the following feature combination: $Items$ (I); $Items+Transactions$ (IT); and $Items+Length+Frequency+Transactions$ (ILFT). We consider two settings for parameter $\ell$: $\ell = 0$ and $\ell = 1$. 

Figure~\ref{fig:letsip2:l:0} shows the effect of \newletsip's parameters on regret w.r.t. two  performance measures ($\max.\varphi$ and $\avg.\varphi$) and query retention $\ell$ = 0. 
Figures~\ref{fig:top0:max:lin} and \ref{fig:top0:max:exp} show that both aggregation functions ensure the lowest quality regrets with $k = 10$ w.r.t. the $maximal$ aggregator. This indicates that our approach is able to identify the properties of the target ranking from ordered lists of patterns even when the query size increases. Additionally, the exponential function yields lower regret and $\eta = 0.13$ seems the best value.
Regarding the $average$ quality regret (see Figures~\ref{fig:top0:avg:lin} and \ref{fig:top0:avg:exp}), $k=10$ continues to be the better query size for the exponential aggregation function but for the linear function, this depends on the value of $\eta$. In addition, as Figures~\ref{fig:top0:max:exp:lin} and \ref{fig:top0:avg:exp:lin} show, while the exponential function lead to lower maximal regret with $\eta = 0.13$, the linear function can lead to lower average regret for higher values of $\eta$. 

Figure~\ref{fig:letsip2:l:1} shows the effect of \newletsip's parameters on regret for query retention $\ell =1$. Retaining one highest-ranked pattern from the previous query w.r.t. $maximal$ quality does not affect the conclusions drawn previously: $k = 10$ being the better query size and exponential function results in the lowest regret with $\eta = 0.13$. However, we can see the opposite behaviour w.r.t. $average$ quality (see Figures~\ref{fig:top1:avg:lin} and \ref{fig:top1:avg:lin:exp}): querying $5$ patterns allows attaining low regret values for both functions, linear and exponential. Interestingly, as Figure~\ref{fig:top1:max:exp:lin} shows, the exponential function outperforms linear function on almost all values of $\eta$. Based on these findings, we use the following parameters in the next experiments:  $k= 10$, $aggregation$ = exponential and $\eta = 0.13$. 


\begin{table}[t]
	\begin{center}
		\begin{minipage}{\textwidth}
			\caption{Evaluation of the importance of pattern features and comparison of \newletsipexp~($\eta = 0.13$) vs. \letsip~w.r.t. two performance measures. $k = 10$. 
				Results are aggregated over representative datasets: Anneal, Chess, German, Heart-cleveland, Hepatitis, Kr-vr-vp, Lymph, Mushroom, Soybean, Vote and Zoo-1.}\label{tab:features:comp:letsip2}
			\scalebox{.70}{
				\begin{tabular*}{\textwidth}{@{\extracolsep{\fill}}lccccccccc@{\extracolsep{\fill}}}
					\cmidrule{1-10}%
					& \multicolumn{4}{@{}c@{}}{$\ell=0$ (fully random queries)}  && \multicolumn{4}{@{}c@{}}{$\ell=1$ (retain $top$ $\ell$ patterns from the previous query)} \\\cmidrule{2-5}\cmidrule{7-10}%
					& \multicolumn{2}{@{}c@{}}{Regret: $\max.\varphi$}& \multicolumn{2}{@{}c@{}}{Regret: $\avg.\varphi$}  && \multicolumn{2}{@{}c@{}}{Regret: $\max.\varphi$}& \multicolumn{2}{@{}c@{}}{Regret: $\avg.\varphi$}\\\cmidrule{2-5}\cmidrule{7-10}%
					{Features} & $\letsip$& $\newletsipexp$ & $\letsip$ & $\newletsipexp$ && $\letsip$ & $\newletsipexp$ & $\letsip$ & $\newletsipexp$ \\
					\cmidrule{2-5}\cmidrule{7-10}
					I  & {\bf 269.419}& {271.434}    & {617.328} & {\bf 615.994} && {\bf 196.655} & {200.975} & {\bf 574.396} & {575.46}\\
					IT & {265.718 }   & {\bf 251.613}& {610.294} & {\bf 566.790} && {194.396} & {1\bf 90.524} & {563.549} & {\bf 534.723}\\
					ITLF & {263.067}    & {\bf 249.577}& {605.779} & {\bf 561.667} && {196.161} & {\bf 193.554} & {568.200} & {\bf 534.129} \\
					\cmidrule{1-10}%
				\end{tabular*}
			}
		\end{minipage}
	\end{center}
\end{table}

\begin{figure}[t]
	\centering
	
	\begin{tabular}{c}
		\subfloat[][Comparing \newletsipexp{} with \letsip: cumulative regret. ]{\includegraphics[scale=0.2]{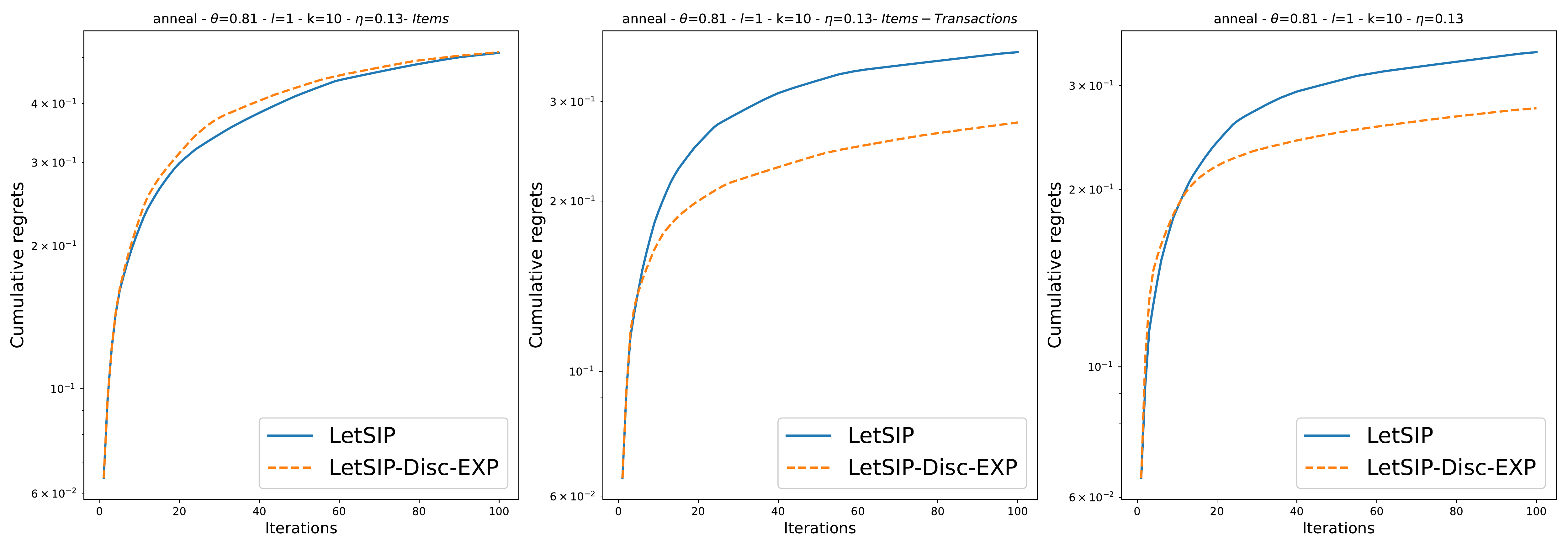}\label{fig:anneal:comp:cum:max}} \\
		\subfloat[][Comparing \newletsipexp{} with \letsip: non cumulative regret. ]{\includegraphics[scale=0.2]{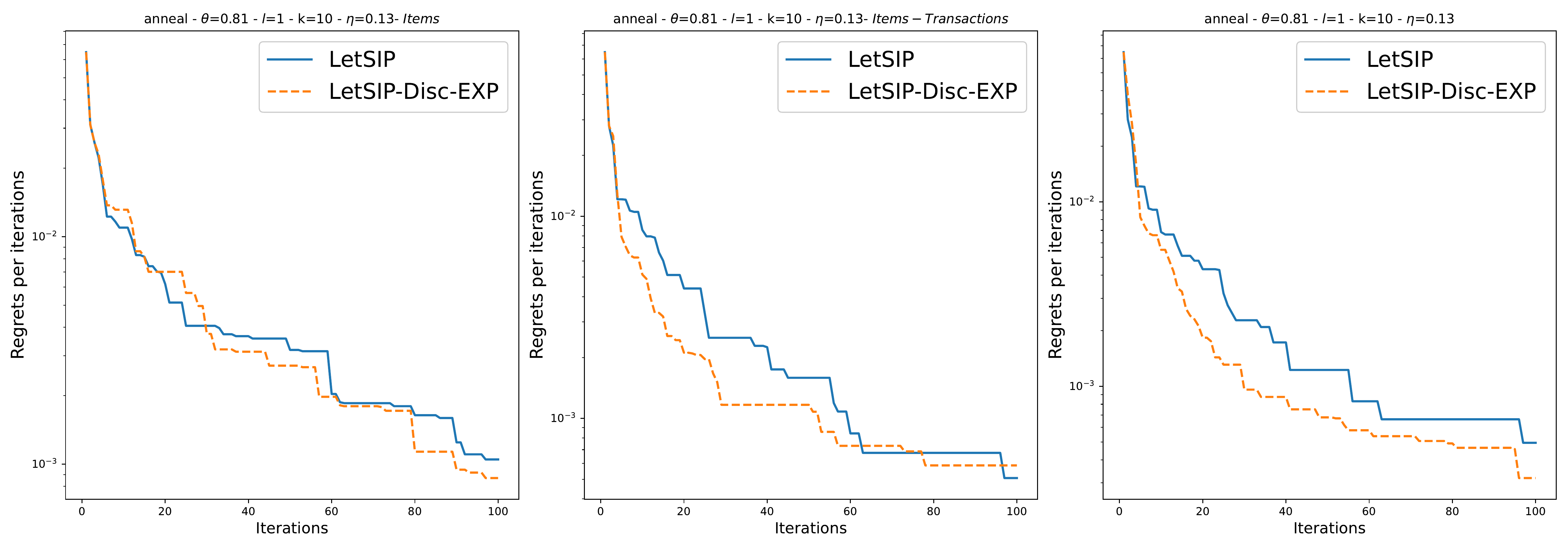}\label{fig:anneal:comp:non:cum:max}}
		
	\end{tabular}
	\caption{Anneal dataset: A detailed view of comparison between \newletsipexp{} and \letsip{} (cumulative and non-cumulative regret) for different pattern features w.r.t. $maximal$ quality, $k = 10$ and $\ell = 1$. Other results are reported in the Appendix~\ref{secA1}.}
	\label{fig:anneal:comp:max:top:0}
\end{figure}

\begin{figure}[t]
	\centering
	
	\begin{tabular}{cc}
		\subfloat[][German-credit]{\includegraphics[height=3.0cm,scale=0.75]{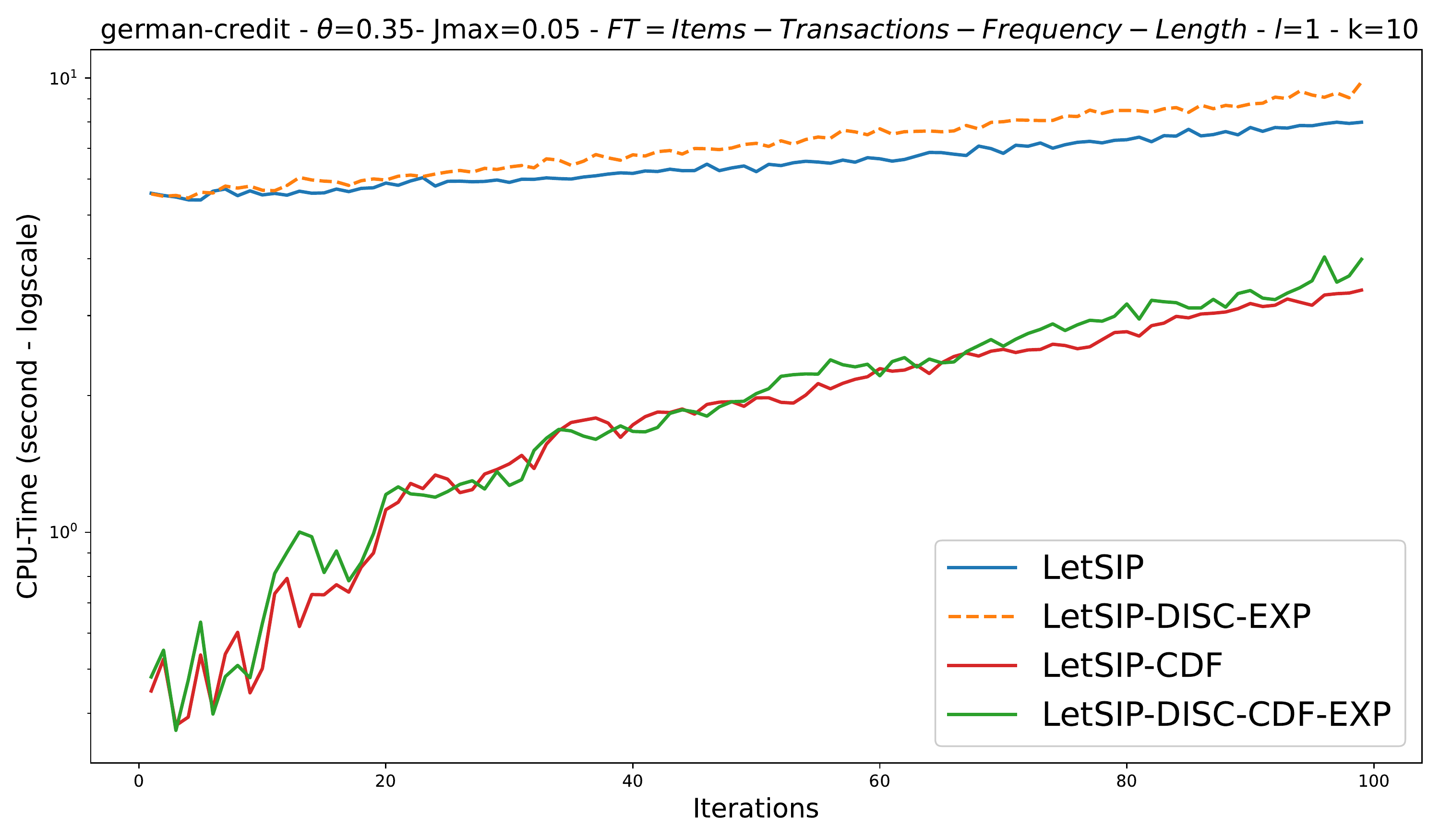}\label{fig:german:max:ILFT:time}}
		&
		\subfloat[][Chess]{\includegraphics[height=3.0cm,scale=0.75]{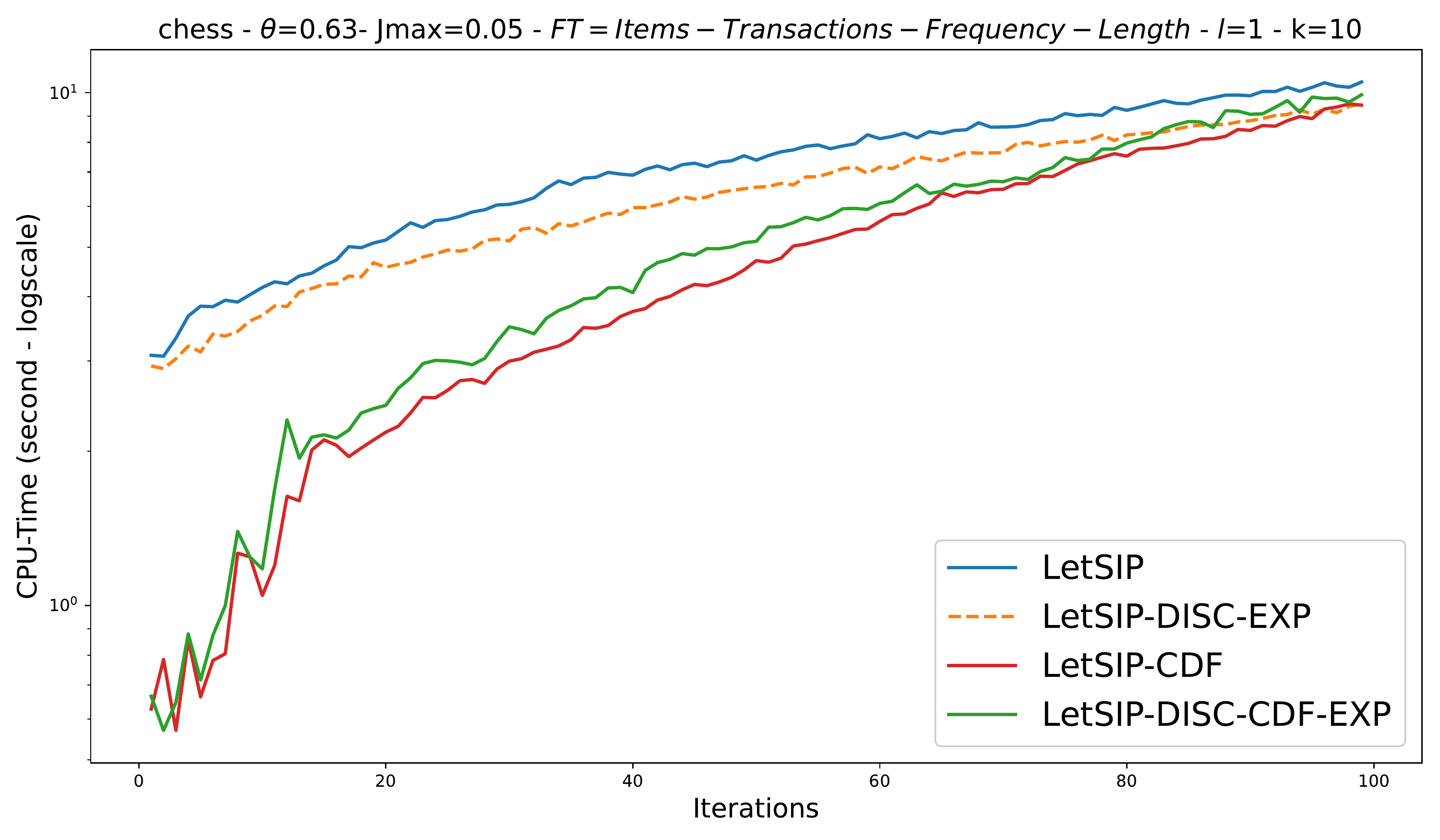}\label{fig:chess:max:ILFT:time}}
		\\ 
		\subfloat[][Heart-cleveland]{\includegraphics[height=3.0cm,scale=0.75]{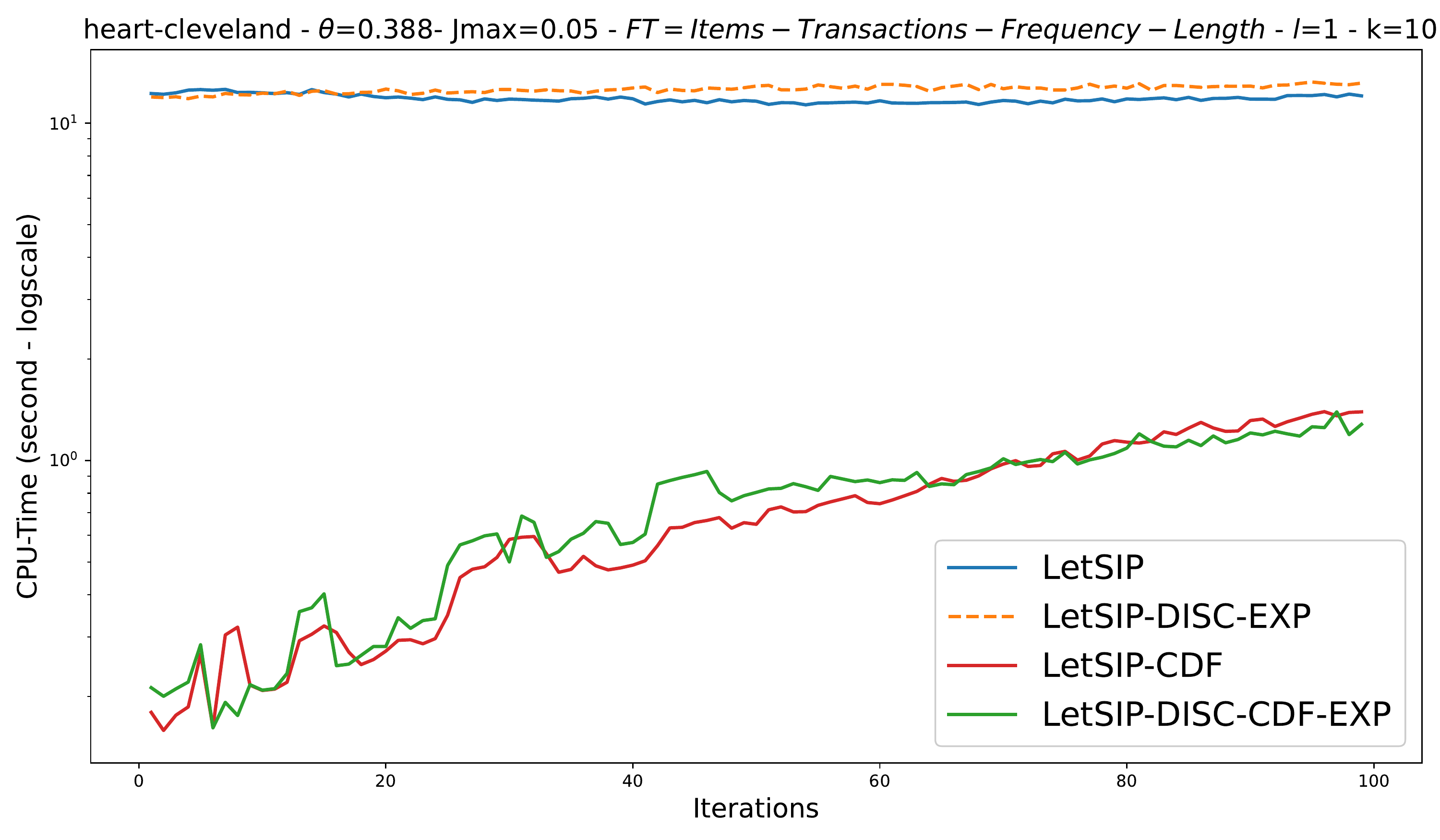}\label{fig:heart-cl:max:ILFT:time}}
		&
		\subfloat[][Zoo-1]{\includegraphics[height=3.0cm,scale=0.75]{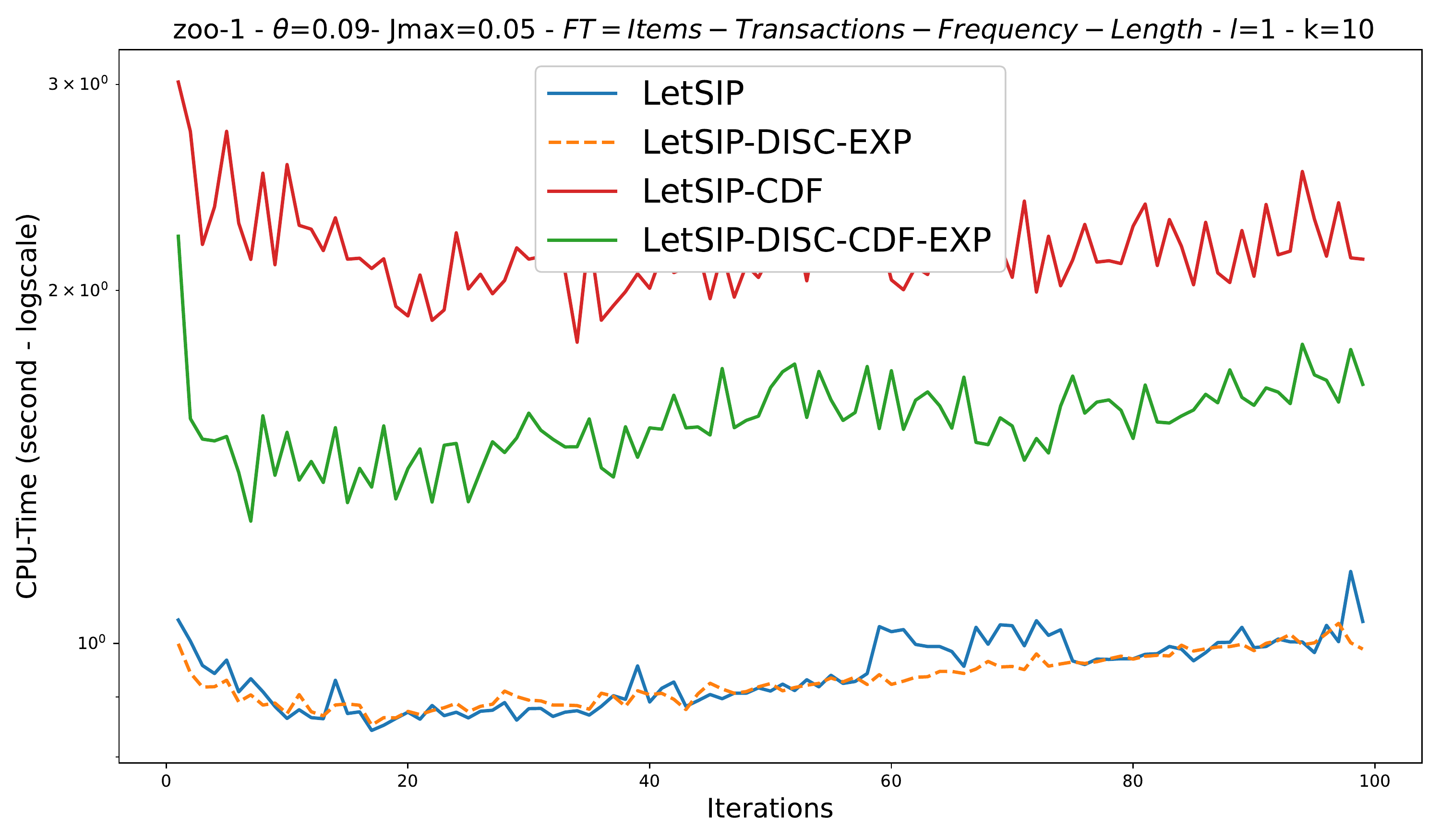}\label{fig:zoo-1:max:ILFT:time}}
	\end{tabular}
	\caption{CPU-time analysis (\textit{sec.}) of \letsip{}, \letsipcdf, \newletsip{} and \newletsipcdfexp{} w.r.t. the ILFT feature combination. $k = 10$ and $\ell = 1$. Other results are reported in the Appendix~\ref{secC1}.}
	\label{fig:comp:time}
\end{figure}

\begin{table}[h]
	\begin{center}
		\begin{minipage}{\textwidth}
			\caption{Evaluation of the importance of pattern features and comparison of \letsipcdf~vs. \letsip~w.r.t. different values of $\jmax$. Results are aggregated over $11$ datasets : Lymph, Hepatitis, Heart-cleveland, Zoo-1, German-credit, Vote, Heart, Hypo, Soybean, Vehicle and Wine.}\label{tab:features:comp:lessipcdf:a}
			\scalebox{.70}{
				\begin{tabular*}{\textwidth}{@{\extracolsep{\fill}}lccccccccc@{\extracolsep{\fill}}}
					\cmidrule{1-10}%
					& \multicolumn{8}{@{}c@{}}{$\jmax=0.05 \quad \wedge \quad \ell=1$} \\\cmidrule{2-10}
					& \multicolumn{4}{@{}c@{}}{$k=5$}  && \multicolumn{4}{@{}c@{}}{$k=10$} \\\cmidrule{2-5}\cmidrule{7-10}%
					& \multicolumn{2}{@{}c@{}}{Regret: $\max.\varphi$}& \multicolumn{2}{@{}c@{}}{Regret: $\avg.\varphi$}  && \multicolumn{2}{@{}c@{}}{Regret: $\max.\varphi$}& \multicolumn{2}{@{}c@{}}{Regret: $\avg.\varphi$}\\\cmidrule{2-5}\cmidrule{7-10}%
					{Features} & $\letsip$& $\letsipcdf$ & $\letsip$ & $\letsipcdf$ && $\letsip$ & $\letsipcdf$ & $\letsip$ & $\letsipcdf$ \\
					\cmidrule{2-5}\cmidrule{7-10}
					I  & {204.595}& {\bf 13.860}    & {424.472} & {\bf 337.069} && {196.444} & {\bf 10.986} & {433.142} & {\bf 376.380}\\
					IT & {202.929}   & {\bf15.334}& {386.053} & {\bf 319.262} && {195.497} & {\bf 11.401} & {394.804} & {\bf 356.897}\\
					ITLF & {206.064}    & {\bf 15.517}& {383.799} & {\bf 320.838} && {197.422} & {\bf 11.521} & {393.147} & {\bf 356.395} \\
					\cmidrule{1-10}
					\cmidrule{1-10}%
					& \multicolumn{8}{@{}c@{}}{$\jmax=0.05 \quad \wedge \quad \ell=0$} \\\cmidrule{2-10}
					& \multicolumn{4}{@{}c@{}}{$k=5$}  && \multicolumn{4}{@{}c@{}}{$k=10$} \\\cmidrule{2-5}\cmidrule{7-10}%
					& \multicolumn{2}{@{}c@{}}{Regret: $\max.\varphi$}& \multicolumn{2}{@{}c@{}}{Regret: $\avg.\varphi$}  && \multicolumn{2}{@{}c@{}}{Regret: $\max.\varphi$}& \multicolumn{2}{@{}c@{}}{Regret: $\avg.\varphi$}\\\cmidrule{2-5}\cmidrule{7-10}%
					{Features} & $\letsip$& $\letsipcdf$ & $\letsip$ & $\letsipcdf$ && $\letsip$ & $\letsipcdf$ & $\letsip$ & $\letsipcdf$ \\
					\cmidrule{2-5}\cmidrule{7-10}
					I  & {278.654}& {\bf 106.825}    & {475.504} & {\bf 418.096} && {236.058} & {\bf 52.951} & {450.757} & {\bf 418.395}\\
					IT & {264.425}   & {\bf103.516}& {428.160} & {\bf 392.522} && {235.397} & {\bf 53.939} & {414.606} & {\bf 395.671}\\
					ITLF & {265.064}    & {\bf 102.076}& {427.891} & {\bf 390.341} && {233.159} & {\bf 54.099} & {411.115} & {\bf 394.695} \\
					\cmidrule{1-10}%
					\cmidrule{1-10}%
					& \multicolumn{8}{@{}c@{}}{$\jmax=0.1 \quad \wedge \quad \ell=1$} \\\cmidrule{2-10}
					& \multicolumn{4}{@{}c@{}}{$k=5$}  && \multicolumn{4}{@{}c@{}}{$k=10$} \\\cmidrule{2-5}\cmidrule{7-10}%
					& \multicolumn{2}{@{}c@{}}{Regret: $\max.\varphi$}& \multicolumn{2}{@{}c@{}}{Regret: $\avg.\varphi$}  && \multicolumn{2}{@{}c@{}}{Regret: $\max.\varphi$}& \multicolumn{2}{@{}c@{}}{Regret: $\avg.\varphi$}\\\cmidrule{2-5}\cmidrule{7-10}%
					{Features} & $\letsip$& $\letsipcdf$ & $\letsip$ & $\letsipcdf$ && $\letsip$ & $\letsipcdf$ & $\letsip$ & $\letsipcdf$ \\
					\cmidrule{2-5}\cmidrule{7-10}
					I  & {204.595}& {\bf 8.647}    & {424.472} & {\bf 348.129} && {196.444} & {\bf 4.360} & {433.142} & {\bf 386.818}\\
					IT & {202.929}   & {\bf 8.577}& {386.053} & {\bf 327.095} && {195.497} & {\bf 5.688} & {394.804} & {\bf 365.700}\\
					ITLF & {206.064}    & {\bf 8.514}& {383.799} & {\bf 327.133} && {197.422} & {\bf 5.595} & {393.147} & {\bf 366.117} \\
					\cmidrule{1-10}%
					\cmidrule{1-10}%
					& \multicolumn{8}{@{}c@{}}{$\jmax=0.1 \quad \wedge \quad \ell=0$} \\\cmidrule{2-10}
					& \multicolumn{4}{@{}c@{}}{$k=5$}  && \multicolumn{4}{@{}c@{}}{$k=10$} \\\cmidrule{2-5}\cmidrule{7-10}%
					& \multicolumn{2}{@{}c@{}}{Regret: $\max.\varphi$}& \multicolumn{2}{@{}c@{}}{Regret: $\avg.\varphi$}  && \multicolumn{2}{@{}c@{}}{Regret: $\max.\varphi$}& \multicolumn{2}{@{}c@{}}{Regret: $\avg.\varphi$}\\\cmidrule{2-5}\cmidrule{7-10}%
					{Features} & $\letsip$& $\letsipcdf$ & $\letsip$ & $\letsipcdf$ && $\letsip$ & $\letsipcdf$ & $\letsip$ & $\letsipcdf$ \\
					\cmidrule{2-5}\cmidrule{7-10}
					I  & {278.654}& {\bf 114.012}    & {475.504} & {\bf 431.636} && {236.058} & {\bf 54.659} & {450.757} & {\bf 430.048}\\
					IT & {264.425}   & {\bf96.029}& {428.160} & {\bf 406.527} && {235.397} & {\bf 43.409} & {414.606} & {\bf 408.651}\\
					ITLF & {265.064}    & {\bf 96.724}& {427.891} & {\bf 407.117} && {233.159} & {\bf 43.419} & {411.115} & {\bf 409.534} \\
					
				\end{tabular*}
			}
		\end{minipage}
	\end{center}
\end{table}
\begin{table}[h]
	\begin{center}
		\begin{minipage}{\textwidth}
			\caption{Evaluation of the importance of pattern features and comparison of \letsipcdf~vs. \letsip~w.r.t. different values of $\jmax$. Results are aggregated over $4$ datasets : Kr-vs-kp, Chess, Mushroom and Anneal.}\label{tab:features:comp:lessipcdf:b}
			\scalebox{.70}{
				\begin{tabular*}{\textwidth}{@{\extracolsep{\fill}}lccccccccc@{\extracolsep{\fill}}}
					\cmidrule{1-10}%
					& \multicolumn{8}{@{}c@{}}{$\jmax=0.05 \quad \wedge \quad \ell=1$} \\\cmidrule{2-10}
					& \multicolumn{4}{@{}c@{}}{$k=5$}  && \multicolumn{4}{@{}c@{}}{$k=10$} \\\cmidrule{2-5}\cmidrule{7-10}%
					& \multicolumn{2}{@{}c@{}}{Regret: $\max.\varphi$}& \multicolumn{2}{@{}c@{}}{Regret: $\avg.\varphi$}  && \multicolumn{2}{@{}c@{}}{Regret: $\max.\varphi$}& \multicolumn{2}{@{}c@{}}{Regret: $\avg.\varphi$}\\\cmidrule{2-5}\cmidrule{7-10}%
					{Features} & $\letsip$& $\letsipcdf$ & $\letsip$ & $\letsipcdf$ && $\letsip$ & $\letsipcdf$ & $\letsip$ & $\letsipcdf$ \\
					\cmidrule{2-5}\cmidrule{7-10}
					I  & {\bf 4.846}& {277.351}    & {\bf 153.230} & {338.841} && {\bf 2.052} & {276.587} & {\bf 173.482} & {346.336}\\
					IT & {\bf 3.721}   & {277.083}& {\bf 136.091} & {337.744} && {\bf 1.805} & {276.389} & {\bf 159.646} & {344.876}\\
					ITLF & {\bf 3.813}    & {277.083}& {\bf 139.184} & {337.744} && {\bf 1.765} & {276.395} & {\bf 165.982} & {344.702} \\
					\cmidrule{1-10}
					\cmidrule{1-10}%
					& \multicolumn{8}{@{}c@{}}{$\jmax=0.05 \quad \wedge \quad \ell=0$} \\\cmidrule{2-10}
					& \multicolumn{4}{@{}c@{}}{$k=5$}  && \multicolumn{4}{@{}c@{}}{$k=10$} \\\cmidrule{2-5}\cmidrule{7-10}%
					& \multicolumn{2}{@{}c@{}}{Regret: $\max.\varphi$}& \multicolumn{2}{@{}c@{}}{Regret: $\avg.\varphi$}  && \multicolumn{2}{@{}c@{}}{Regret: $\max.\varphi$}& \multicolumn{2}{@{}c@{}}{Regret: $\avg.\varphi$}\\\cmidrule{2-5}\cmidrule{7-10}%
					{Features} & $\letsip$& $\letsipcdf$ & $\letsip$ & $\letsipcdf$ && $\letsip$ & $\letsipcdf$ & $\letsip$ & $\letsipcdf$ \\
					\cmidrule{2-5}\cmidrule{7-10}
					I  & {\bf 63.471}& {309.137}    & {\bf 193.516} & {353.573} && {\bf 34.480} & {291.099} & {\bf 192.332} & {353.918}\\
					IT & {\bf 52.438}   & {307.783}& {\bf 179.867} & {352.255} && {\bf 28.021} & {291.124} & {\bf 184.793} & {352.418}\\
					ITLF & {\bf 52.554}    & {307.707}& {\bf 179.214} & {352.238} && {\bf 27.751} & {290.265} & {\bf 183.861} & {351.850} \\
					\cmidrule{1-10}%
					\cmidrule{1-10}%
					& \multicolumn{8}{@{}c@{}}{$\jmax=0.1 \quad \wedge \quad \ell=1$} \\\cmidrule{2-10}
					& \multicolumn{4}{@{}c@{}}{$k=5$}  && \multicolumn{4}{@{}c@{}}{$k=10$} \\\cmidrule{2-5}\cmidrule{7-10}%
					& \multicolumn{2}{@{}c@{}}{Regret: $\max.\varphi$}& \multicolumn{2}{@{}c@{}}{Regret: $\avg.\varphi$}  && \multicolumn{2}{@{}c@{}}{Regret: $\max.\varphi$}& \multicolumn{2}{@{}c@{}}{Regret: $\avg.\varphi$}\\\cmidrule{2-5}\cmidrule{7-10}%
					{Features} & $\letsip$& $\letsipcdf$ & $\letsip$ & $\letsipcdf$ && $\letsip$ & $\letsipcdf$ & $\letsip$ & $\letsipcdf$ \\
					\cmidrule{2-5}\cmidrule{7-10}
					I  & {\bf 4.846}& {277.948}    & {\bf 153.230} & {338.357} && {\bf 2.052} & {276.515} & {\bf 173.482} & {345.652}\\
					IT & {\bf 3.721}   & {277.414}& {\bf 136.091} & {336.254} && {\bf 1.805} & {276.432} & {\bf 159.646} & {344.235}\\
					ITLF & {\bf 3.813}    & {277.412}& {\bf 139.184} & {336.137} && {\bf 1.765} & {276.426} & {\bf 165.982} & {344.243} \\
					\cmidrule{1-10}%
					\cmidrule{1-10}%
					& \multicolumn{8}{@{}c@{}}{$\jmax=0.1 \quad \wedge \quad \ell=0$} \\\cmidrule{2-10}
					& \multicolumn{4}{@{}c@{}}{$k=5$}  && \multicolumn{4}{@{}c@{}}{$k=10$} \\\cmidrule{2-5}\cmidrule{7-10}%
					& \multicolumn{2}{@{}c@{}}{Regret: $\max.\varphi$}& \multicolumn{2}{@{}c@{}}{Regret: $\avg.\varphi$}  && \multicolumn{2}{@{}c@{}}{Regret: $\max.\varphi$}& \multicolumn{2}{@{}c@{}}{Regret: $\avg.\varphi$}\\\cmidrule{2-5}\cmidrule{7-10}%
					{Features} & $\letsip$& $\letsipcdf$ & $\letsip$ & $\letsipcdf$ && $\letsip$ & $\letsipcdf$ & $\letsip$ & $\letsipcdf$ \\
					\cmidrule{2-5}\cmidrule{7-10}
					I  & {\bf 63.471}& {309.973}    & {\bf 193.516} & {353.422} && {\bf 34.480} & {291.349} & {\bf 192.332} & {353.411}\\
					IT & {\bf 52.438}   & {305.787}& {\bf 179.867} & {352.103} && {\bf 28.021} & {290.251} & {\bf 184.793} & {352.371}\\
					ITLF & {\bf 52.554}    & {306.007}& {\bf 179.214} & {352.039} && {\bf 27.751} & {290.893} & {\bf 183.861} & {352.192} \\
					\cmidrule{1-10}%
					
				\end{tabular*}
			}
		\end{minipage}
	\end{center}
\end{table}

\smallskip
\noindent
\textbf{Evaluating the importance of pattern features}. In order to evaluate the importance of various feature sets we performed the following procedure. We construct feature representations of patterns incrementally, i.e. we start with an empty representation and add feature sets one by one, based on the improvement of learning accurate pattern rankings that they enable. Table~\ref{tab:features:comp:letsip2} shows the results for two settings of query retention $\ell$. As we can observe, additional features provide valuable 
information to learn more accurate pattern rankings, particularly for \newletsipexp{} where the regret values all decrease. 
These results are consistent with those observed in \cite{Dzyuba:letsip}. However, the importance of features depends on the pattern type and the target measure $\varphi$~\cite{DzyubaLNR14}. For surprising pattern mining, $Length$ is the most likely to be included in the best feature set, because long patterns tend to have higher values of Surprisingness. $Items$ are important as well, because individual item frequencies are directly included in the formula of Surprisingness. $Transactions$ are important because this feature set helps capture interactions between other features, albeit indirectly. 

\begin{figure}[t]
	\centering
	
	\begin{tabular}{c}
		\subfloat[][Comparing \newletsipexp{} with \letsip: cumulative regret. ]{\includegraphics[scale=0.2]{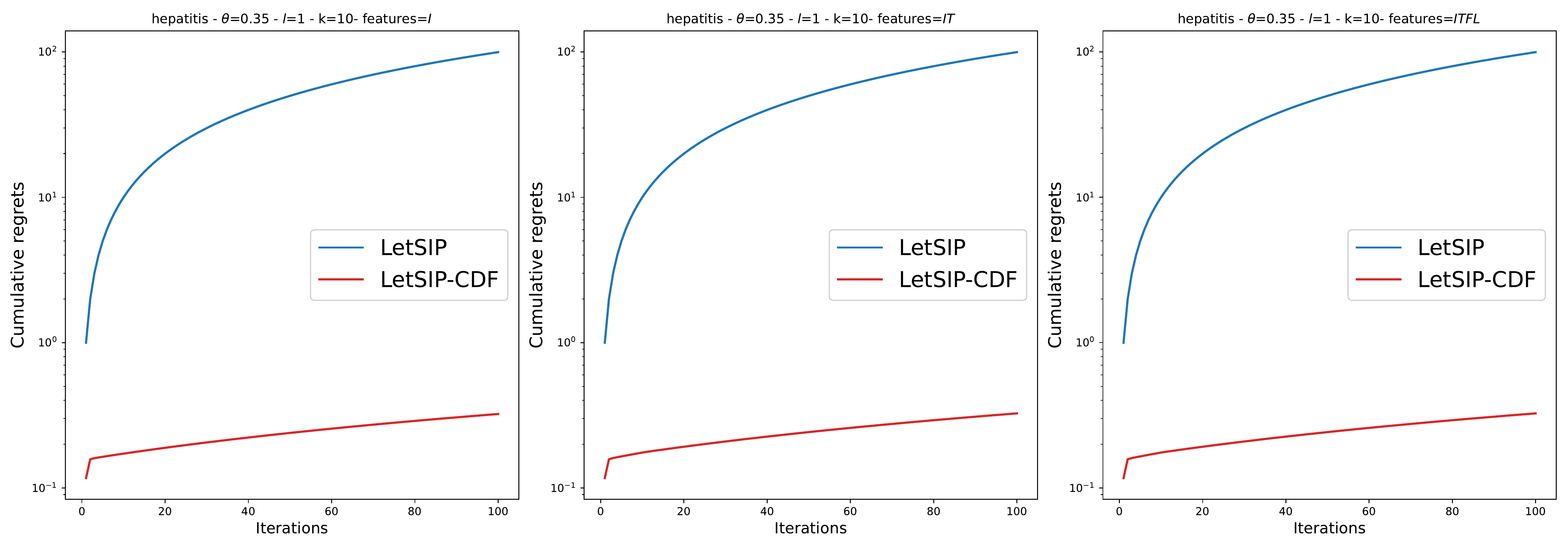}\label{fig:hepatitis:comp:cum:max:2}} \\
		\subfloat[][Comparing \newletsipexp{} with \letsip: non cumulative regret. ]{\includegraphics[scale=0.2]{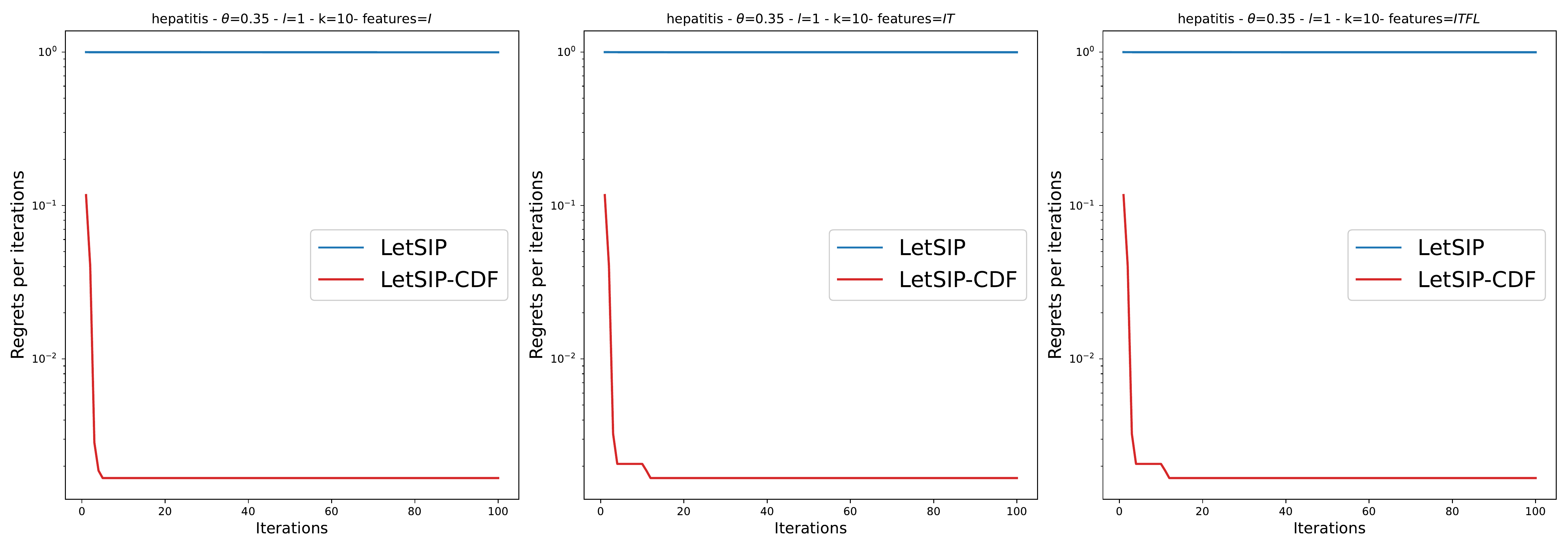}\label{fig:hepatitis:comp:non:cum:max:2}}
		
	\end{tabular}
	\caption{Hepatitis dataset: A detailed view of comparison between \letsipcdf{} and \letsip{} (cumulative and non-cumulative regret) for different pattern features w.r.t. $maximal$ quality, $k = 10$ and $\ell = 1$. Other results are reported in the Appendix~\ref{secB1}.}
	\label{fig:hepatitis:comp:max:2}
\end{figure}


\smallskip
\noindent
\textbf{Comparing quality regrets of \newletsip~with \letsip}. \cite{Dzyuba:letsip} showed that \letsip~outperforms two state-of-the-art interactive methods,~\aple~\cite{DzyubaLNR14}, another approach based on active preference learning to learn a linear ranking function using \ranksvm, and \ipm~\cite{DBLP:journals/sadm/BhuiyanH16}, an MCMC-based interactive sampling framework. Thus, we compare \newletsipexp{} with \letsip. We use query size $k$ and feature representations identical to \newletsipexp{} and we vary the query retention $\ell$. Table~\ref{tab:features:comp:letsip2} reports the regret values. When considering $Items$ as feature representation for patterns, \letsip{} performs the best w.r.t. $maximal$ quality. However, selecting queries uniformly at random allows \newletsipexp{} to ensure the lowest quality regrets w.r.t. $average$ quality. Regarding the other features  (IT and ITLF), \newletsipexp{} always outperforms \letsip{} whatever the value of the query retention $\ell$ and the performance measure considered. 
On the other hand, we observe that the improvements are more pronounced w.r.t. $average$ regret. These results indicate that the learned ranks by \newletsip{} of the $k$ selected patterns of each query in the target ranking are more accurate compared to those learned by \letsip; average values of surprisingness increase substantially compared to \letsip{} rankings. 

Finally, retaining one highest-ranked pattern ($\ell = 1$) results in the lowest regret whatever the feature representations of patterns. As explained in \cite{Dzyuba:letsip}, fully random queries ($\ell = 0$) do not 
enable sufficient exploitation for learning accurate weights.

Figure~\ref{fig:anneal:comp:max:top:0} presents a detailed view of the comparison between \newletsipcdfexp{} and \letsip{} on Anneal dataset (other results are given in the Appendix~\ref{secA1}). Curves show the evolution of the regret (cumulative and non cumulative) in the resulting frequent patterns over $100$ iterations of learning. The results confirm that the learned ranking function of \newletsipexp{} have the capacity to identify frequent patterns with 
higher quality (i.e., of lowest regrets).

\smallskip
\noindent
\textbf{CPU-times analysis}. Finally, regarding the CPU-times, Figure~\ref{fig:comp:time} compares the performance of \newletsipexp{} and \letsip{} on German-credit, Chess, Heart-Cleveland and Zoo-1 datasets. Overall, learning more complex descriptors from the user-defined pattern ranking does not add significantly to the runtimes of our approach. \newletsipexp{} clearly beats \letsip{} on four datasets (i.e., Chess, Heart-cleveland, Kr-vs-kp, and Soybean) and performs similarly on three other ones (see complementary results in the Appendix~\ref{secC1}), while \letsip{} performs better on three datasets (i.e., German-credit, Hepatitis, Zoo). 

On average, one iteration of \newletsipexp{} takes 8.84s (8.81s for \letsip) on AMD Opteron 6174 with 2.2 GHz. The largest proportion of LetSIP’s runtime costs is associated with the sampling component (costs of weight learning are low). As it will be shown in the next Section, exploiting \cdflexics{} within \letsip{} enables to improve the running times. 

\subsubsection{Evaluating Aspects of $\letsipcdf$}
\label{xp:eval:letsipcdf}
We investigate the effects of the choice of features and parameter values on the performance of \letsipcdf{} and we compare it to the original \letsip. We use the same protocol as in Section~\ref{xp:eval:letsip2}. We consider two settings for parameter $\jmax \in \{0.05, 0.1\}$. Tables~\ref{tab:features:comp:lessipcdf:a} and~\ref{tab:features:comp:lessipcdf:b} show detailed results. 

As observed previously, in almost of the cases, adding new features allow to decrease the value of the regret. Interestingly, increasing the query size decreases the maximal quality regret about twofold, while for the average quality regret $k = 5$ seems to be the best setting. 
Regarding the parameter $\jmax$, the performance depends on the target measure $\varphi$: $\jmax = 0.1$ results in the lowest regret with respect to the $maximal$ quality, while  $\jmax = 0.05$ ensures the lowest $average$ quality regrets. Finally, for the choice of values for query retention $\ell$, we observe the same conclusions drawn previously: $\ell = 1$ being the better query retention value. 

\smallskip
\noindent
\textbf{Comparing quality regrets of \letsipcdf~with \letsip}. \letsipcdf{} exhibits two different behaviours.  On $11$  datasets (see Table~\ref{tab:features:comp:lessipcdf:a}), the regret of \letsipcdf{} is substantially lower than that of \letsip, particularly for the $maximal$ quality and for $\ell = 1$ where the decrease is very impressive (an average gain of $94\%$). However, it is also results in the highest quality regrets on four datasets (see Table~\ref{tab:features:comp:lessipcdf:b}), which negates the previous results.


An analysis of the results on these four data sets shows us that $\closedx$ mines very few patterns for each query. This is due to the use of a rather low diversity threshold (\jmax=0.05) which considerably reduces the many redundancies in these data sets. The trade-off is, however, that it becomes extremely difficult to improve the regret given the few patterns.

Figure~\ref{fig:hepatitis:comp:max:2} presents a detailed view of the comparison between \letsipcdf{} and \letsip{} (cumulative and non-cumulative regret) on Hepatitis dataset (other results are given in the Appendix~\ref{secB1}). 
These results confirm once again the interest of ensuring higher diversity for learning accurate weights.

\smallskip
\noindent
\textbf{CPU-times analysis}. Figure~\ref{fig:comp:time} also compares the performance of \letsipcdf{} and \letsip{} on German-credit, Chess, Heart-Cleveland and Zoo-1 data sets. Other results are given in the Appendix~\ref{secC1}. Clearly, \letsipcdf{} is more efficient than \letsip{}. This is explainable by the large reduction of the size of each cell performed by \closedx{} to avoid the mining of non diverse patterns, which speeds exploration up. The only exception are the two data sets Zoo-1 and Vote where \letsip{} is faster. When comparing to \newletsipcdfexp{}, we can observe that \letsipcdf{} remains better, except again for the two data sets Zoo-1 and Vote.

\subsubsection{Evaluating Aspects of $\newletsipcdf$}
\label{xp:eval:newletsipcdf}
Our last experiments evaluates the combination of our two improvements to \letsip{}, i.e., the improved sampling of \cdflexics{} and the more complex features learned dynamically during the iterative mining process. We report detailed results for different features combination, and different parameter settings (query size $k$, query retention $\ell$, and 
$\eta$). We only show the results for $\jmax = 0.05$. 
As previously, \newletsipcdf{} clearly beats \letsipcdf{} on $11$ datasets (see Tables~\ref{tab:features:comp:newletsipcdf:a}) whatever the feature representations of patterns and the value of parameters $k$ and $\eta$. When comparing to \letsipcdf, the  choice of values for query retention $\ell$ clearly affects \newletsipcdf. Fully random queries ($\ell = 0$) seems to favour \newletsipcdf{}, yielding  small improvements. However, both aggregation functions \aggregexp{} and \aggregsum{} get similar results with a slight advantage for \aggregexp. 
The results of \newletsipcdf{} with $\ell = 1$ are comparable with those of \letsipcdf, but we can notice a very slight advantage to \newletsipcdf{} with ILFT as feature representation for patterns. As for \letsipcdf, \newletsipcdf{} 
performs badly on four datasets (see Table~\ref{tab:features:comp:newletsipcdf:a}). 

\smallskip
\noindent
\textbf{CPU-times analysis}. When analysing the running times (see  Figure~\ref{fig:comp:time}), we can see that \letsipcdf{} and \newletsipcdf{} perform very similarly on most of the instances, except on two datasets (Kr-vr-kp and Zoo-1) where \newletsipcdf{} is faster.

\begin{landscape}
\input{regret_LetSIP_vs_LetSIP-cdf_vs_newLetSIP-cdf}
\end{landscape}

%% file: regret_LetSIP_vs_LetSIP-cdf_vs_newLetSIP-cdf.tex
\begin{table}[h]
	\begin{center}
		\begin{minipage}{\textwidth}
			\caption{Evaluation of the importance of pattern features and comparison of \newletsipcdf vs \letsipcdf~vs. \letsip~w.r.t. different values of $\jmax$. Results are aggregated over $11$ datasets : Lymph, Hepatitis, Heart-cleveland, Zoo, German-credit, Soybean, Vote, Heart, Hypo, Vehicle and Wine.}\label{tab:features:comp:newletsipcdf:a}
			\scalebox{.5}{
				\begin{tabular*}{\textwidth}{@{\extracolsep{\fill}}llccccccccccccccccccc@{\extracolsep{\fill}}}
					\cmidrule{1-21}%
					& & \multicolumn{19}{@{}c@{}}{$\jmax=0.05 \quad \wedge \quad \ell=1$} \\\cmidrule{3-21}
					& & \multicolumn{9}{@{}c@{}}{$k=5$}  && \multicolumn{9}{@{}c@{}}{$k=10$} \\\cmidrule{3-11}\cmidrule{13-21}%
					& & \multicolumn{4}{@{}c@{}}{Regret: $\max.\varphi$} && \multicolumn{4}{@{}c@{}}{Regret: $\avg.\varphi$}  && \multicolumn{4}{@{}c@{}}{Regret: $\max.\varphi$} && \multicolumn{4}{@{}c@{}}{Regret: $\avg.\varphi$}\\\cmidrule{3-6}\cmidrule{8-11}\cmidrule{13-16}\cmidrule{18-21}%
					\multirow{2}*{Features} & \multirow{2}*{$\eta$} & \multirow{2}*{$\letsip$} & \multirow{2}*{$\letsipcdf$} & \multicolumn{2}{@{}c@{}}{$\newletsipcdf$} && \multirow{2}*{$\letsip$} & \multirow{2}*{$\letsipcdf$} & \multicolumn{2}{@{}c@{}}{$\newletsipcdf$} && \multirow{2}*{$\letsip$} & \multirow{2}*{$\letsipcdf$} & \multicolumn{2}{@{}c@{}}{$\newletsipcdf$} && \multirow{2}*{$\letsip$} & \multirow{2}*{$\letsipcdf$} & \multicolumn{2}{@{}c@{}}{$\newletsipcdf$} \\
					\cmidrule{5-6}\cmidrule{10-11}\cmidrule{15-16}\cmidrule{20-21}
					
					& & & & $\aggregexp$ & $\aggregsum$ && & & $\aggregexp$ & $\aggregsum$ && & & $\aggregexp$ & $\aggregsum$ && & & $\aggregexp$ & $\aggregsum$ \\
					\cmidrule{3-11}\cmidrule{13-21}
					
					\multirow{3}*{I} & 0.25 & \multirow{3}*{204.595} & \multirow{3}*{13.860} & 13.934 & \cellcolor{G5}{\bf 13.806} && \multirow{3}*{424.472} & \multirow{3}*{\bf 337.069} & 341.575 & 340.623 && \multirow{3}*{196.444} & \multirow{3}*{\bf 10.986} & 11.127 & 11.167 && \multirow{3}*{433.142} & \multirow{3}*{\bf 376.380} & 378.948 & 378.690  \\
					
					& 0.2 & & & 14.018 & 14.017 &&  &  & 340.322 & 341.283 &&  &  & 11.229 & 11.122 &&  &  & 379.438 & 379.070 \\
					
					& 0.13 & & & 13.957 & 13.982 &&  &  & 339.773 & 340.131 &&  &  & 11.195 & 11.283 &&  &  & 377.142 & 377.700 \\
					\cmidrule{3-21}
					
					\multirow{3}*{IT} & 0.25 & \multirow{3}*{202.929} & \multirow{3}*{\bf 15.334} & 15.533 & 15.424 && \multirow{3}*{386.053} & \multirow{3}*{319.262} & 320.589 & 320.195 && \multirow{3}*{195.497} & \multirow{3}*{\bf 11.401} & 11.654 & 11.660 && \multirow{3}*{394.804} & \multirow{3}*{\bf 356.897} & 358.164 & 357.460 \\
					
					& 0.2 & &  & 15.857 & 15.798 &&  &  & 320.975 & 321.036 &&  &  & 11.596 & 11.510 &&  &  & 357.211 & 357.257 \\
					
					& 0.13 & &  & 15.698 & 15.404 &&  & & 319.366 & \cellcolor{G5}{\bf 318.690} &&  &  & 11.694 & 11.678 &&  &  & 358.401 & 358.065 \\
					\cmidrule{3-21}
					
					\multirow{3}*{ITLF} & 0.25 & \multirow{3}*{206.064} & \multirow{3}*{15.517} & 15.552 & 15.539 && \multirow{3}*{383.799} & \multirow{3}*{320.838} & {\bf 320.572} & 320.792 && \multirow{3}*{197.422} & \multirow{3}*{11.521} & \cellcolor{G5}{\bf 11.478} & \cellcolor{G5}{\bf 11.475} && \multirow{3}*{393.147} & \multirow{3}*{\bf 356.395} & 359.163 & 357.518 \\
					
					& 0.2 & & & \cellcolor{G5}{\bf 15.373} & 15.568 &&  &  & \cellcolor{G5}{\bf 319.694} & 320.782 &&  & & 11.562 & 11.577 &&  &  & 358.267 & 357.770 \\
					
					& 0.13 & &  & 15.504 & \cellcolor{G5}{\bf 15.432} &&  &  & {\bf 320.250} & 320.435 &&  &  & 11.568 & 11.603 &&  &  & 357.917 & 358.845 \\
					
					\cmidrule{1-21}
					\cmidrule{1-21}%
					
					& & \multicolumn{19}{@{}c@{}}{$\jmax=0.05 \quad \wedge \quad \ell=0$} \\\cmidrule{3-21}
					& & \multicolumn{9}{@{}c@{}}{$k=5$}  && \multicolumn{9}{@{}c@{}}{$k=10$} \\\cmidrule{3-11}\cmidrule{13-21}%
					& & \multicolumn{4}{@{}c@{}}{Regret: $\max.\varphi$} && \multicolumn{4}{@{}c@{}}{Regret: $\avg.\varphi$}  && \multicolumn{4}{@{}c@{}}{Regret: $\max.\varphi$} && \multicolumn{4}{@{}c@{}}{Regret: $\avg.\varphi$}\\\cmidrule{3-6}\cmidrule{8-11}\cmidrule{13-16}\cmidrule{18-21}%
					\multirow{2}*{Features} & \multirow{2}*{$\eta$} & \multirow{2}*{$\letsip$} & \multirow{2}*{$\letsipcdf$} & \multicolumn{2}{@{}c@{}}{$\newletsipcdf$} && \multirow{2}*{$\letsip$} & \multirow{2}*{$\letsipcdf$} & \multicolumn{2}{@{}c@{}}{$\newletsipcdf$} && \multirow{2}*{$\letsip$} & \multirow{2}*{$\letsipcdf$} & \multicolumn{2}{@{}c@{}}{$\newletsipcdf$} && \multirow{2}*{$\letsip$} & \multirow{2}*{$\letsipcdf$} & \multicolumn{2}{@{}c@{}}{$\newletsipcdf$} \\
					\cmidrule{5-6}\cmidrule{10-11}\cmidrule{15-16}\cmidrule{20-21}
					
					& & & & $\aggregexp$ & $\aggregsum$ && & & $\aggregexp$ & $\aggregsum$ && & & $\aggregexp$ & $\aggregsum$ && & & $\aggregexp$ & $\aggregsum$ \\
					\cmidrule{3-11}\cmidrule{13-21}
					
					\multirow{3}*{I} & 0.25 & \multirow{3}*{278.654} & \multirow{3}*{106.825} & {\bf 106.584} & 107.366 && \multirow{3}*{475.504} & \multirow{3}*{\bf 418.096} & 420.378 & 421.057 && \multirow{3}*{236.058} & \multirow{3}*{52.951} & \cellcolor{G5}{\bf 52.131} & \cellcolor{G5}{\bf52.678} && \multirow{3}*{450.757} & \multirow{3}*{\bf 418.395} & 419.840 & 420.121  \\
					
					& 0.2 &  &  & 106.865 & \cellcolor{G5}{\bf 105.939} &&  &  & 420.432 & 419.443 &&  &  & {\bf 52.811} & 53.053 &&  &  & 419.507 & 420.455 \\
					
					& 0.13 &  &  & \cellcolor{G5}{\bf 106.378} & 106.834 &&  &  & 418.817 & 419.686 &&  &  & {\bf 52.768} & 52.980 &&  &  & 419.344 & 419.514 \\
					\cmidrule{3-21}
					
					\multirow{3}*{IT} & 0.25 & \multirow{3}*{264.425} & \multirow{3}*{103.516} & {\bf 102.476} & 102.641 && \multirow{3}*{428.160} & \multirow{3}*{392.522} & {\bf 389.494} & 390.337 && \multirow{3}*{235.397} & \multirow{3}*{53.939} & 54.234 & 54.072 && \multirow{3}*{414.606} & \multirow{3}*{395.671} & \cellcolor{G5}{\bf 394.871} & 395.244 \\
					
					& 0.2 &  &  & \cellcolor{G5}{\bf 100.760} & 100.996 &&  &  & \cellcolor{G5}{\bf 388.681} & \cellcolor{G5}{\bf 388.808} &&  &  & {\bf 53.874} & {\bf 53.459} &&  &  & 395.652 & {\bf 395.075} \\
					
					& 0.13 &  &  & 102.946 & \cellcolor{G5}{\bf 101.420} &&  &  & {\bf 391.117} & {\bf 390.838} &&  &  & \cellcolor{G5}{\bf 53.466} & \cellcolor{G5}{\bf 53.450} &&  &  & 395.346 & \cellcolor{G5}{\bf 394.423} \\
					\cmidrule{3-21}
					
					\multirow{3}*{ITLF} & 0.25 & \multirow{3}*{265.064} & \multirow{3}*{102.076} & 102.838 & 102.197 && \multirow{3}*{427.891} & & 391.805 & 391.119 && \multirow{3}*{233.159} & \multirow{3}*{54.099} & {\bf 54.004} & 54.271 && \multirow{3}*{411.115} & \multirow{3}*{394.695} & 395.034 & 396.046 \\
					
					& 0.2 &  & & {\bf 101.825} & 102.152 &&  & \multirow{1}*{390.341} & {\bf 389.838} & 390.432 &&  &  & \cellcolor{G5}{\bf 53.599} & \cellcolor{G5}{\bf 53.312} &&  & & 394.887 & \cellcolor{G5}{\bf 394.219} \\
					
					& 0.13 &  &  & \cellcolor{G5}{\bf 101.531} & \cellcolor{G5}{\bf 101.638} &&  &  & \cellcolor{G5}{\bf 388.785} & \cellcolor{G5}{\bf 390.178} &&  &  & {\bf 53.764} & {\bf 53.730} &&  & & 394.854 & 395.121 \\
					
					\cmidrule{1-21}
					\cmidrule{1-21}%
					
				\end{tabular*}
			}
		\end{minipage}
	\end{center}
\end{table}


\begin{table}[h]
	\begin{center}
		\begin{minipage}{\textwidth}
			\caption{Evaluation of the importance of pattern features and comparison of \newletsipcdf vs \letsipcdf~vs. \letsip~w.r.t. different values of $\jmax$. Results are aggregated over $4$ datasets : Kr-vs-kp, Chess, Mushroom and Anneal.}\label{tab:features:comp:newletsipcdf:b}
			\scalebox{.5}{
				\begin{tabular*}{\textwidth}{@{\extracolsep{\fill}}llccccccccccccccccccc@{\extracolsep{\fill}}}
					\cmidrule{1-21}%
					& & \multicolumn{19}{@{}c@{}}{$\jmax=0.05 \quad \wedge \quad \ell=1$} \\\cmidrule{3-21}
					& & \multicolumn{9}{@{}c@{}}{$k=5$}  && \multicolumn{9}{@{}c@{}}{$k=10$} \\\cmidrule{3-11}\cmidrule{13-21}%
					& & \multicolumn{4}{@{}c@{}}{Regret: $\max.\varphi$} && \multicolumn{4}{@{}c@{}}{Regret: $\avg.\varphi$}  && \multicolumn{4}{@{}c@{}}{Regret: $\max.\varphi$} && \multicolumn{4}{@{}c@{}}{Regret: $\avg.\varphi$}\\\cmidrule{3-6}\cmidrule{8-11}\cmidrule{13-16}\cmidrule{18-21}%
					\multirow{2}*{Features} & \multirow{2}*{$\eta$} & \multirow{2}*{$\letsip$} & \multirow{2}*{$\letsipcdf$} & \multicolumn{2}{@{}c@{}}{$\newletsipcdf$} && \multirow{2}*{$\letsip$} & \multirow{2}*{$\letsipcdf$} & \multicolumn{2}{@{}c@{}}{$\newletsipcdf$} && \multirow{2}*{$\letsip$} & \multirow{2}*{$\letsipcdf$} & \multicolumn{2}{@{}c@{}}{$\newletsipcdf$} && \multirow{2}*{$\letsip$} & \multirow{2}*{$\letsipcdf$} & \multicolumn{2}{@{}c@{}}{$\newletsipcdf$} \\
					\cmidrule{5-6}\cmidrule{10-11}\cmidrule{15-16}\cmidrule{20-21}
					
					& & & & $\aggregexp$ & $\aggregsum$ && & & $\aggregexp$ & $\aggregsum$ && & & $\aggregexp$ & $\aggregsum$ && & & $\aggregexp$ & $\aggregsum$ \\
					\cmidrule{3-11}\cmidrule{13-21}
					
					\multirow{3}*{I} & 0.25 & \multirow{3}*{\bf 4.846} & \multirow{3}*{277.351} & 277.369 & 277.369 && \multirow{3}*{\bf 153.230} & \multirow{3}*{338.841} & 339.439 & 339.293 && \multirow{3}*{\bf 2.052} & \multirow{3}*{276.587} & 276.624 & 276.626 && \multirow{3}*{\bf 173.482} & \multirow{3}*{346.336} & 346.241 & 346.167  \\
					
					& 0.2 &  & & 277.289 & 277.289 && & & 338.776 & 338.635 && & & 276.653 & 276.653 && & & 346.442 & 346.521 \\
					
					& 0.13 &  & & 277.296 & 277.296 && & & 338.963 & 339.004 && & & 276.767 & 276.767 && & & 346.311 & 346.296 \\
					\cmidrule{3-21}
					
					\multirow{3}*{IT} & 0.25 & \multirow{3}*{\bf 3.721} & \multirow{3}*{277.083} & 277.114 & 277.114 && \multirow{3}*{136.091} & \multirow{3}*{337.744} & 337.976 & 337.976 && \multirow{3}*{\bf 1.805} & \multirow{3}*{276.389} & 276.599 & 276.599 && \multirow{3}*{159.646} & \multirow{3}*{344.876} & 344.324 & 344.294 \\
					
					& 0.2 & & & 277.116 & 277.113 && & & 337.834 & 337.756 && & & 276.559 & 276.559 && & & 344.605 & 344.690 \\
					
					& 0.13 & & & 277.116 & 277.116 && & & 338.160 & 338.160 && & & 276.593 & 276.593 && & & 344.557 & 344.550 \\
					\cmidrule{3-21}
					
					\multirow{3}*{ITLF} & 0.25 & \multirow{3}*{\bf 3.813} & \multirow{3}*{277.083} & 277.114 & 277.114 && \multirow{3}*{139.184} & \multirow{3}*{337.744} & 338.259 & 337.976 && \multirow{3}*{\bf 1.765} & \multirow{3}*{276.395} & 276.596 & 276.596 && \multirow{3}*{165.982} & \multirow{3}*{344.702} & 344.479 & 344.454 \\
					
					& 0.2 & & & 277.116 & 277.116 && & & 337.834 & 337.834 && & & 276.560 & 276.559 && & & 344.497 & 344.646 \\
					
					& 0.13 & & & 277.116 & 277.115 && & & 338.160 & 338.078 && & & 276.594 & 276.594 && & & 344.409 & 344.409 \\
					
					\cmidrule{1-19}
					\cmidrule{1-19}%
					
					& & \multicolumn{19}{@{}c@{}}{$\jmax=0.05 \quad \wedge \quad \ell=0$} \\\cmidrule{3-21}
					& & \multicolumn{9}{@{}c@{}}{$k=5$}  && \multicolumn{9}{@{}c@{}}{$k=10$} \\\cmidrule{3-11}\cmidrule{13-21}%
					& & \multicolumn{4}{@{}c@{}}{Regret: $\max.\varphi$} && \multicolumn{4}{@{}c@{}}{Regret: $\avg.\varphi$}  && \multicolumn{4}{@{}c@{}}{Regret: $\max.\varphi$} && \multicolumn{4}{@{}c@{}}{Regret: $\avg.\varphi$}\\\cmidrule{3-6}\cmidrule{8-11}\cmidrule{13-16}\cmidrule{18-21}%
					\multirow{2}*{Features} & \multirow{2}*{$\eta$} & \multirow{2}*{$\letsip$} & \multirow{2}*{$\letsipcdf$} & \multicolumn{2}{@{}c@{}}{$\newletsipcdf$} && \multirow{2}*{$\letsip$} & \multirow{2}*{$\letsipcdf$} & \multicolumn{2}{@{}c@{}}{$\newletsipcdf$} && \multirow{2}*{$\letsip$} & \multirow{2}*{$\letsipcdf$} & \multicolumn{2}{@{}c@{}}{$\newletsipcdf$} && \multirow{2}*{$\letsip$} & \multirow{2}*{$\letsipcdf$} & \multicolumn{2}{@{}c@{}}{$\newletsipcdf$} \\
					\cmidrule{5-6}\cmidrule{10-11}\cmidrule{15-16}\cmidrule{20-21}
					
					& & & & $\aggregexp$ & $\aggregsum$ && & & $\aggregexp$ & $\aggregsum$ && & & $\aggregexp$ & $\aggregsum$ && & & $\aggregexp$ & $\aggregsum$ \\
					\cmidrule{3-11}\cmidrule{13-21}
					
					\multirow{3}*{I} & 0.25 & \multirow{3}*{\bf 63.471} & \multirow{3}*{309.137} & 310.595 & 311.030 && \multirow{3}*{\bf 193.516} & \multirow{3}*{353.573} & 354.350 & 354.350 && \multirow{3}*{\bf 34.480} & \multirow{3}*{291.099} & 291.640 & 291.393 && \multirow{3}*{\bf 192.332} & \multirow{3}*{353.918} & 353.919 & 353.938  \\
					
					& 0.2 & & & 311.746 & 311.736 && & & 354.446 & 354.463 && & & 290.972 & 290.637 && & & 353.659 & 353.396 \\
					
					& 0.13 & & & 308.817 & 308.978 && & & 353.450 & 353.482 && & & 291.020 & 290.889 && & & 353.562 & 353.507 \\
					\cmidrule{3-21}
					
					\multirow{3}*{IT} & 0.25 & \multirow{3}*{\bf 52.438} & \multirow{3}*{307.783} & 308.256 & 308.041 && \multirow{3}*{179.867} & \multirow{3}*{352.255} & 352.338 & 352.426 && \multirow{3}*{\bf 28.021} & \multirow{3}*{291.124} & 291.428 & 291.454 && \multirow{3}*{\bf 184.793} & \multirow{3}*{352.418} & 352.716 & 352.693 \\
					
					& 0.2 & & & 308.661 & 308.294 && & & 352.516 & 352.405 && & & 290.208 & 290.646 &&  & & 352.482 & 352.413 \\
					
					& 0.13 & & & 307.111 & 307.111 && 179.867 & 352.255 & 351.959 & 351.959 && & & 290.909 & 290.909 && & & 352.092 & 352.092 \\
					\cmidrule{3-21}
					
					\multirow{3}*{ITLF} & 0.25 & \multirow{3}*{\bf 52.554} & \multirow{3}*{307.707} & 308.256 & 308.256 && \multirow{3}*{\bf 179.214} & \multirow{3}*{352.238} & 352.338 & 352.338 && \multirow{3}*{\bf 27.751} & \multirow{3}*{290.265} & 291.221 & 291.277 && \multirow{3}*{\bf 183.861} & \multirow{3}*{351.850} & 352.578 & 352.582 \\
					
					& 0.2 & & & 308.453 & 308.453 && & & 352.395 & 352.395 && & & 290.300 & 290.319 && & & 352.342 & 352.352 \\
					
					& 0.13 & & & 307.111 & 307.111 && & & 351.959 & 351.959 && & & 290.700 & 290.700 && & & 352.057 & 352.057 \\
					
					\cmidrule{1-21}
					\cmidrule{1-21}%

				\end{tabular*}
			}
		\end{minipage}
	\end{center}
\end{table}

%% file: 7-related-work.tex
\section{Related Work}
\label{sec:related_works}

Pattern mining has been faced with the problem of returning too large, and potentially uninteresting result sets since early on.

The first proposed solution consisted of condensed representations \cite{Lakhal,DBLP:conf/ideas/2001:free,DBLP:conf/pkdd/2002:non-derivable} by reducing the redudancy of the covers of patterns.
 Top-$k$ mining~\cite{morishita,WangHLT05} is efficient but results in strongly related, redundant patterns showing a lack of diversity.
 Pattern set mining~\cite{chosen-few,DBLP:conf/kdd/2006:miki,patternteams,DBLP:conf/sdm/RaedtZ07} takes into account the relationships between the patterns, which allows to control redundancy and lead to small sets. It requires that one mines an ideally large pattern set to select from, and typically doesn't involve the user.
 
 Constraint-based pattern mining has been leveraged as  a general, flexible solution to pattern mining for a while \cite{DBLP:conf/kdd/RaedtGN08,schaus2017coversize,}, and has since seen numerous improvements, including closed itemset mining \cite{lazaar2016global} and, eventually mining of diverse itemsets \cite{HienLALLOZ20}.
 
 The most recent proposal to dealing with the question of finding interesting patterns involves the user, via interactive pattern mining \cite{DBLP:conf/icml/Rueping09} often involving sampling \cite{DBLP:journals/sadm/BhuiyanH16}, with \letsip{} \cite{Dzyuba:letsip} one of the end points of this development.

%% file: 8-conclusions.tex
\section{Conclusion}
\label{sec:conclusion}

In this paper, we have proposed an improvement to the state-of-the art of iterative pattern mining: instead of using static low-level features that have been pre-defined before the process starts, our approach learns more complex descriptors from the user-defined pattern ranking. These features allow to capture the importance of item interactions, and, as shown experimentally, lead to lower cumulative and individual regret than using low-level features. 
We have explored two multiplicative aggregation functions for mapping weights learned for complex features back to their component items, and find that the exponential multiplicative factor gives better results on most of the data sets we worked with.
Furthermore, we have proposed to use \cdflexics{} as the sampling component in \letsip{} to maximize the pattern diversity of each query and showed a straightforward combination of both improvements in \letsip{}. \cdflexics{} takes the XOR-constraint based data partitioning of \flexics{} and extends it with the diversity constraint \closedx{}, introduced in \cite{HienLALLOZ20}. The results show convincing improvements both in terms of learning more accurate pattern rankings and CPU-times. 
We have evaluated our proposal only on itemset data so far since the majority of existing work, including the method that we extended, is defined for this kind of data. But the importance of using complex dynamic features can be expected to be even higher when interactively mining complex, i.e. sequential, tree-, or graph-structured data. We will explore this direction in future work.

%% file: appendix.tex
\begin{appendices}	
\section{Detailed view of results for other datasets ($\newletsip$ vs. $\letsip$)}\label{secA1}

Figures \ref{fig:1}-\ref{fig:4a} plot a detailed view of comparison between \newletsipexp{} and \letsip{} (cumulative and non-cumulative regret) for different datasets w.r.t. $maximal$ quality, $k = 10$ and $\ell = 1$.

\begin{figure}[htbp]
	\centering
	\begin{tabular}{c}
		\subfloat[][Chess: cumulative regret. ]{\includegraphics[scale=0.2]{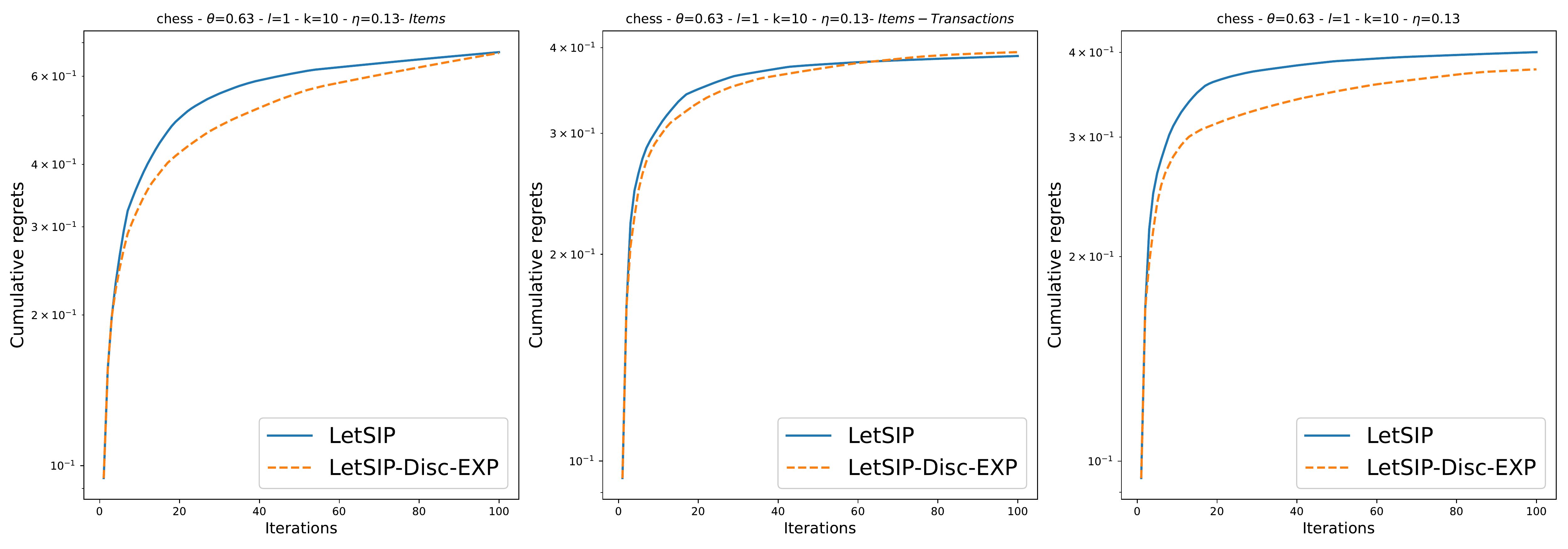}} \\
		\subfloat[][Chess: non cumulative regret. ]{\includegraphics[scale=0.2]{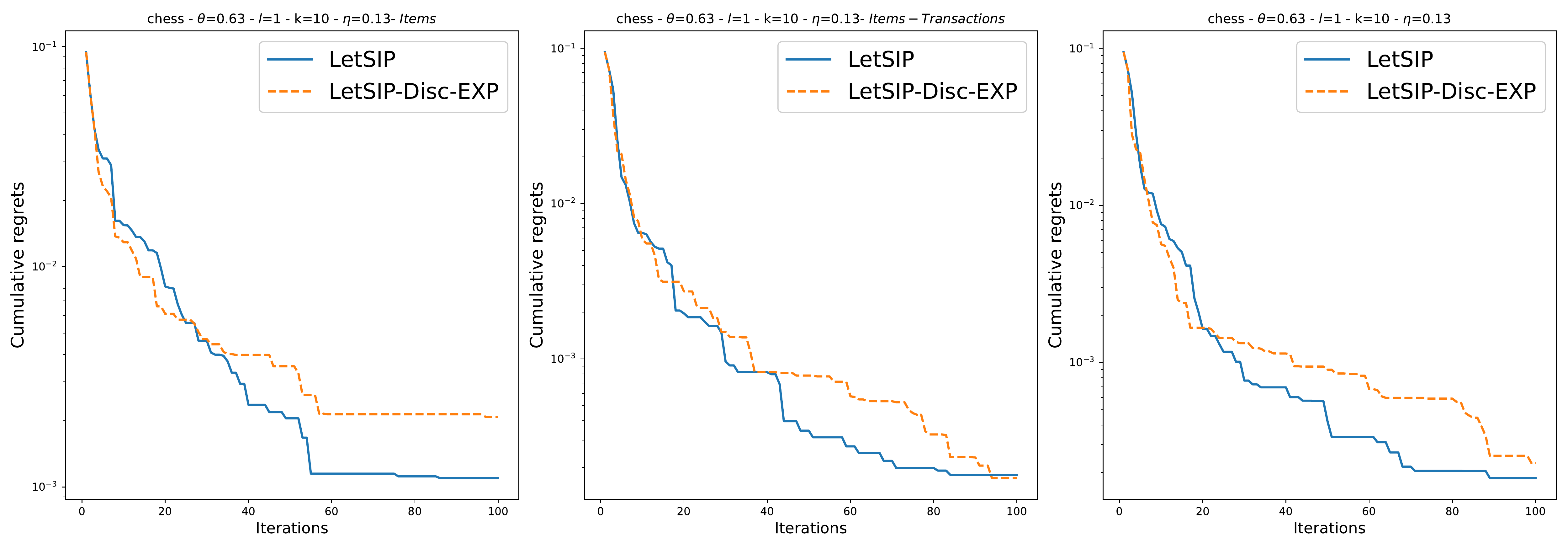}}
	\end{tabular}
	
	\begin{tabular}{c}
		\subfloat[][German-credit: cumulative regret. ]{\includegraphics[scale=0.2]{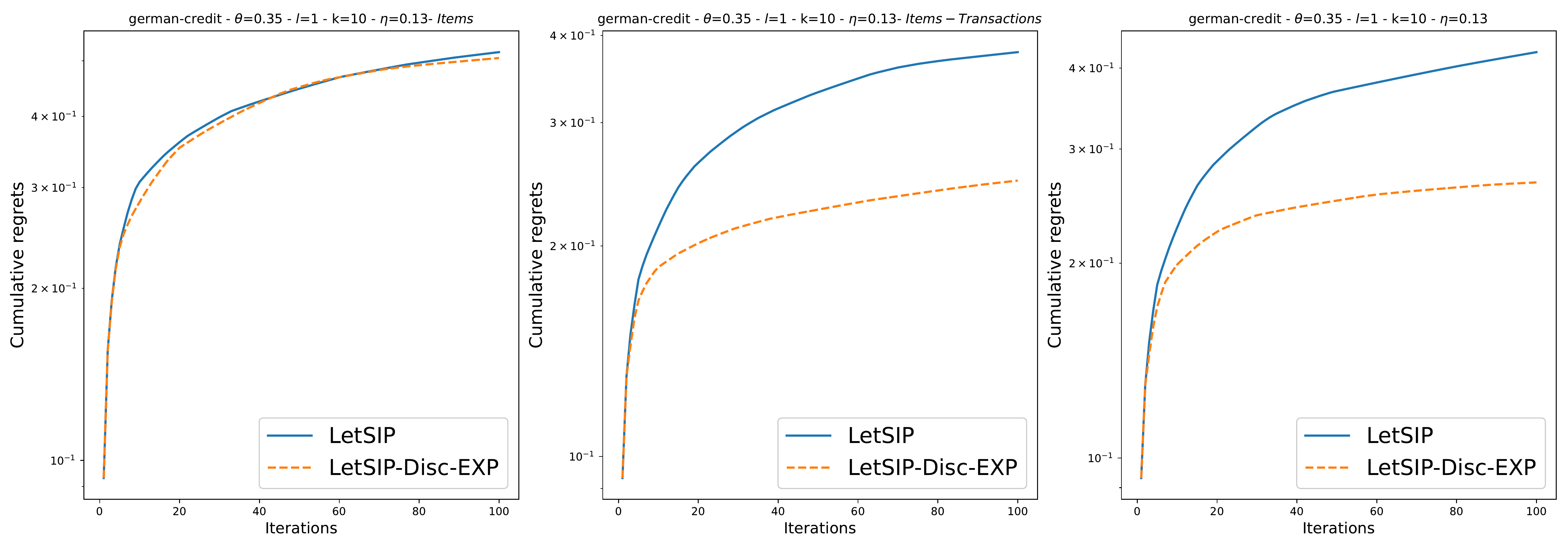}} \\
		\subfloat[][German-credit: non cumulative regret. ]{\includegraphics[scale=0.2]{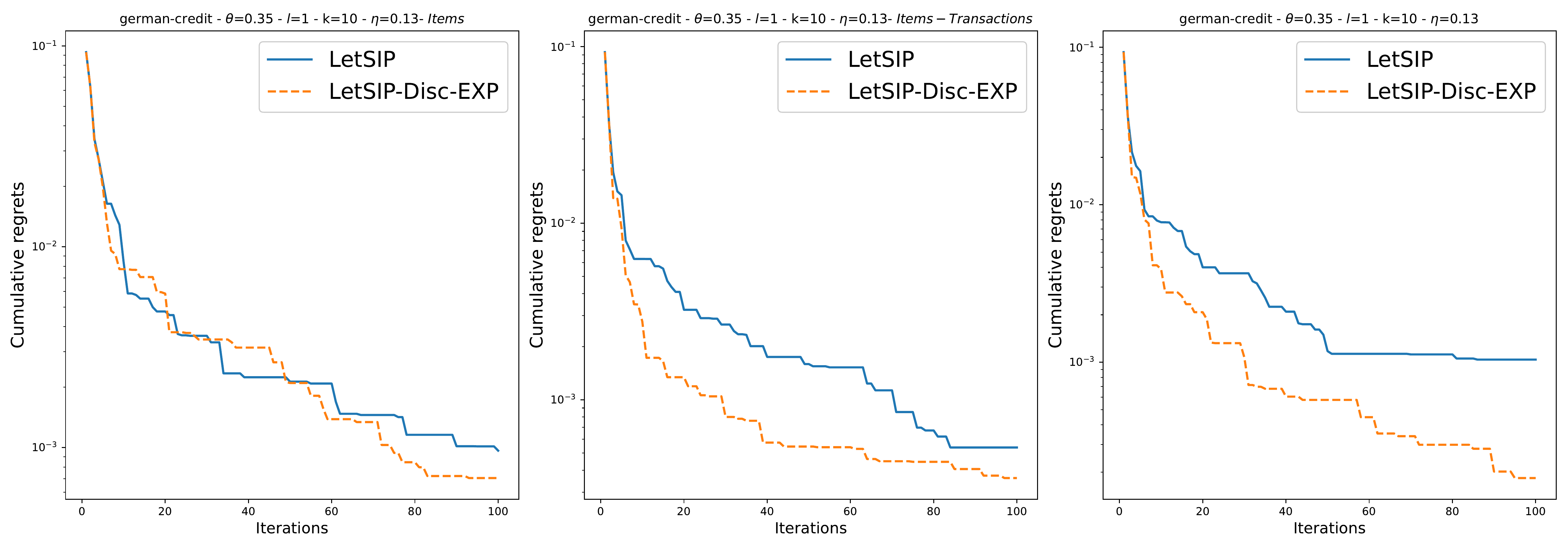}}	
	\end{tabular}
	\caption{A detailed view of comparison between \newletsipexp{} and \letsip{} (cumulative and non-cumulative regret) for different pattern features w.r.t. $maximal$ quality, $k = 10$ and $\ell = 1$.}
	\label{fig:1}
\end{figure}

\begin{figure}[t]
	\centering
	\begin{tabular}{c}
		\subfloat[][Heart-cleveland: cumulative regret. ]{\includegraphics[scale=0.2]{letsip_vs_letsip_disc_cumul_german-credit-0v35-eps-converted-to.pdf}} \\
		\subfloat[][Heart-cleveland: non cumulative regret. ]{\includegraphics[scale=0.2]{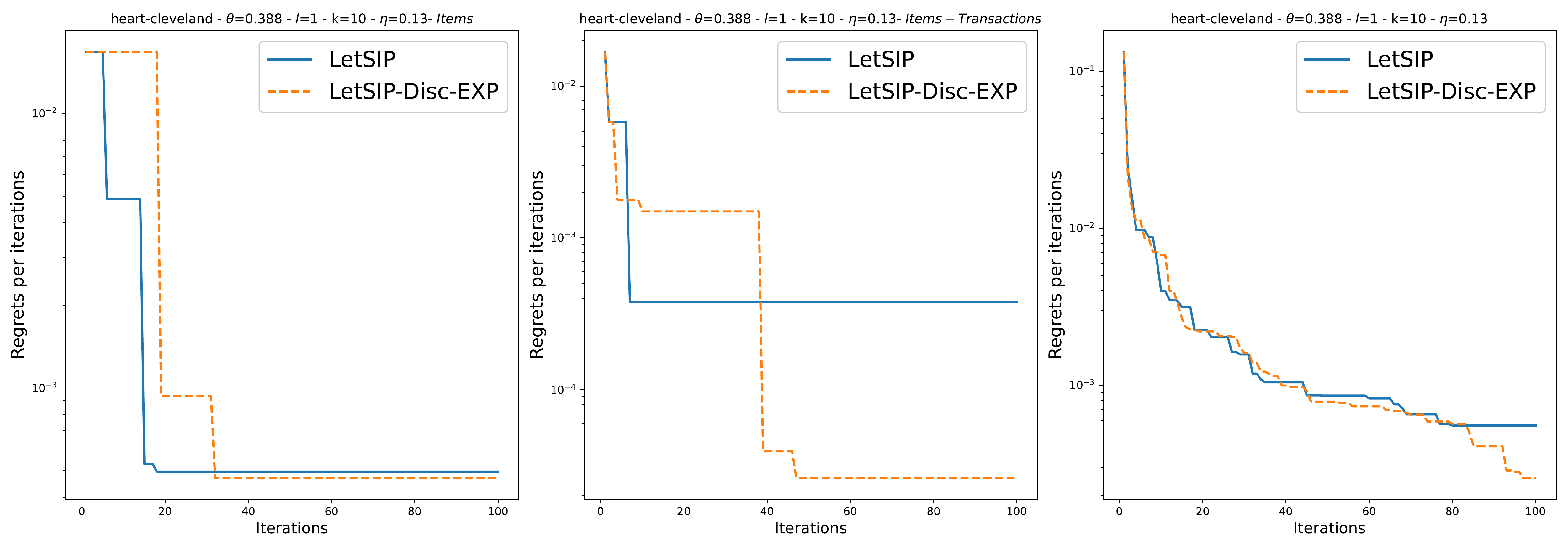}}	
	\end{tabular}
	
	\begin{tabular}{c}
		\subfloat[][Hepatitis: cumulative regret. ]{\includegraphics[scale=0.2]{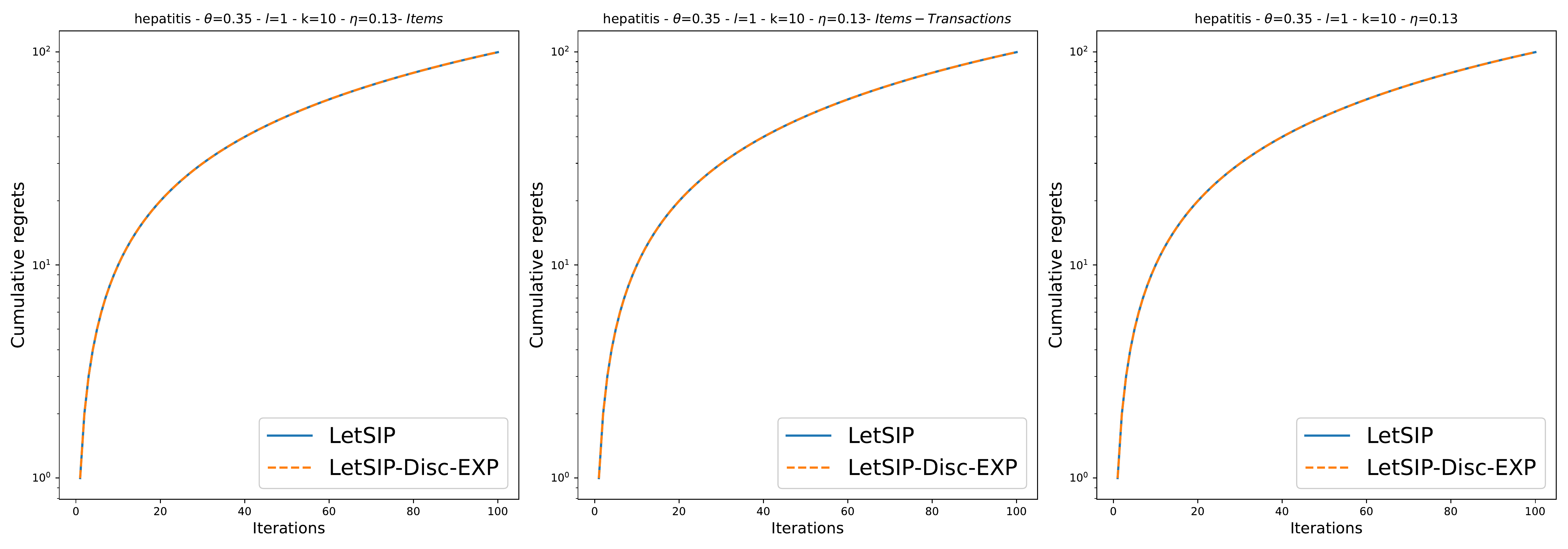}} \\
		\subfloat[][Hepatitis: non cumulative regret. ]{\includegraphics[scale=0.2]{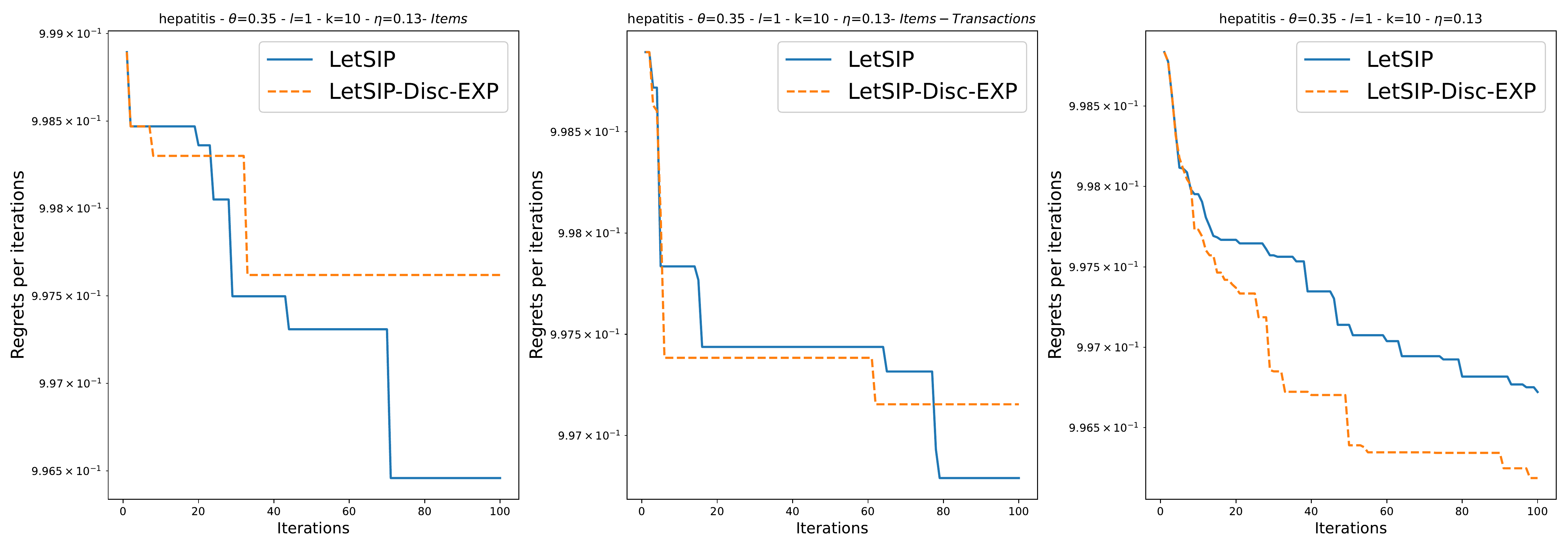}}	
	\end{tabular}
	\caption{A detailed view of comparison between \newletsipexp{} and \letsip{} (cumulative and non-cumulative regret) for different pattern features w.r.t. $maximal$ quality, $k = 10$ and $\ell = 1$.}
	\label{fig:2}
\end{figure}

\begin{figure}[t]
	\centering
	\begin{tabular}{c}
		\subfloat[][Kr-vs-kp: cumulative regret. ]{\includegraphics[scale=0.2]{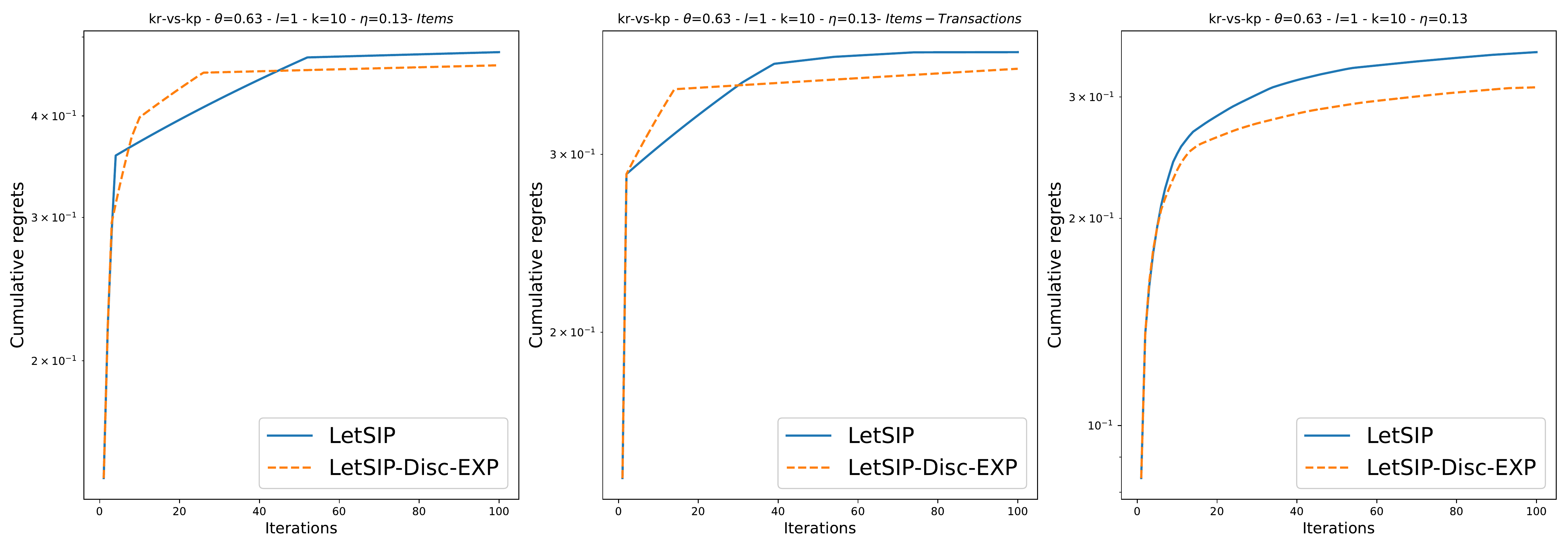}} \\
		\subfloat[][Kr-vs-kp: non cumulative regret. ]{\includegraphics[scale=0.2]{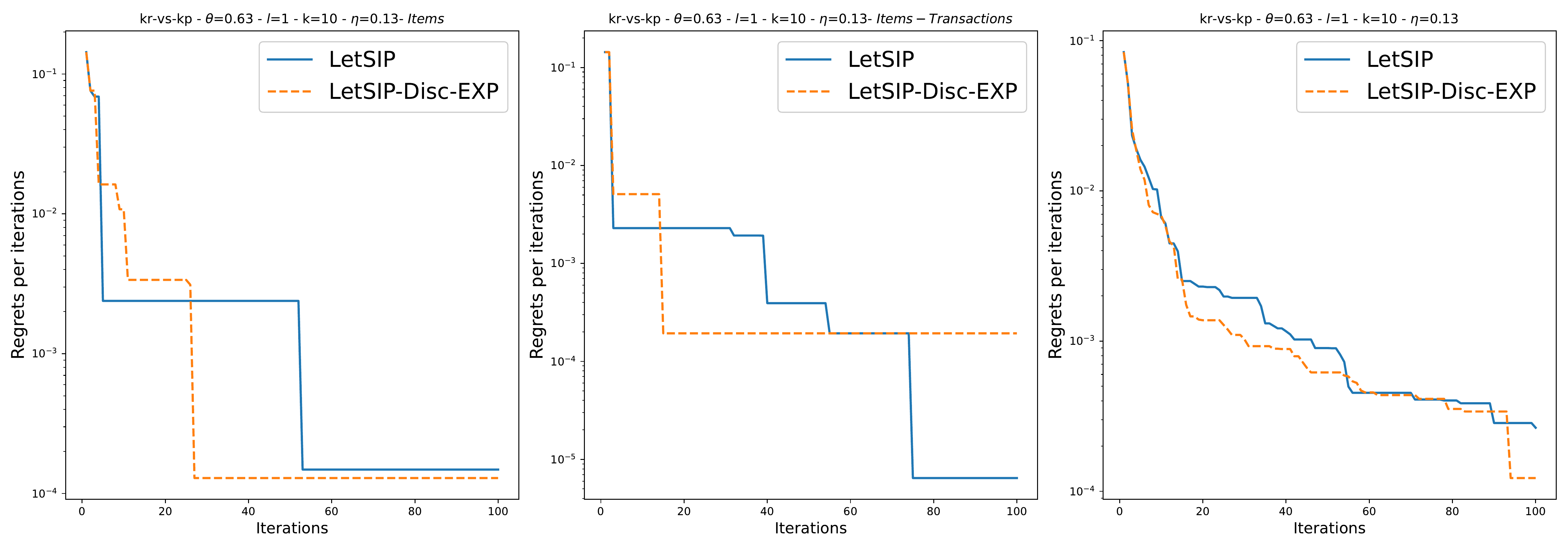}}	
	\end{tabular}
	
	\begin{tabular}{c}
		\subfloat[][Mushroom: cumulative regret. ]{\includegraphics[scale=0.2]{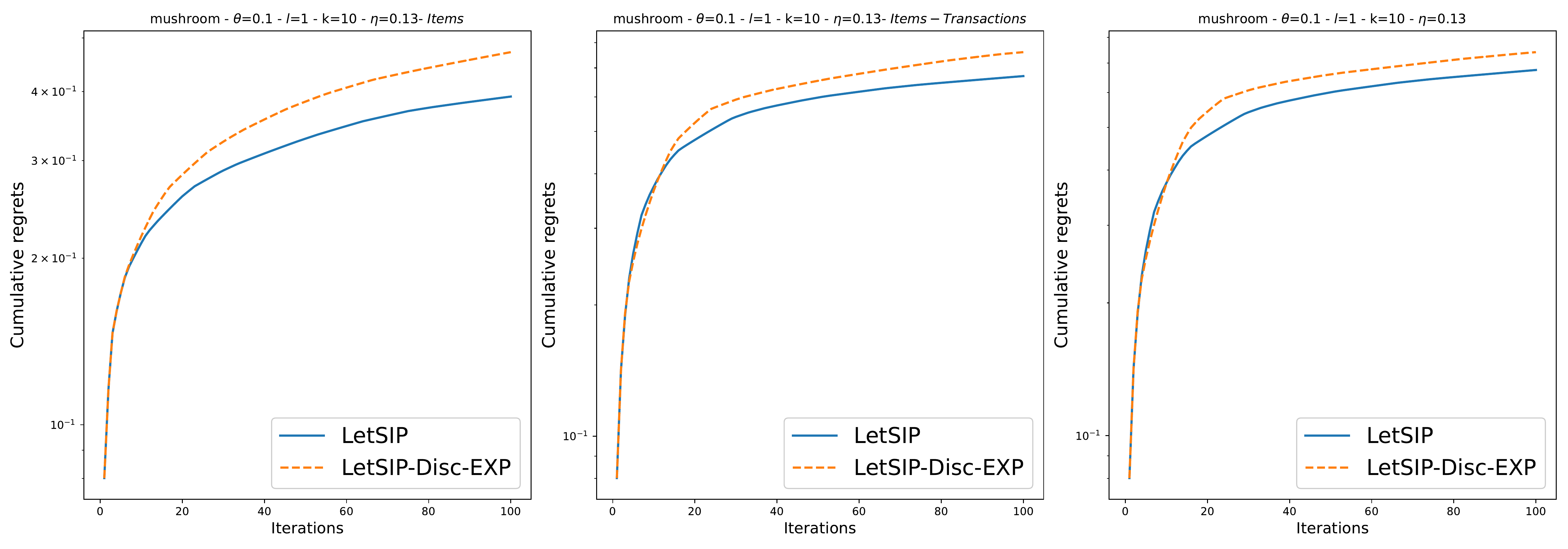}} \\
		\subfloat[][Mushroom: non cumulative regret. ]{\includegraphics[scale=0.2]{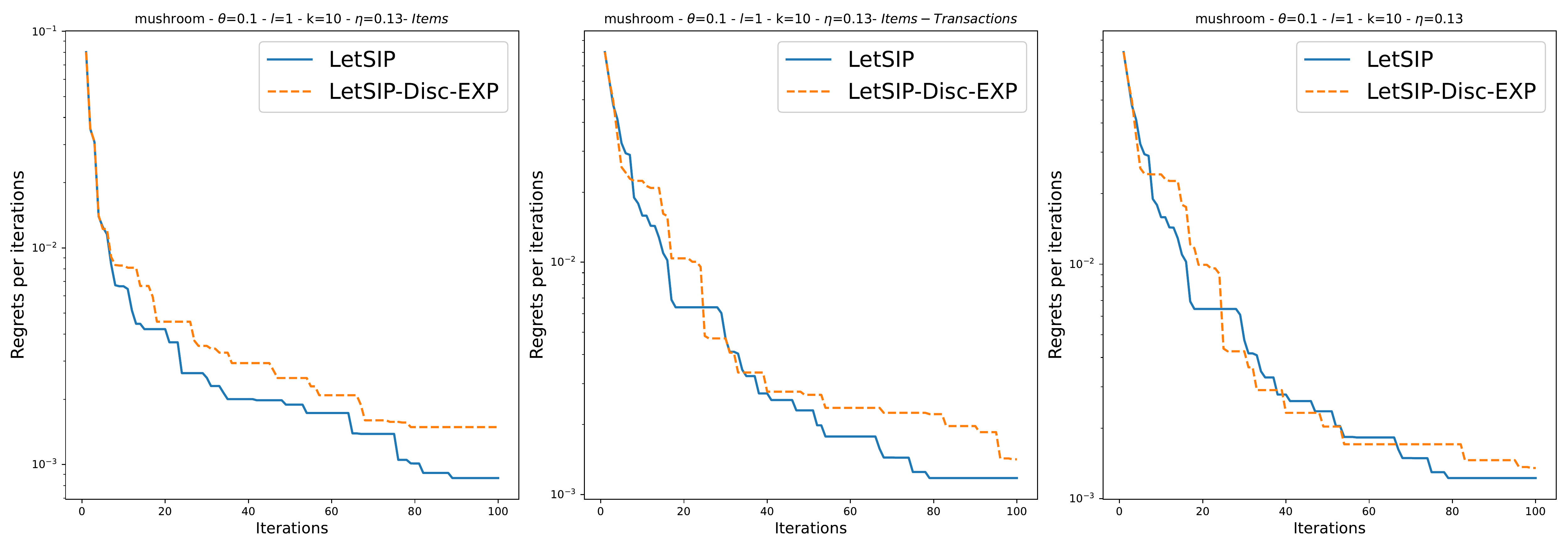}}	
	\end{tabular}
	\caption{A detailed view of comparison between \newletsipexp{} and \letsip{} (cumulative and non-cumulative regret) for different pattern features w.r.t. $maximal$ quality, $k = 10$ and $\ell = 1$.}
	\label{fig:3}
\end{figure}

\begin{figure}[t]
	\centering
	\begin{tabular}{c}
		\subfloat[][Soybean: cumulative regret. ]{\includegraphics[scale=0.2]{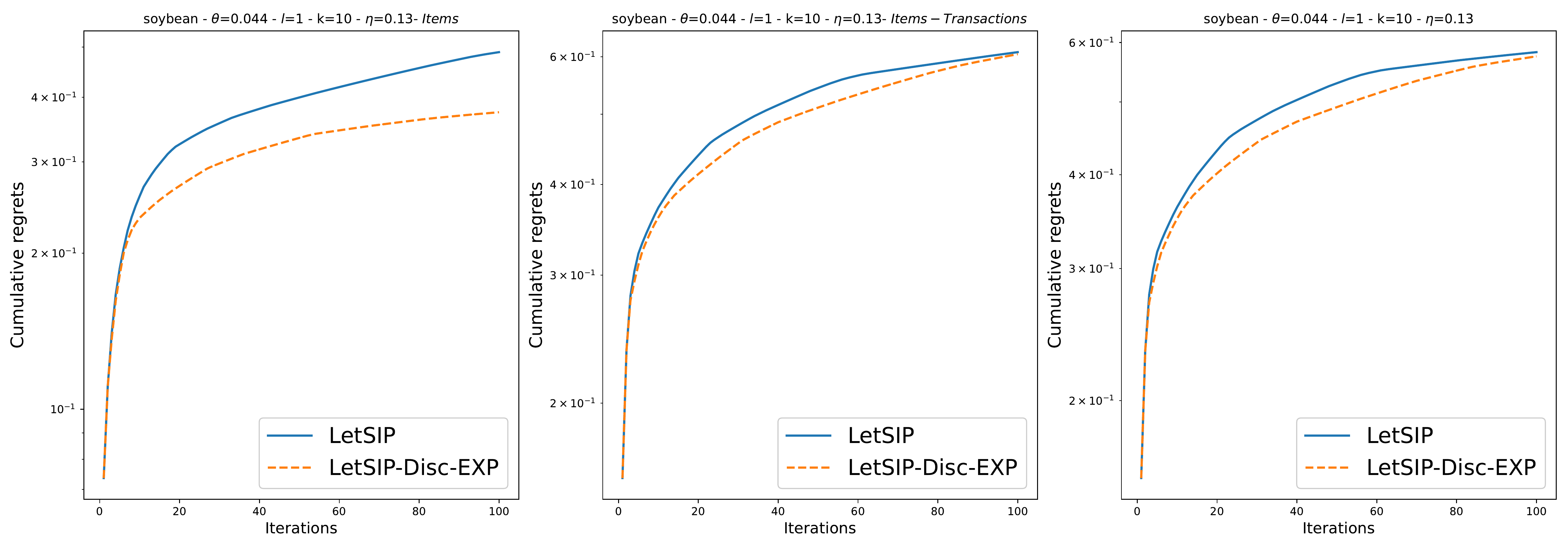}} \\
		\subfloat[][Soybean: non cumulative regret. ]{\includegraphics[scale=0.2]{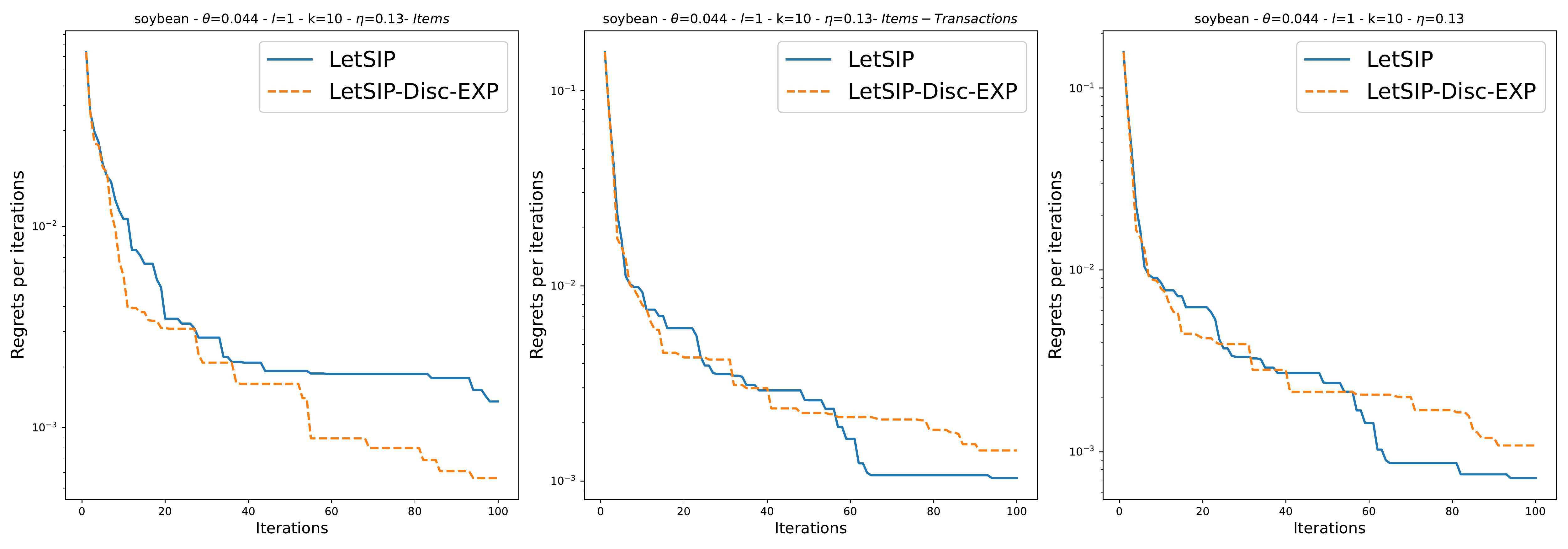}}	
	\end{tabular}
	
	\begin{tabular}{c}
		\subfloat[][Vote: cumulative regret. ]{\includegraphics[scale=0.2]{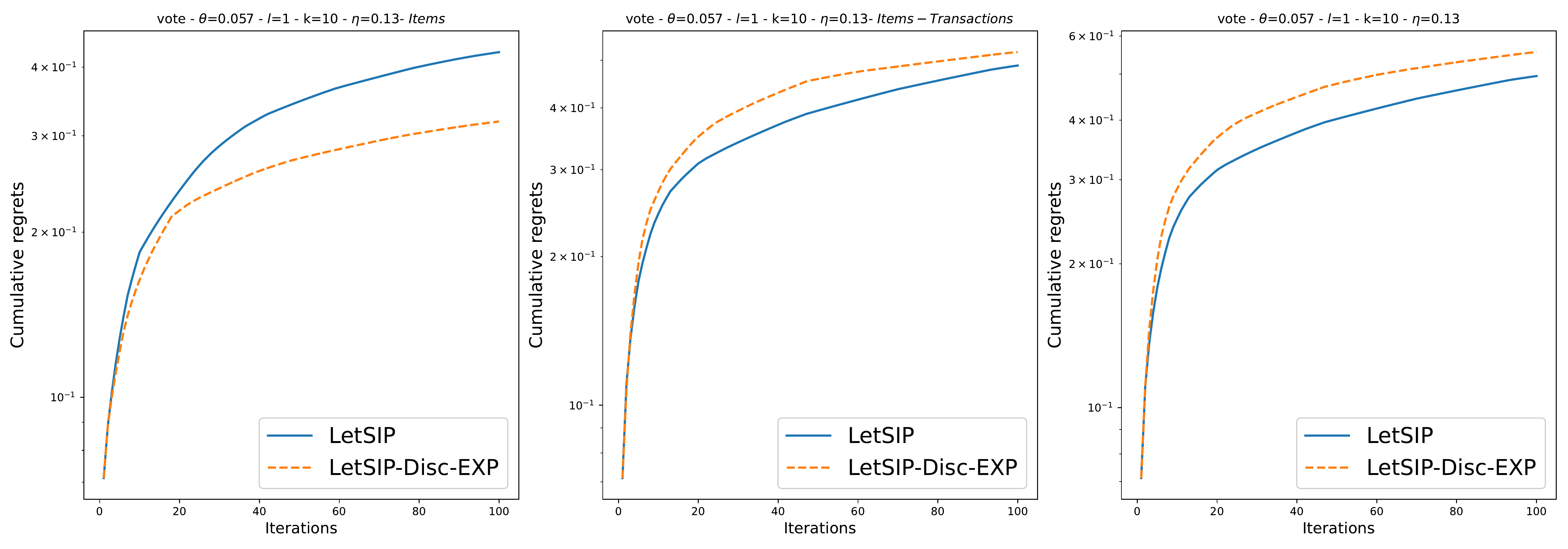}} \\
		\subfloat[][Vote: non cumulative regret. ]{\includegraphics[scale=0.2]{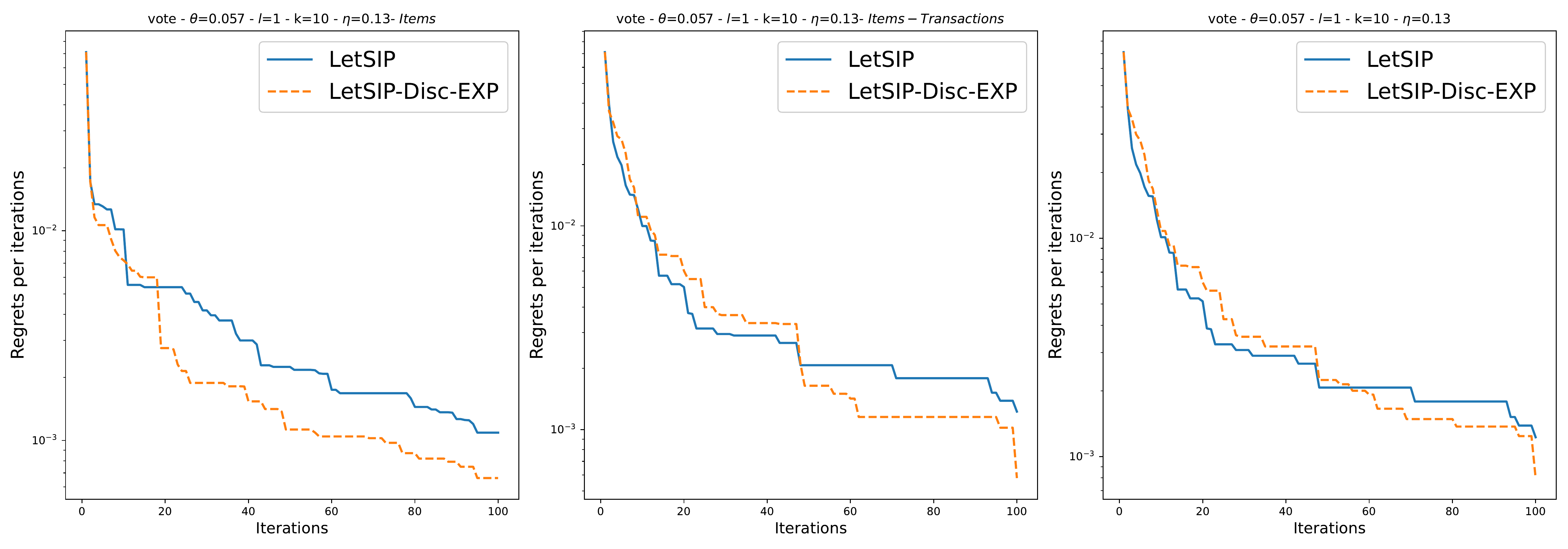}}	
	\end{tabular}
	\caption{A detailed view of comparison between \newletsipexp{} and \letsip{} (cumulative and non-cumulative regret) for different pattern features w.r.t. $maximal$ quality, $k = 10$ and $\ell = 1$.}
	\label{fig:4}
\end{figure}

\begin{figure}[t]
	\centering
	\begin{tabular}{c}
		\subfloat[][Zoo-1: cumulative regret. ]{\includegraphics[scale=0.2]{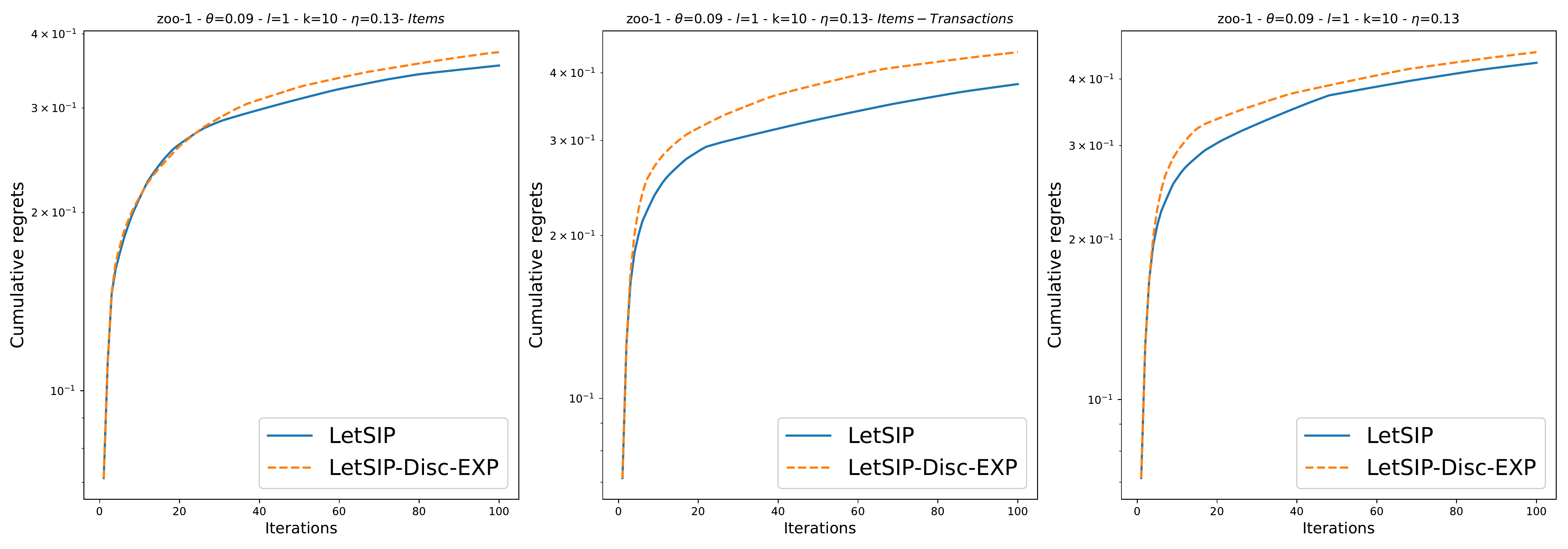}} \\
		\subfloat[][Zoo-1: non cumulative regret. ]{\includegraphics[scale=0.2]{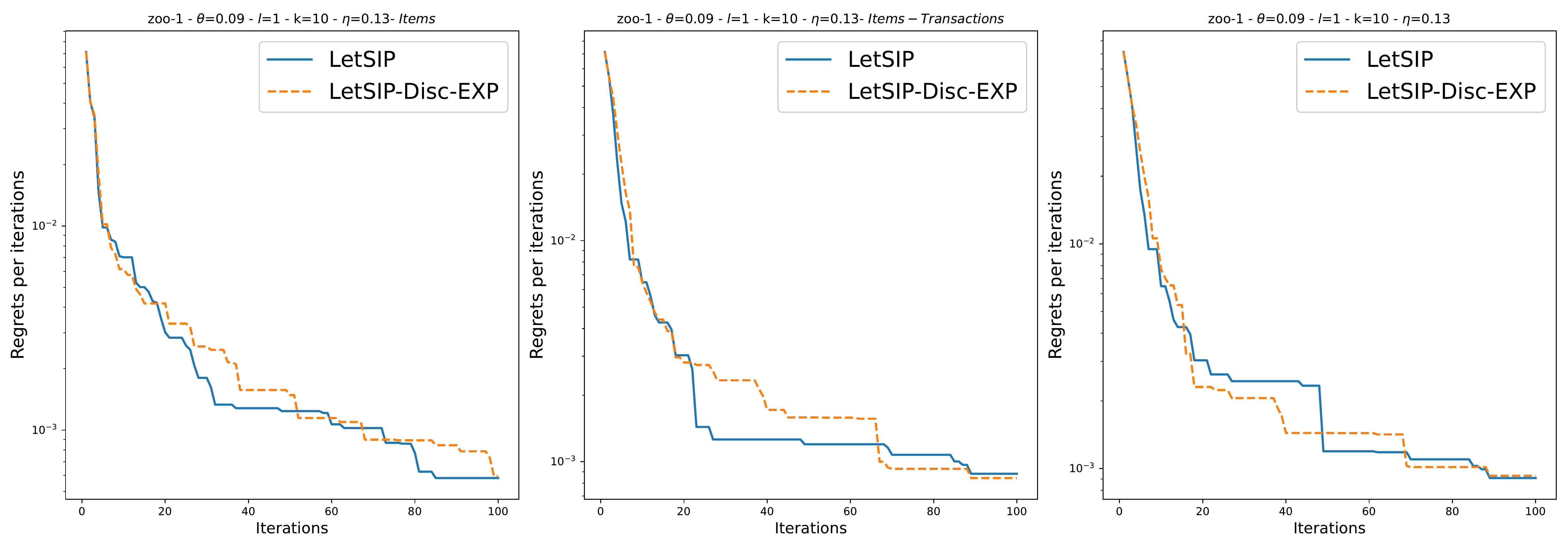}}	
	\end{tabular}
	\caption{A detailed view of comparison between \newletsipexp{} and \letsip{} (cumulative and non-cumulative regret) for different pattern features w.r.t. $maximal$ quality, $k = 10$ and $\ell = 1$.}
	\label{fig:4a}
\end{figure}


\section{Detailed view of results for other datasets ($\letsipcdf$ vs. $\letsip$)}\label{secB1}

Figures \ref{fig:5}-\ref{fig:8} plot a detailed view of comparison between \letsipcdf{} and \letsip{} (cumulative and non-cumulative regret) for different datasets w.r.t. $maximal$ quality, $k = 10$ and $\ell = 1$.

\begin{figure}[htbp]
	\centering
	\begin{tabular}{c}
		\subfloat[][Chess: cumulative regret. ]{\includegraphics[scale=0.2]{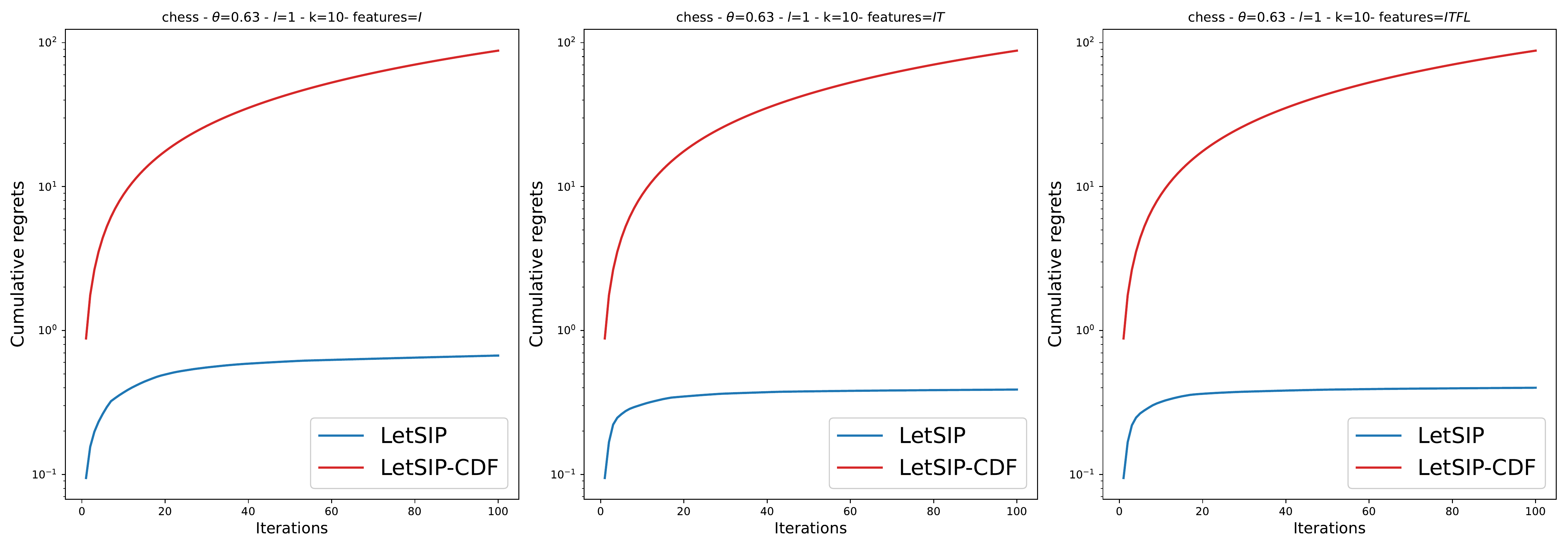}} \\
		\subfloat[][Chess: non cumulative regret. ]{\includegraphics[scale=0.2]{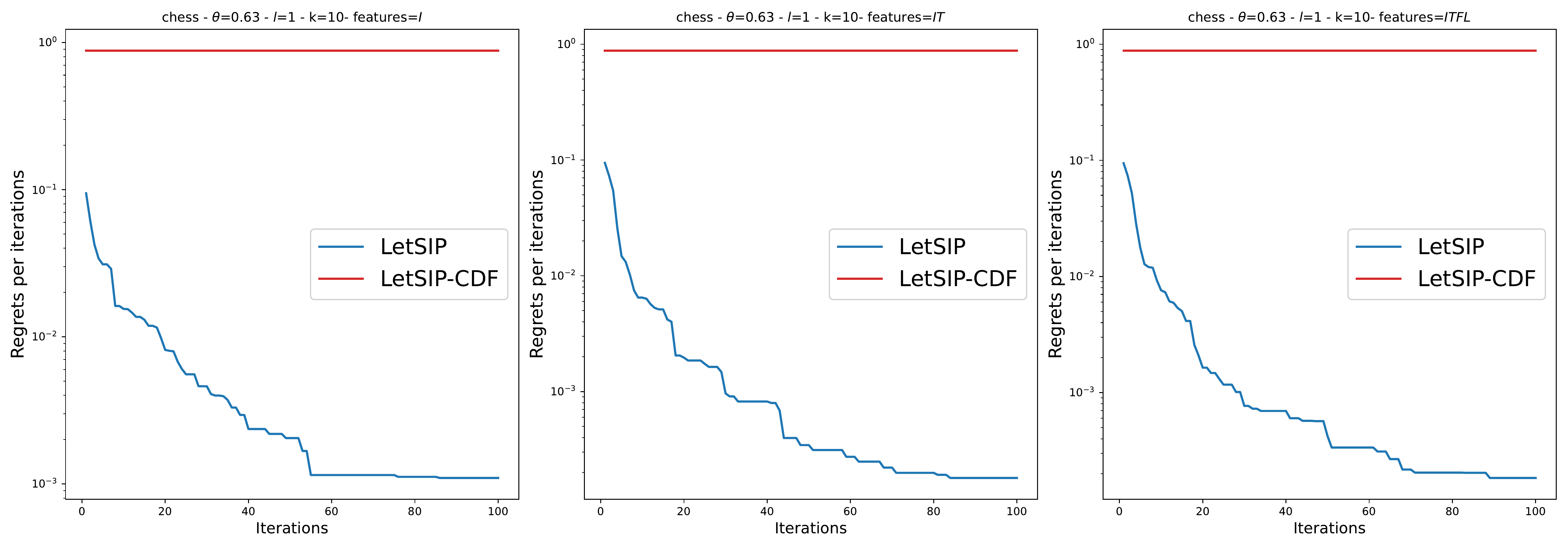}}
	\end{tabular}
	
	\begin{tabular}{c}
		\subfloat[][German-credit: cumulative regret. ]{\includegraphics[scale=0.2]{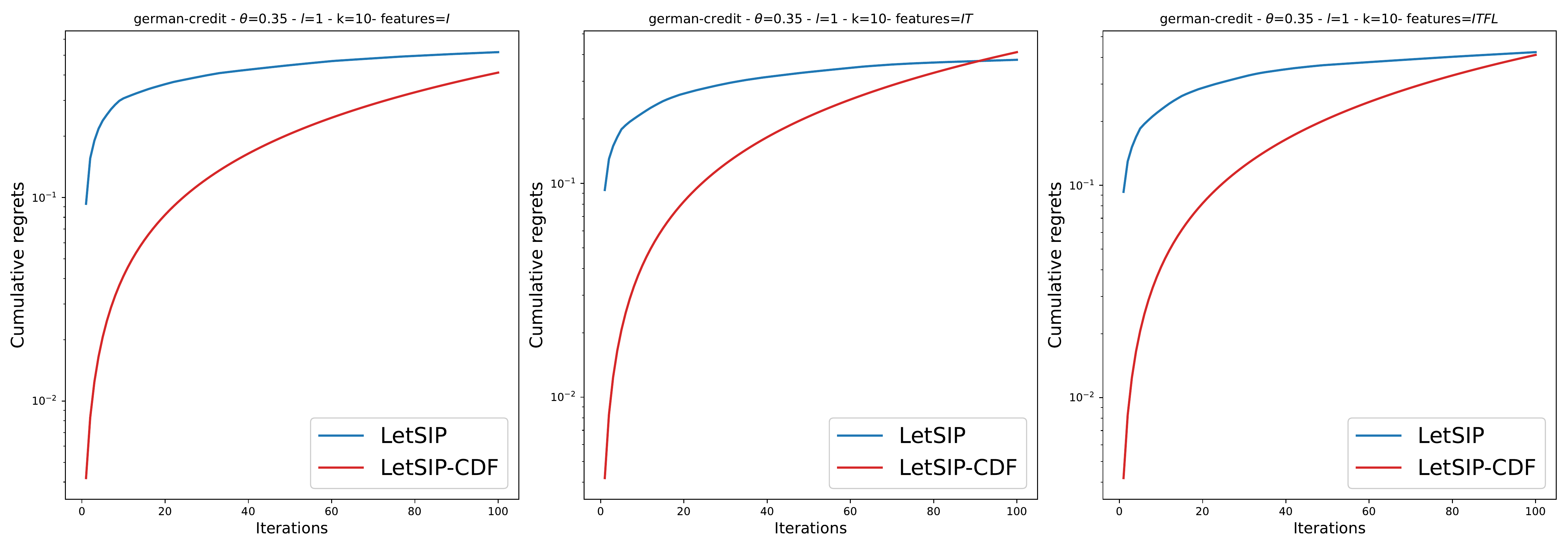}} \\
		\subfloat[][German-credit: non cumulative regret. ]{\includegraphics[scale=0.2]{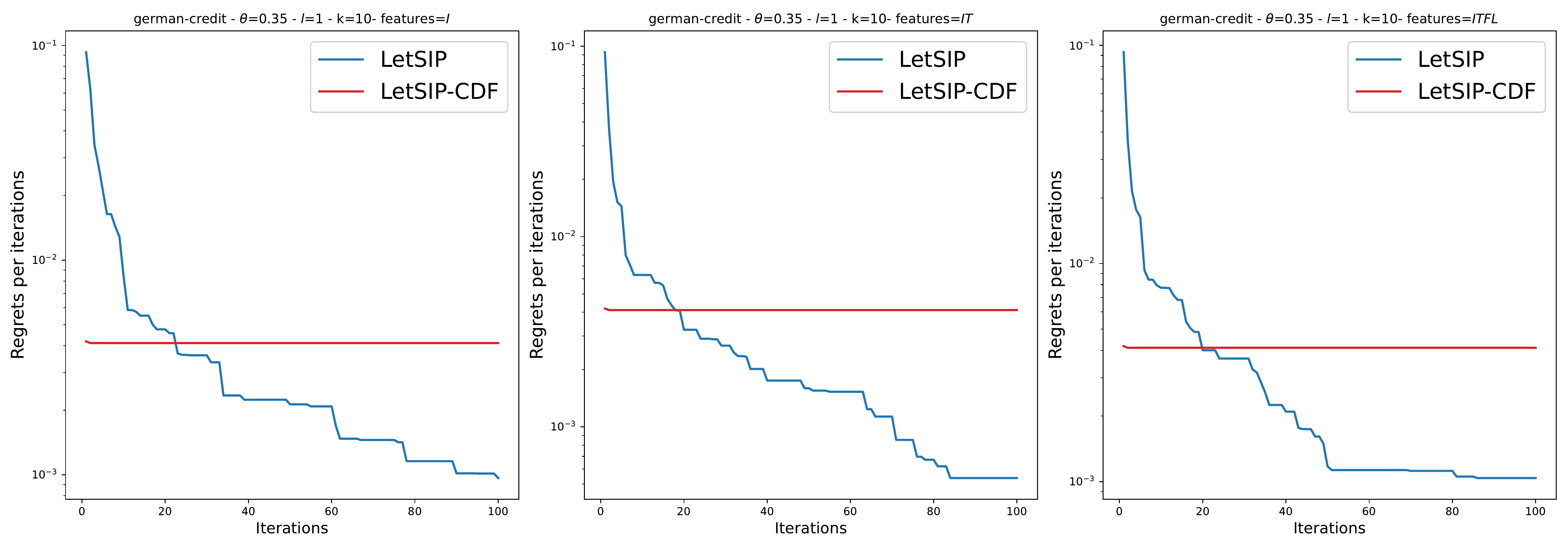}}	
	\end{tabular}
	\caption{A detailed view of comparison between \letsipcdf{} and \letsip{} (cumulative and non-cumulative regret) for different pattern features w.r.t. $maximal$ quality, $k = 10$ and $\ell = 1$.}
	\label{fig:5}
\end{figure}

\begin{figure}[t]
	\centering
	\begin{tabular}{c}
		\subfloat[][Heart-cleveland: cumulative regret. ]{\includegraphics[scale=0.2]{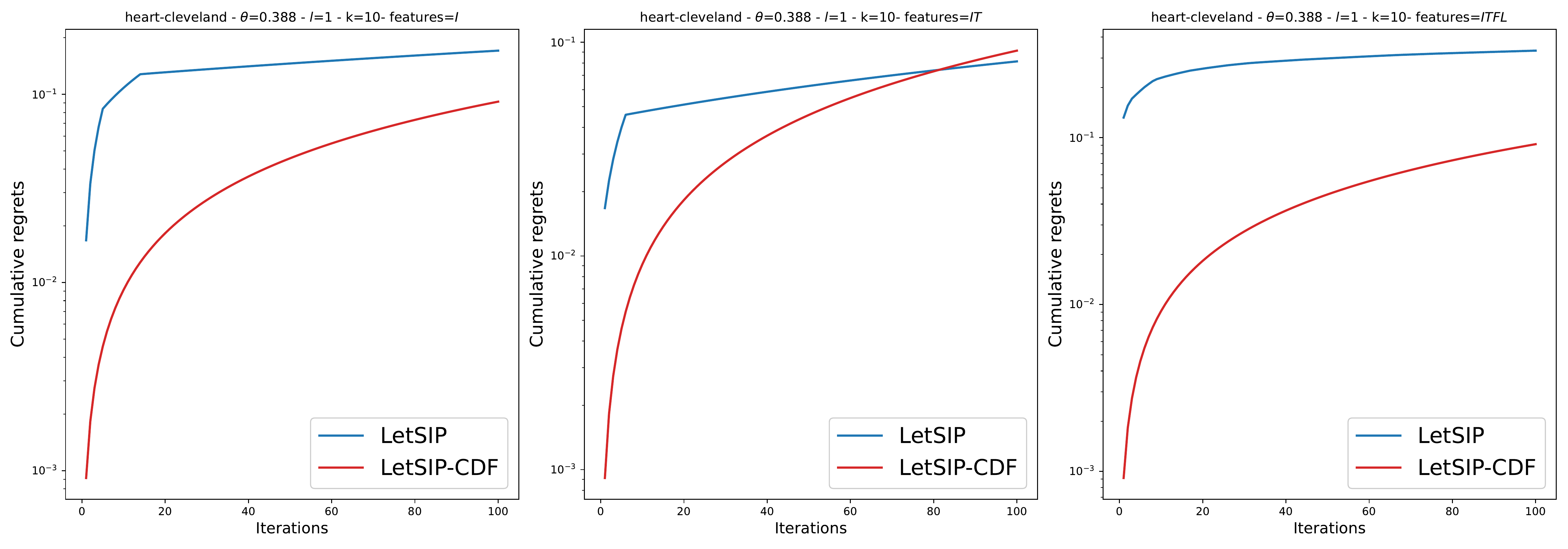}} \\
		\subfloat[][Heart-cleveland: non cumulative regret. ]{\includegraphics[scale=0.2]{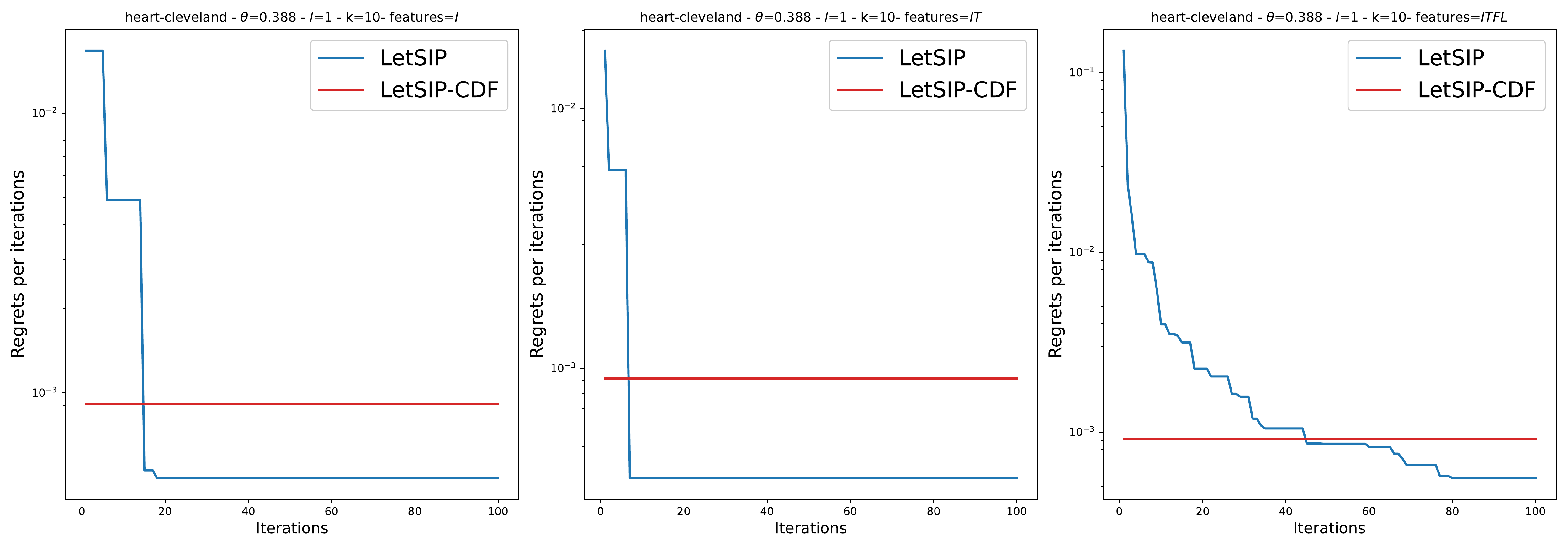}}	
	\end{tabular}
	
	\begin{tabular}{c}
		\subfloat[][Hepatitis: cumulative regret. ]{\includegraphics[scale=0.2]{letsip_vs_letsip_cdf_cumul_hepatitis-0v35-eps-converted-to.pdf}} \\
		\subfloat[][Hepatitis: non cumulative regret. ]{\includegraphics[scale=0.2]{letsip_vs_letsip_cdf_no_cumul_hepatitis-0v35-eps-converted-to.pdf}}	
	\end{tabular}
	\caption{A detailed view of of comparison between \letsipcdf{} and \letsip{} (cumulative and non-cumulative regret) for different pattern features w.r.t. $maximal$ quality, $k = 10$ and $\ell = 1$.}
	\label{fig:6}
\end{figure}

\begin{figure}[t]
	\centering
	\begin{tabular}{c}
		\subfloat[][Kr-vs-kp: cumulative regret. ]{\includegraphics[scale=0.2]{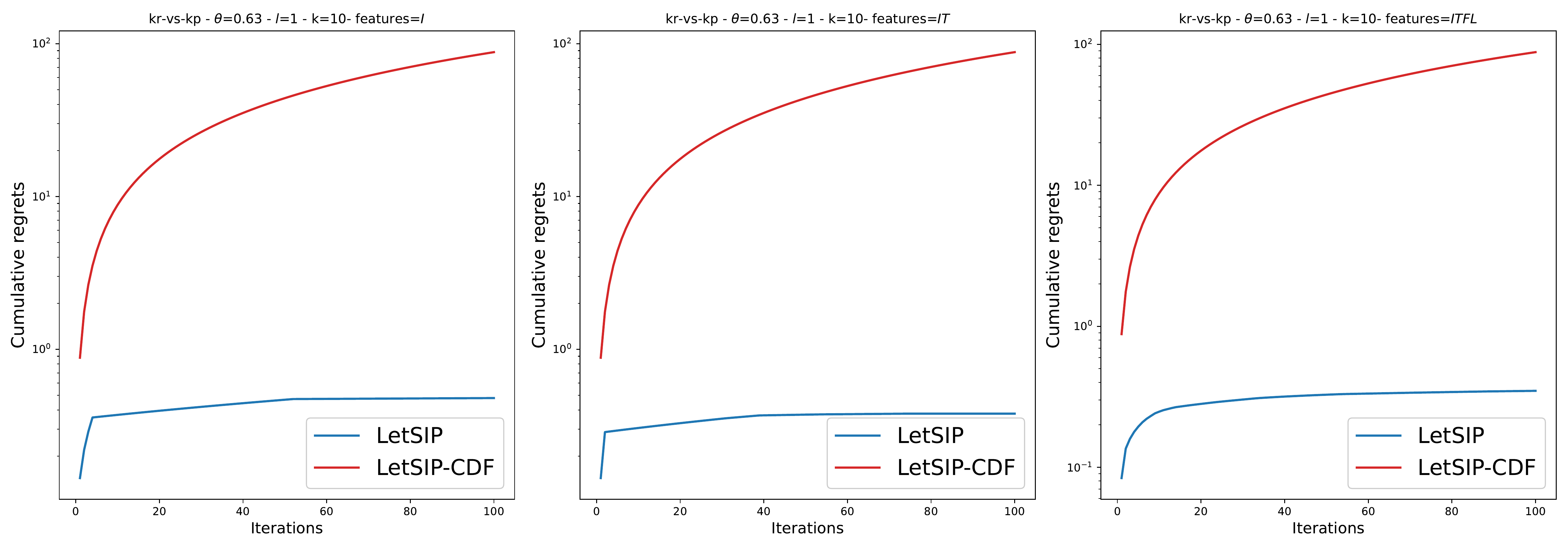}} \\
		\subfloat[][Kr-vs-kp: non cumulative regret. ]{\includegraphics[scale=0.2]{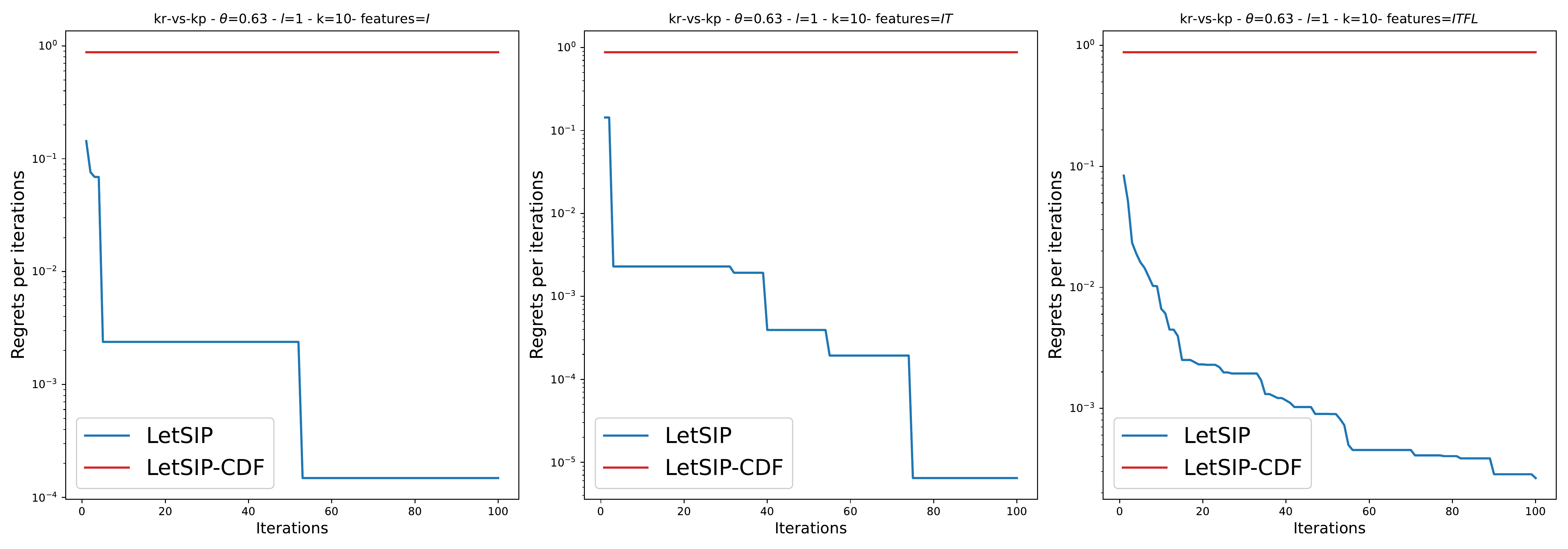}}	
	\end{tabular}
	
	\begin{tabular}{c}
		\subfloat[][Mushroom: cumulative regret. ]{\includegraphics[scale=0.2]{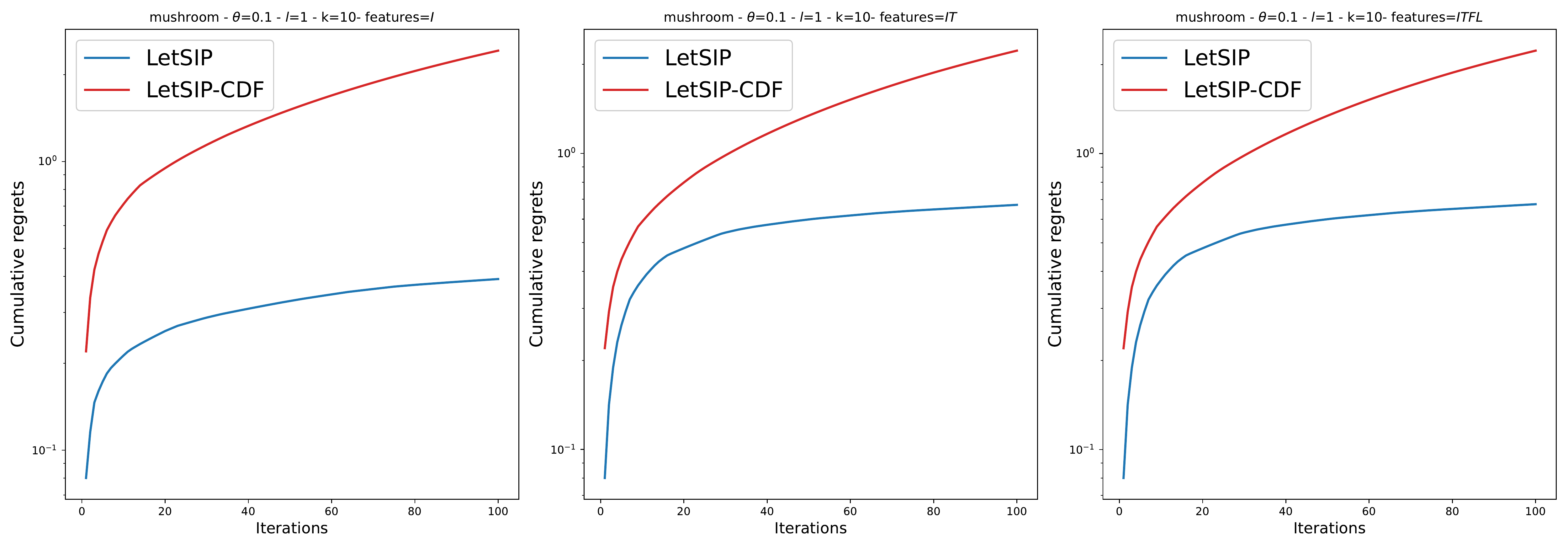}} \\
		\subfloat[][Mushroom: non cumulative regret. ]{\includegraphics[scale=0.2]{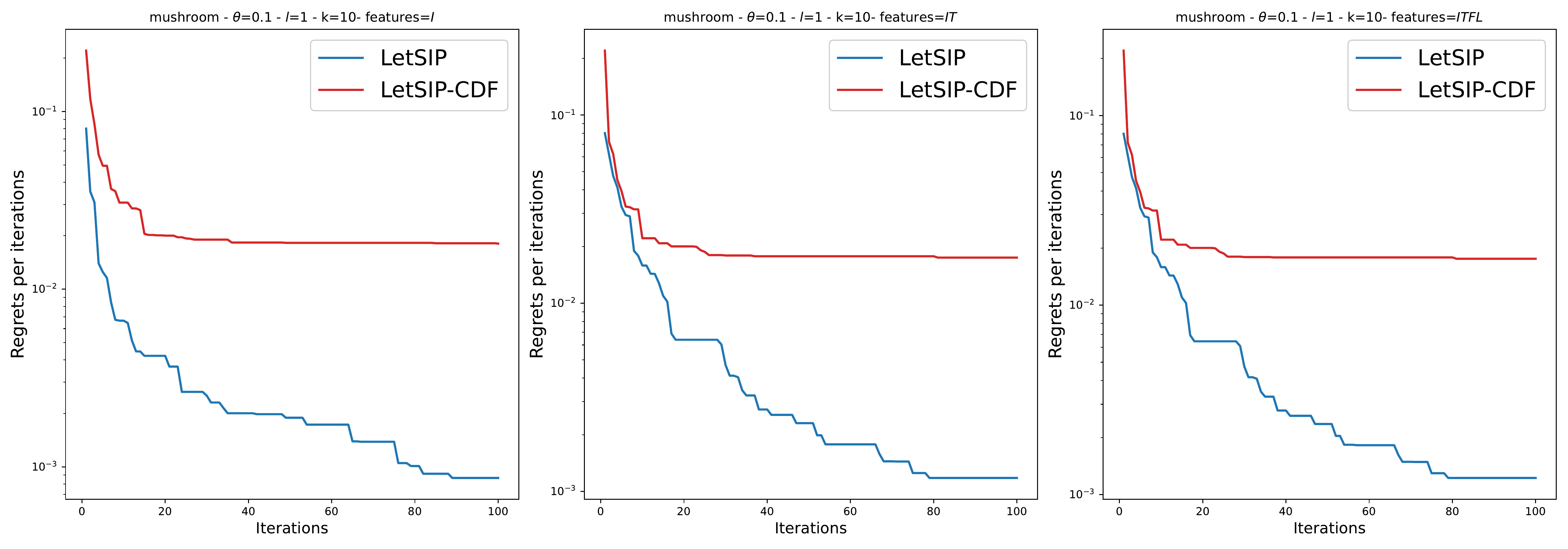}}	
	\end{tabular}
	\caption{A detailed view of of comparison between \letsipcdf{} and \letsip{} (cumulative and non-cumulative regret) for different pattern features w.r.t. $maximal$ quality, $k = 10$ and $\ell = 1$.}
	\label{fig:7}
\end{figure}

\begin{figure}[t]
	\centering
	\begin{tabular}{c}
		\subfloat[][Soybean: cumulative regret. ]{\includegraphics[scale=0.2]{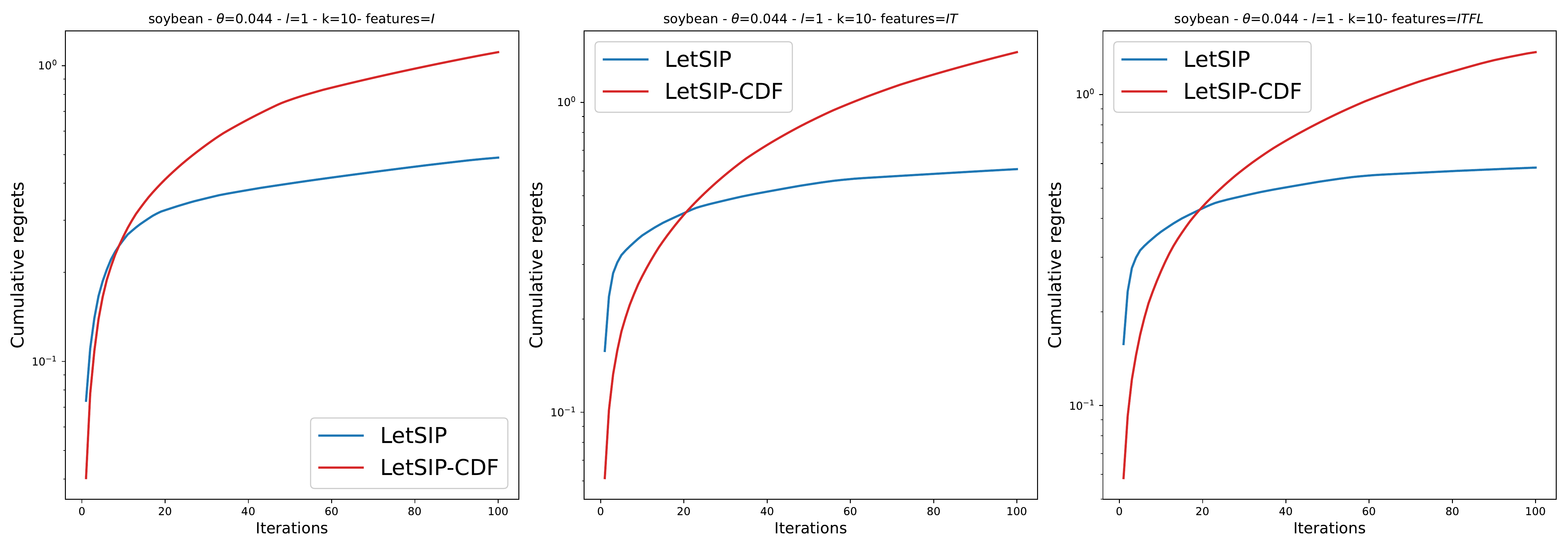}} \\
		\subfloat[][Soybean: non cumulative regret. ]{\includegraphics[scale=0.2]{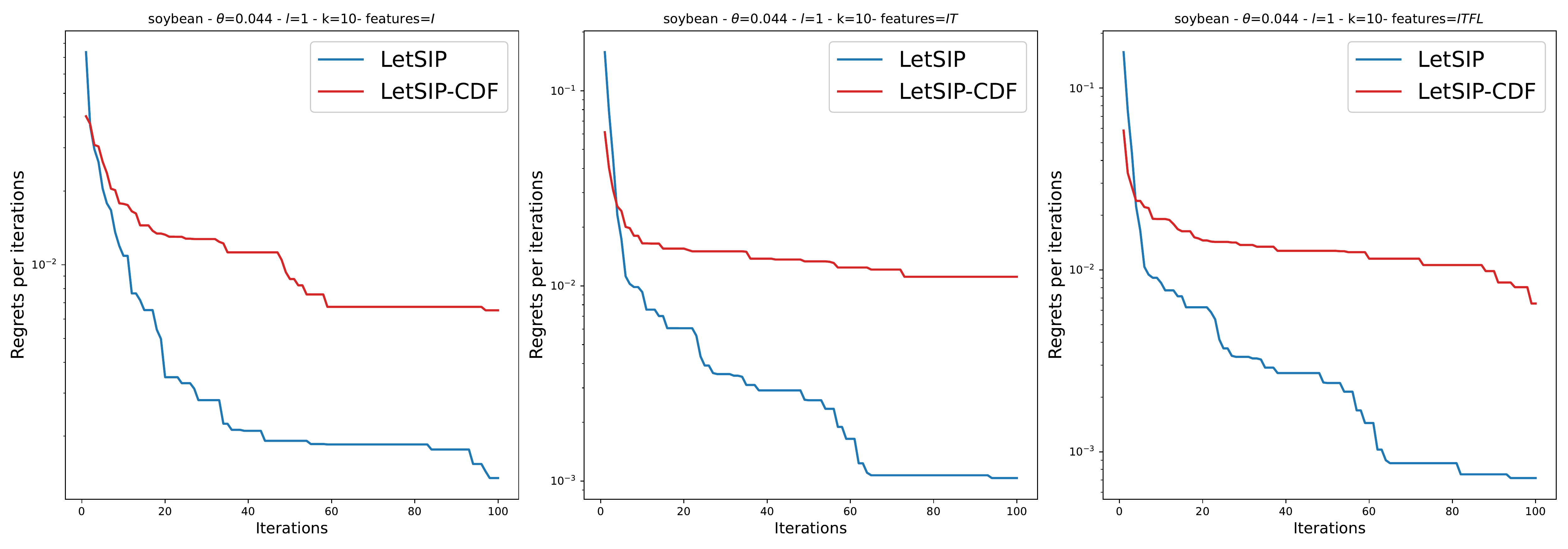}}	
	\end{tabular}
	
	\begin{tabular}{c}
		\subfloat[][Vote: cumulative regret. ]{\includegraphics[scale=0.2]{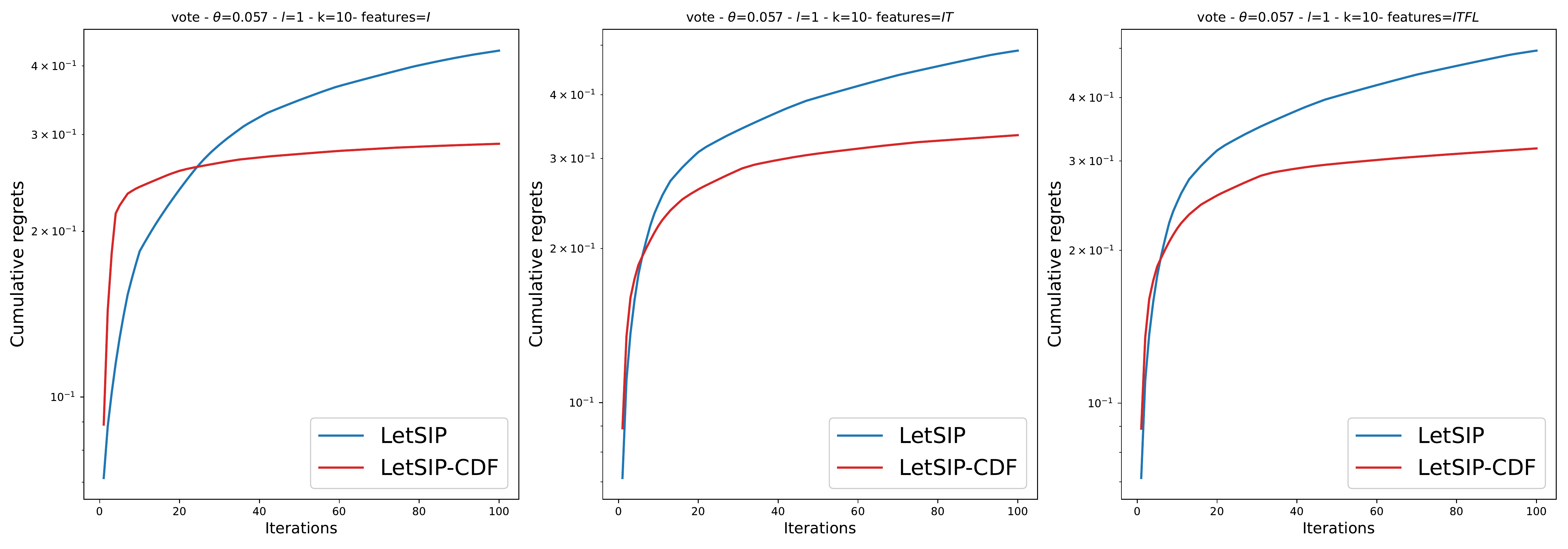}} \\
		\subfloat[][Vote: non cumulative regret. ]{\includegraphics[scale=0.2]{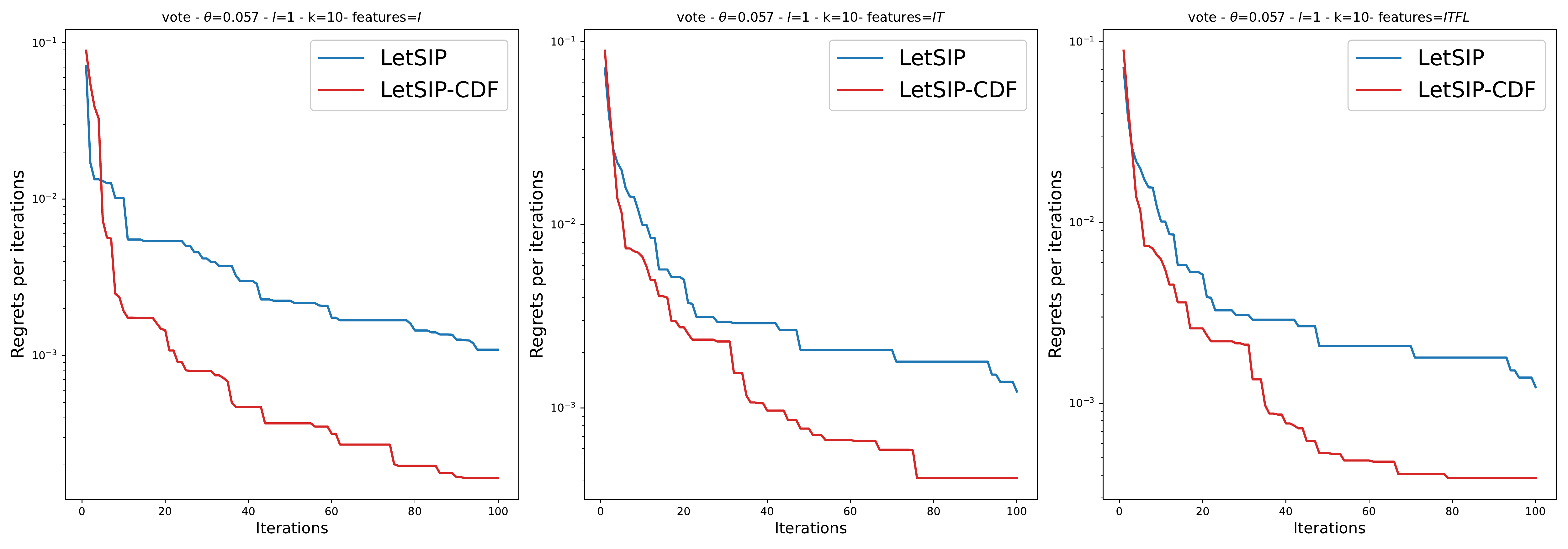}}	
	\end{tabular}
	\caption{A detailed view of of comparison between \letsipcdf{} and \letsip{} (cumulative and non-cumulative regret) for different pattern features w.r.t. $maximal$ quality, $k = 10$ and $\ell = 1$.}
	\label{fig:8}
\end{figure}

\begin{figure}[t]
	\centering
	\begin{tabular}{c}
		\subfloat[][Anneal: cumulative regret. ]{\includegraphics[scale=0.2]{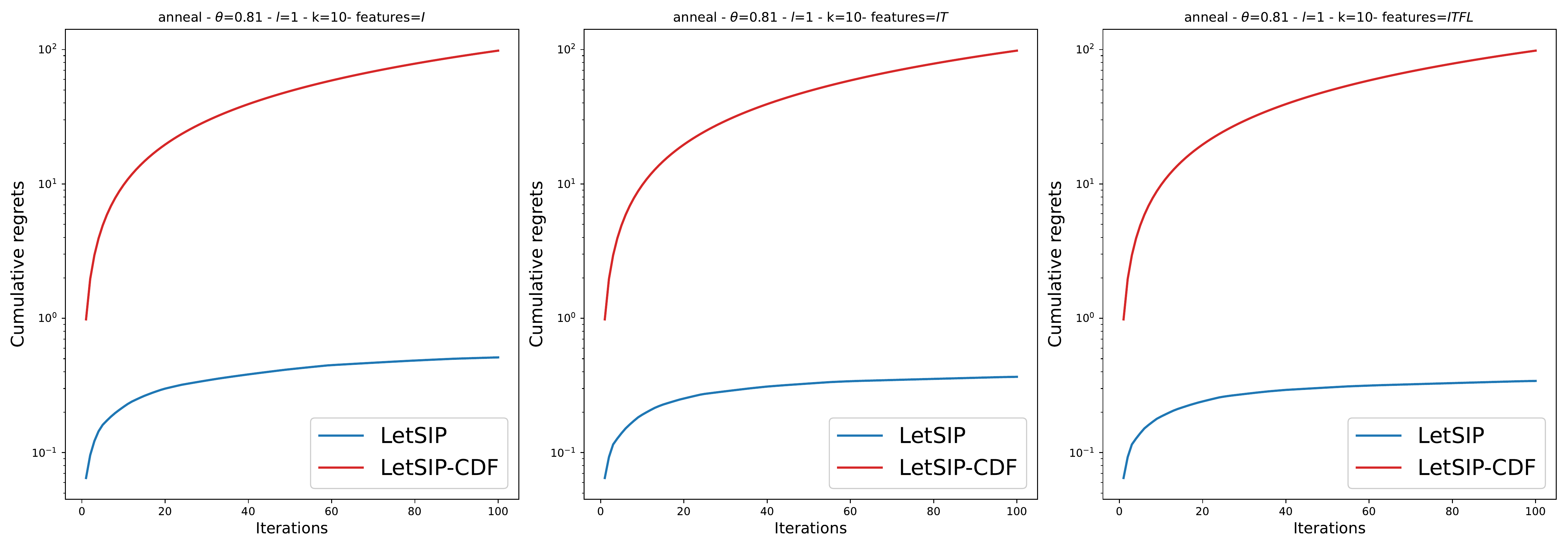}} \\
		\subfloat[][Anneal: non cumulative regret. ]{\includegraphics[scale=0.2]{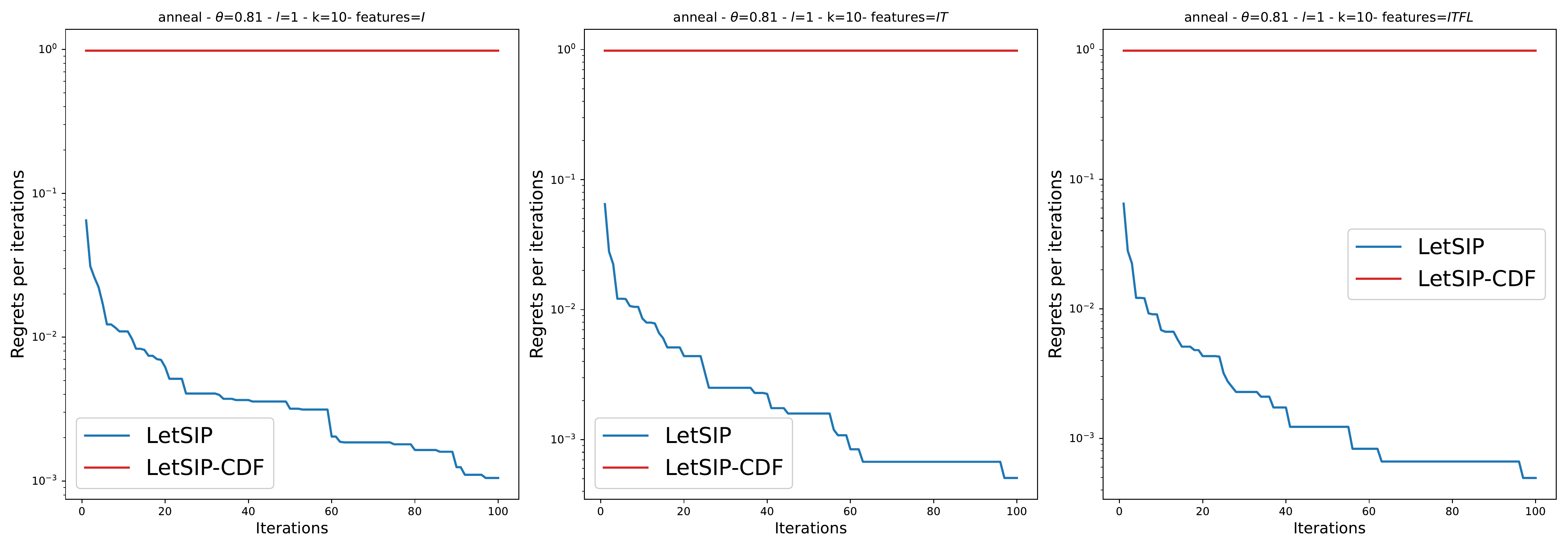}}	
	\end{tabular}
	\begin{tabular}{c}
		\subfloat[][Zoo-1: cumulative regret. ]{\includegraphics[scale=0.2]{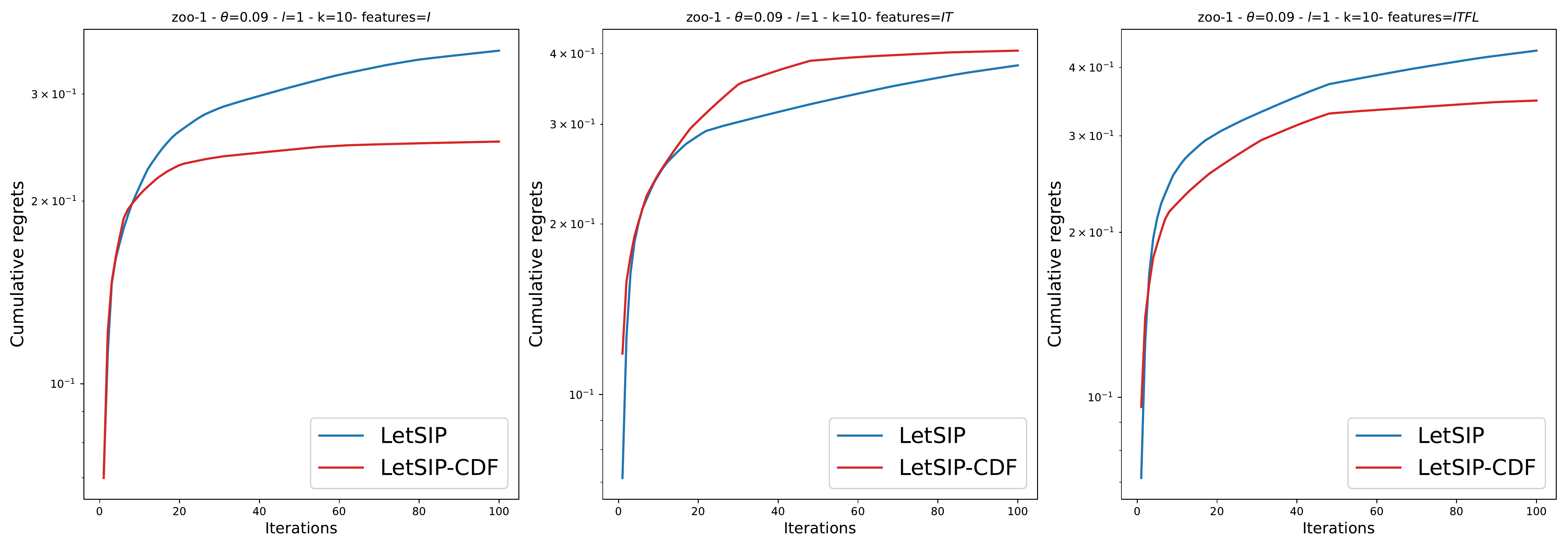}} \\
		\subfloat[][Zoo-1: non cumulative regret. ]{\includegraphics[scale=0.2]{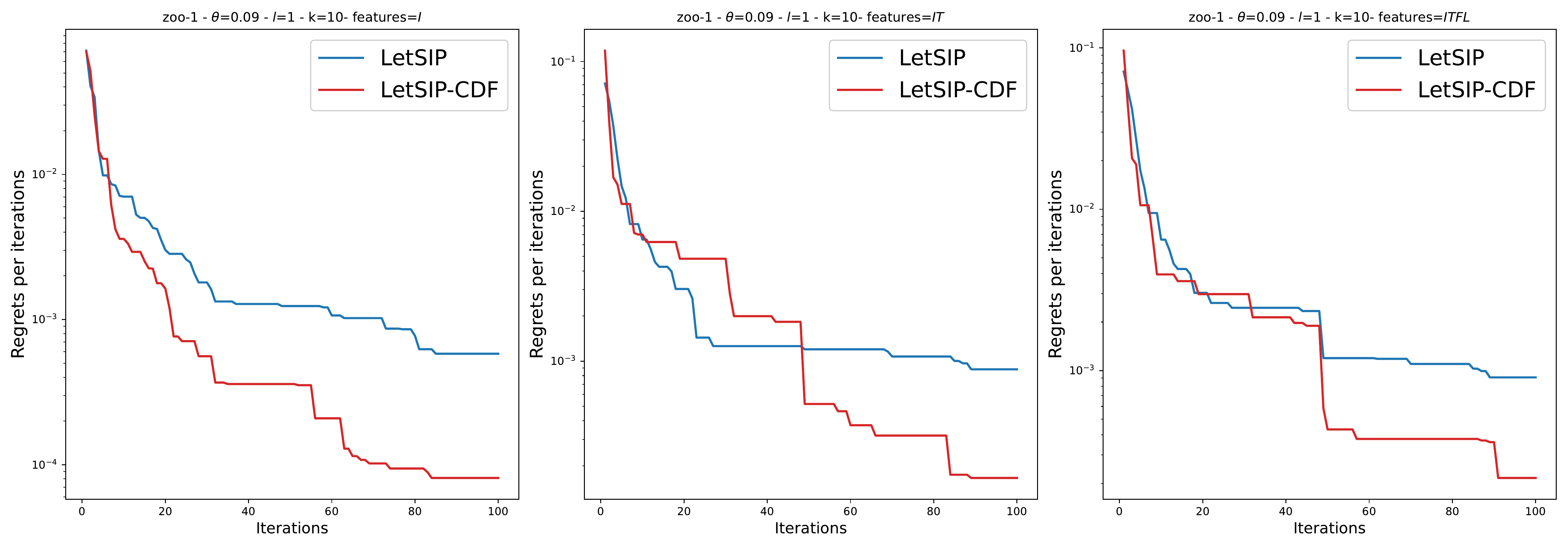}}	
	\end{tabular}
	\caption{A detailed view of comparison between \letsipcdf{} and \letsip{} and \letsip{} (cumulative and non-cumulative regret) for different pattern features w.r.t. $maximal$ quality, $k = 10$ and $\ell = 1$.}
	\label{fig:9}
\end{figure}


\section{Complementary results for CPU-times comparison}\label{secC1}

Figure~\ref{fig:comp:time:appendix} compares the CPU-times of \letsip{}, \letsipcdf{}, \newletsipexp{} and \newletsipcdf w.r.t. the ILFT feature combination, $k = 10$ and $\ell = 1$. 

\begin{figure}[t]
	\centering
	
	\begin{tabular}{cc}
		\subfloat[][Vote]{\includegraphics[height=3.0cm,scale=0.75]{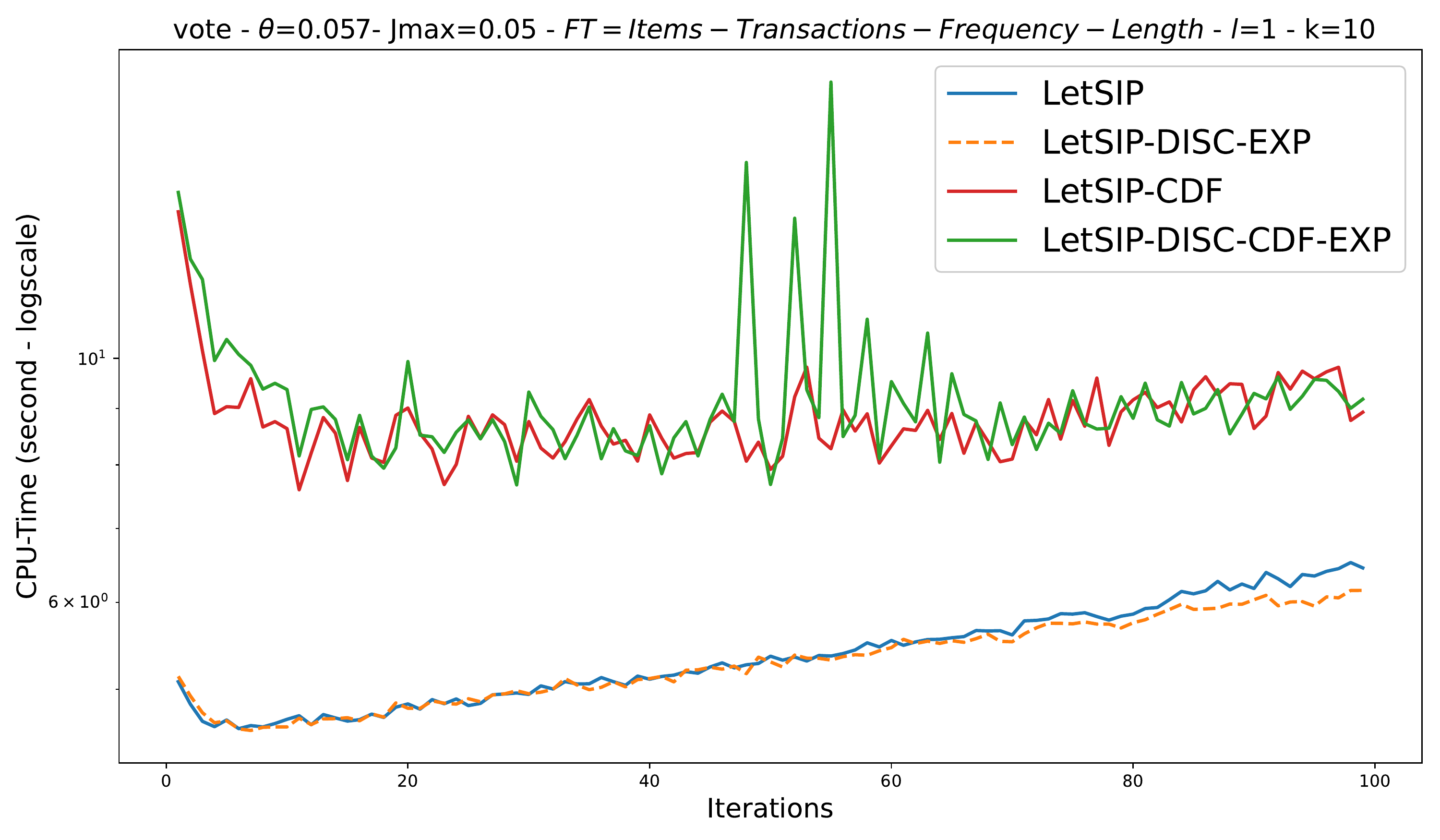}}
		&		\subfloat[][Hepatetis: runtime]{\includegraphics[height=3.0cm,scale=0.75]{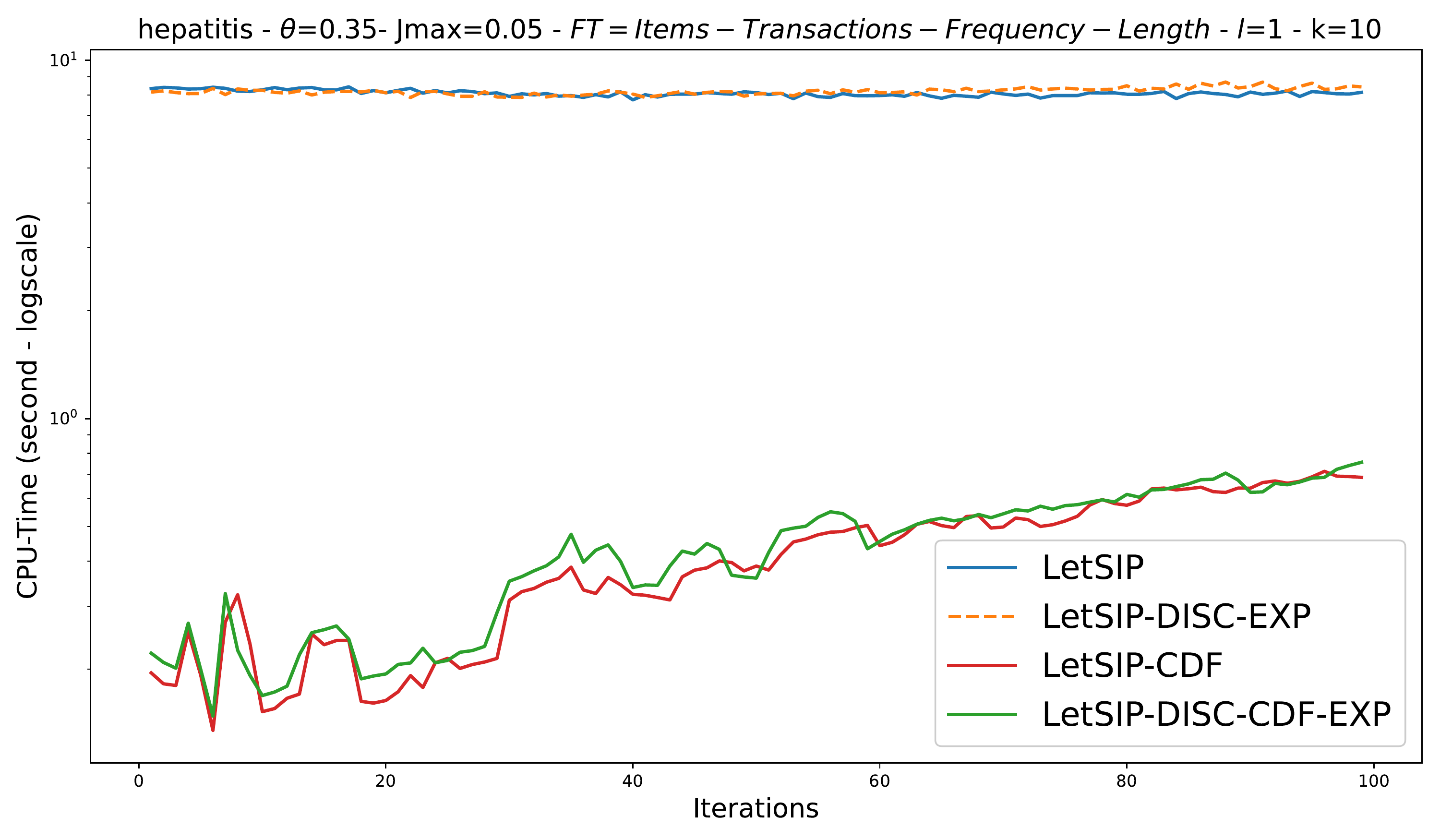}}
i		\\ 
		\subfloat[][Kr-vs-kp]{\includegraphics[height=3.0cm,scale=0.75]{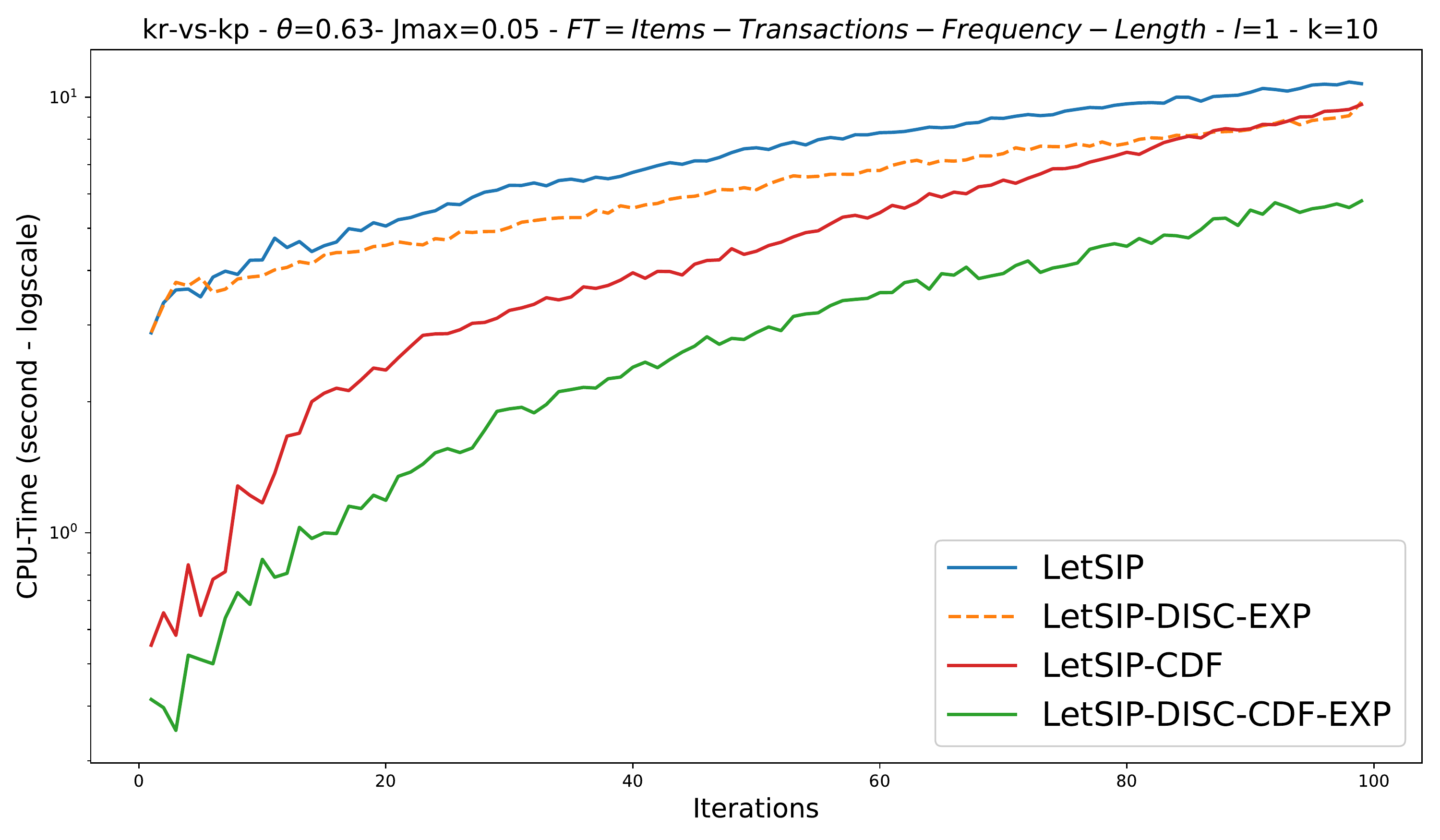}}
		&
		\subfloat[][Lymph]{\includegraphics[height=3.0cm,scale=0.75]{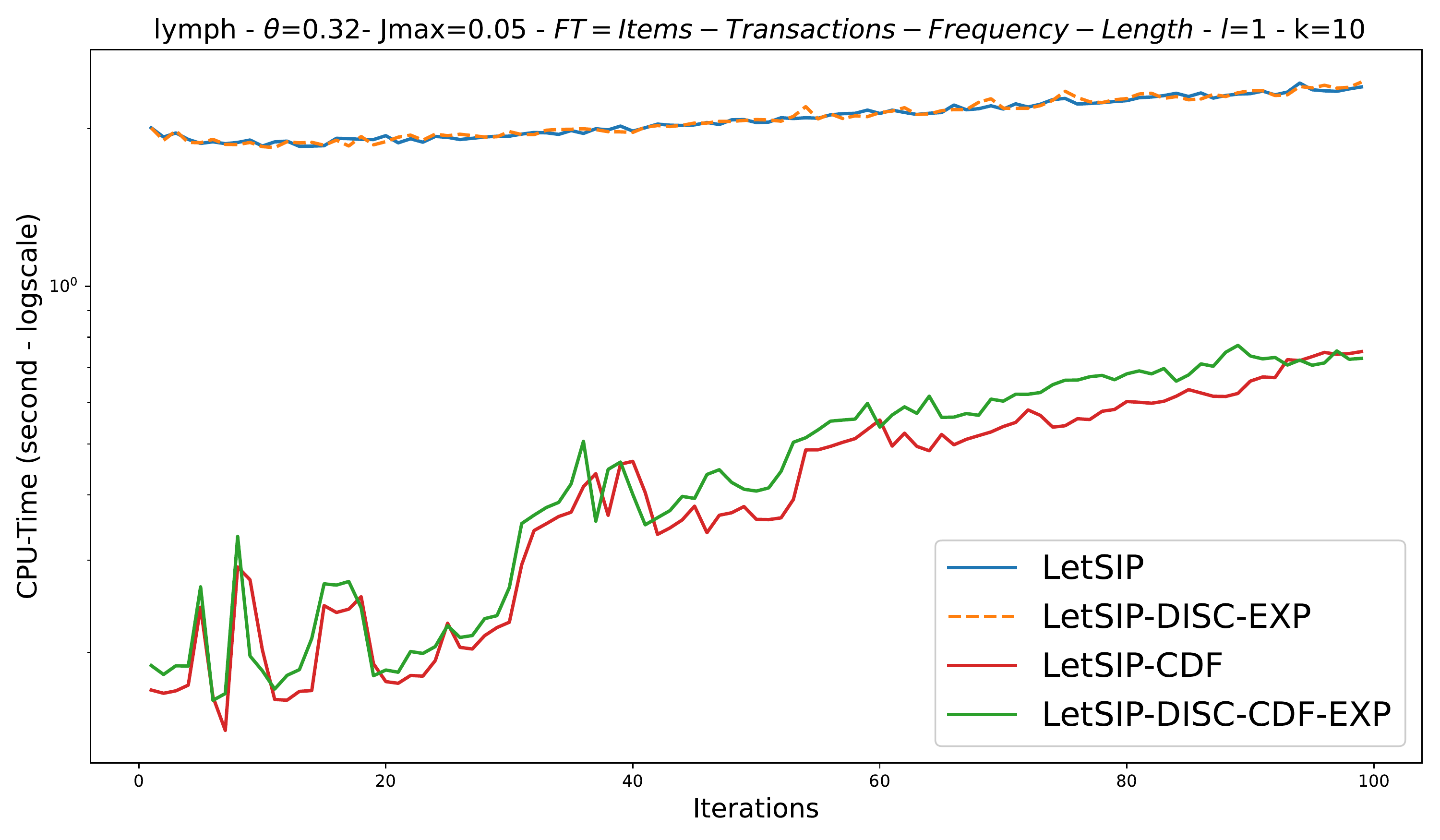}}
		\\ 
		\subfloat[][Mushroom: runtime ]{\includegraphics[height=3.0cm,scale=0.75]{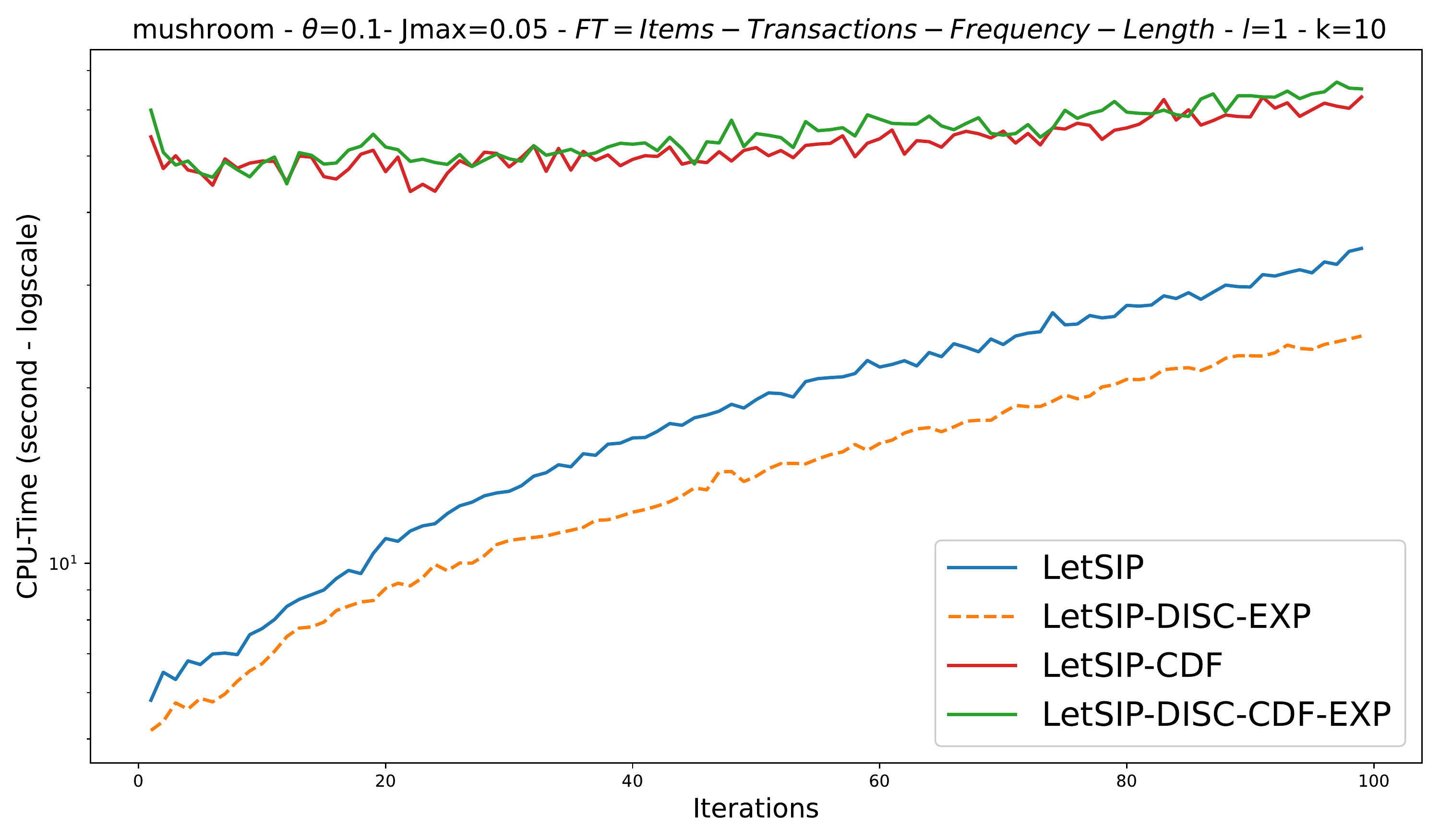}}
		&
		\subfloat[][Soybean: runtime ]{\includegraphics[height=3.0cm,scale=0.75]{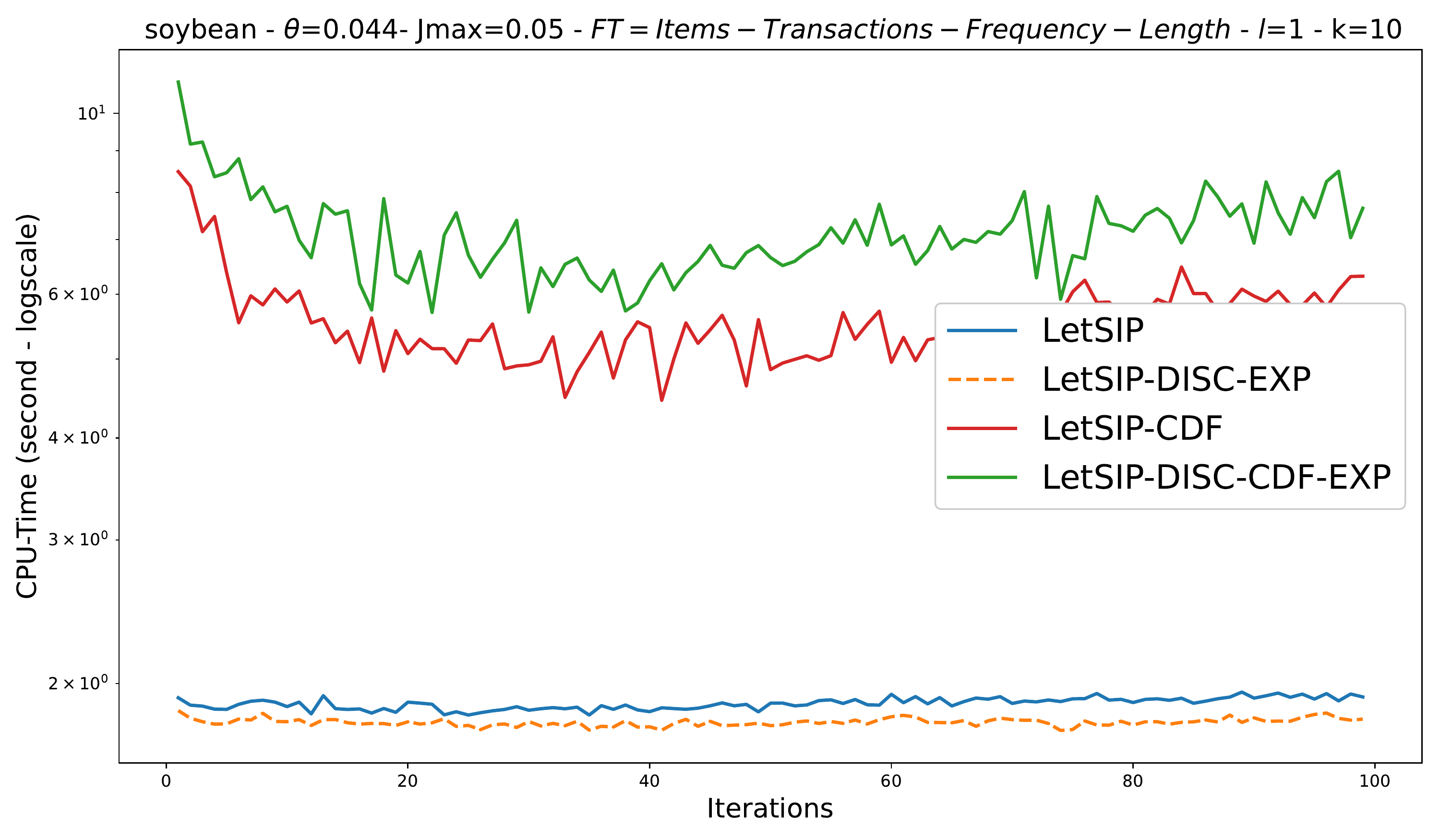}} \\
		
		\subfloat[][Anneal]{\includegraphics[height=3.0cm,scale=0.75]{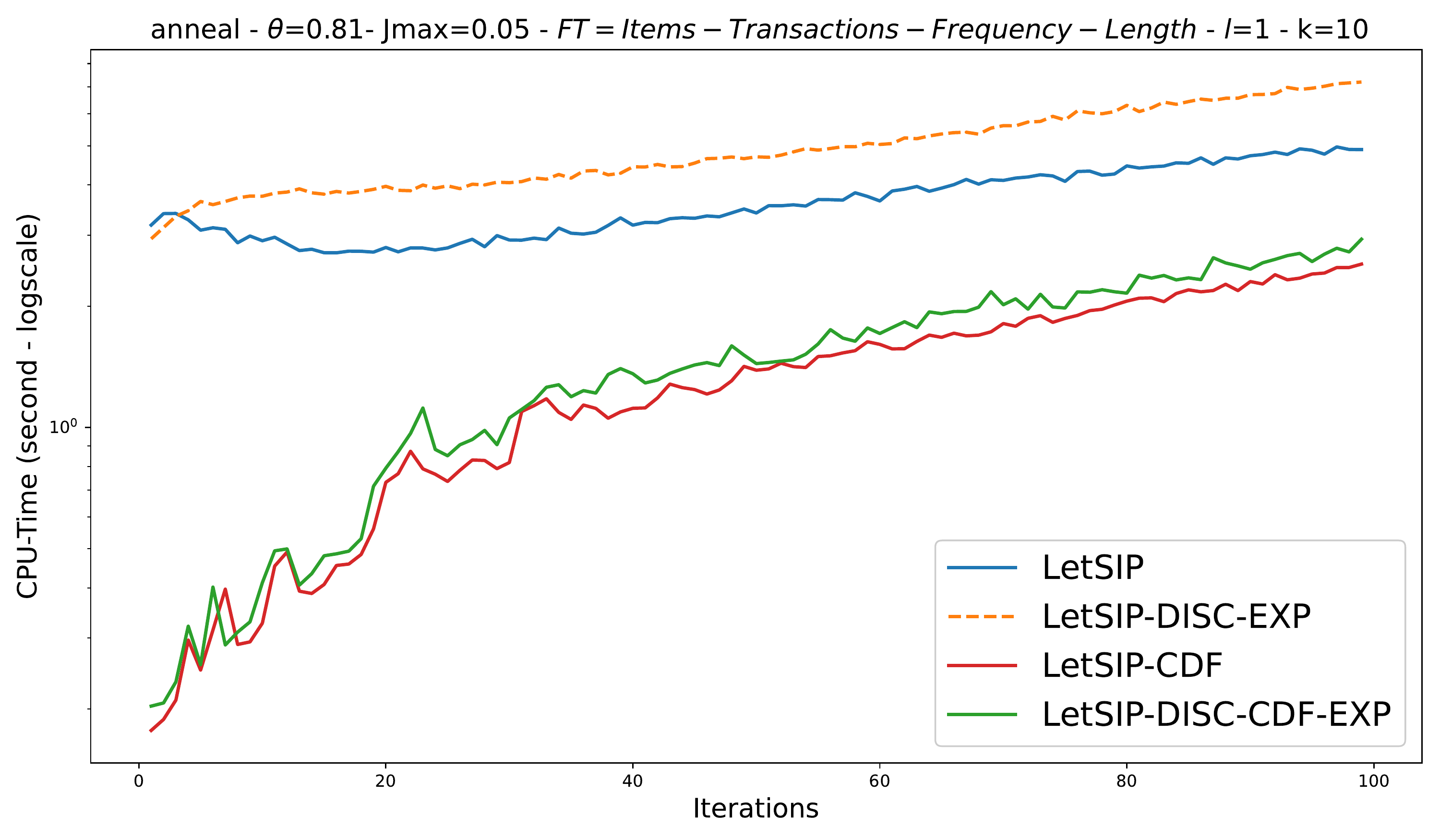}}
		&
		\subfloat[][Zoo-1]{\includegraphics[height=3.0cm,scale=0.75]{cpu_zoo-1-0v09-j_0v05-ITLF-l_1-k_10-0v13-eps-converted-to.pdf}} 
		
	\end{tabular}
	
	\caption{CPU-time analysis (\textit{sec.}) of \letsip{}, \letsipcdf, \newletsip{} and \newletsipcdfexp{} w.r.t. the ILFT feature combination. $k = 10$ and $\ell = 1$.}
	\label{fig:comp:time:appendix}
\end{figure}

	

	
\end{appendices}